\documentclass{article}

\usepackage{microtype}
\usepackage{graphicx}
\usepackage{subcaption}
\usepackage{booktabs} 
\usepackage{pgfplots}
\usepackage{pgfplotstable}
\usepackage{siunitx}
\usepackage{multirow}
\usetikzlibrary{calc}
\usepackage{graphicx}
\usepackage{adjustbox}
\usepackage{booktabs}
\usepackage{tabularx}
\usepackage[table]{xcolor}
\usepgfplotslibrary{groupplots}
\usepgfplotslibrary{fillbetween}
\pgfplotsset{compat=1.18}
\definecolor{royalBlue}{RGB}{65, 105, 225}
\definecolor{brickRed}{RGB}{178, 34, 34}
\definecolor{cRankSub}{RGB}{120,120,120} 

\usepackage{tikz}
\usetikzlibrary{shapes, arrows.meta, positioning, calc, backgrounds, decorations.pathreplacing, shadows}

\definecolor{RoyalBlue}{HTML}{00509D}   
\definecolor{BrightCyan}{HTML}{00C2CB}  
\definecolor{SoftGrey}{HTML}{F0F4F8}    
\definecolor{Slate}{HTML}{485460}       
\definecolor{White}{HTML}{FFFFFF}
\definecolor{RoyalBlue}{HTML}{00509D}   
\definecolor{BrightCyan}{HTML}{00C2CB}  
\definecolor{SoftGrey}{HTML}{F8FAFC}    
\definecolor{Slate}{HTML}{485460}       
\definecolor{AlertRed}{HTML}{E63946}    

\usepackage{hyperref}
\usepackage{doi}

\usepackage[preprint]{icml2026}

\usepackage{amsmath}
\usepackage{amssymb}
\usepackage{mathtools}
\usepackage{amsthm}
\usepackage{thmtools,thm-restate}

\DeclareMathOperator{\R}{\mathbb R}

\let\P\relax
\DeclareMathOperator{\P}{\mathcal P}
\renewcommand{\d}{\mathop{}\!\mathrm{d}}

\usepackage{dsfont}
\DeclareMathOperator{\1}{\mathds{1}}
\DeclareMathOperator{\tT}{\mathsf{T}}

\DeclareMathOperator{\N}{\mathbb N}
\let\phi\varphi

\DeclareMathOperator{\E}{\mathbb E}

\DeclareMathOperator{\D}{\mathrm D}
\DeclareMathOperator{\U}{\mathcal{U}}
\DeclareMathOperator{\TV}{TV}
\DeclareMathOperator{\C}{\mathcal C}
\let\S\relax
\DeclareMathOperator{\S}{\mathbb S}
\DeclareMathOperator{\SD}{SD}

\DeclareMathOperator{\V}{\mathbb V}

\let\H\relax
\DeclareMathOperator{\H}{\mathcal H}

\DeclareMathOperator{\supp}{supp}

\DeclareMathOperator{\Var}{\mathbb V}
\DeclareMathOperator{\Bin}{Bin}
\DeclareMathOperator{\Beta}{Beta}
\DeclareMathOperator{\PP}{\mathbb P}

\DeclareMathOperator{\ran}{ran}
\DeclareMathOperator*{\argmax}{arg\,max}
\DeclareMathOperator*{\argmin}{arg\,min}

\newcommand{\bftablenum}[1]{\multicolumn{1}{>{\bfseries}S}{#1}}

\usepackage[capitalize,noabbrev]{cleveref}
\usepackage{csquotes}

\theoremstyle{plain}

\declaretheorem[name=Theorem,numberwithin=section]{theorem}
\declaretheorem[sibling=theorem]{lemma}

\theoremstyle{definition}
\newtheorem{definition}[theorem]{Definition}

\theoremstyle{remark}
\newtheorem{remark}[theorem]{Remark}
\newtheorem{example}{Example}[section]

\crefname{definition}{Definition}{Definitions}
\Crefname{definition}{Definition}{Definitions}

\usepackage[disable]{todonotes}

\icmltitlerunning{Approximating $f$-Divergences with Rank Statistics}

\begin{document}

\twocolumn[
  \icmltitle{Approximating \texorpdfstring{$f$}{f}-Divergences with Rank Statistics}



  \icmlsetsymbol{equal}{*}

  \begin{icmlauthorlist}
    \icmlauthor{Viktor Stein}{TUM,equal}
    \icmlauthor{José Manuel de Frutos}{uc3m,equal}
  \end{icmlauthorlist}

  \icmlaffiliation{TUM}{Department of Mathematics, Technical University of Munich \& Munich Center for Machine Learning, Germany. The majority of the work was conducted while at the Institute of Mathematics at the Technical University of Berlin, Germany \& the Berlin Mathematical School.}
  \icmlaffiliation{uc3m}{Department of Signal Theory and Communications, Universidad Carlos III, Madrid, Spain}

  \icmlcorrespondingauthor{Viktor Stein}{viktor.stein@tum.de}
  \icmlcorrespondingauthor{José Manuel de Frutos}{jofrutos@ing.uc3m.es}

  \icmlkeywords{f-divergences, rank statistics, divergence estimation, sliced divergences, generative modeling}

  \vskip 0.3in
]



\printAffiliationsAndNotice{\icmlEqualContribution}

\begin{abstract}
We introduce a rank-statistic approximation of $f$-divergences that avoids explicit density-ratio estimation by working directly with the distribution of ranks.
For a resolution parameter $K$, we map the mismatch between two univariate distributions $\mu$ and $\nu$ to a rank histogram on $\{ 0, \ldots, K\}$ and measure its deviation from uniformity via a discrete $f$-divergence, yielding a rank-statistic divergence estimator.
We prove that the resulting estimator of the divergence is monotone in $K$, is always a lower bound of the true $f$-divergence, and we establish quantitative convergence rates for $K\to\infty$  under mild regularity of the quantile-domain density ratio.
To handle high-dimensional data, we define the sliced rank-statistic $f$-divergence by averaging the univariate construction over random projections, and we provide convergence results for the sliced limit as well.
We also derive finite-sample deviation bounds along with asymptotic normality results for the estimator.
Finally, we empirically validate the approach by benchmarking against neural baselines and illustrating its use as a learning objective in generative modeling experiments.
\end{abstract}

\section{Introduction}
\label{sec:intro}
Quantifying discrepancies between probability distributions is fundamental in statistics and machine learning.
A prominent and widely used class of such measures is $f$-divergences, defined in \eqref{eq:cont-f-div}, which include the Kullback-Leibler divergence, total variation, Hellinger, and $\chi^2$-type divergences \citep{ali1966,csiszar1967}.
They arise throughout the field, from hypothesis testing and model comparison to variational objectives for implicit generative modeling \citep{goodfellow2014gan,nowozin2016fgan}.
However, reliably estimating $f$-divergences from samples is challenging: most formulations depend on the density ratio $\frac{\d\mu}{\d\nu}$, so approaches that first estimate the densities (or their ratio) and then substitute can suffer from severe statistical error in moderate-to-high
dimensions \citep{moon2014multivariate, rubenstein2019practical}.

A common workaround is to estimate $f$-divergences via \emph{variational} formulations, which recast
divergence estimation as (regularized) risk minimization and, in many cases, as a classification-style objective
\citep{nguyen2010divergence,nowozin2016fgan,RRGDP2012}.
For KL in particular, the Donsker--Varadhan representation gives another variational objective, which is used by neural mutual-information estimators such as MINE \citep{belghazi2018mine}. Related principles include noise-contrastive estimation for unnormalized models, which learns by discriminating data from artificial noise samples \citep{gutmann2010nce}.
In practice, however, these approaches may require delicate
function-class choices and optimization heuristics, and can inherit the instabilities associated with
adversarial/variational training \citep{arjovsky2017wgan,gulrajani2017wgangp}.

A different line of work mitigates high-dimensional difficulties by comparing \emph{one-dimensional
projections} of distributions and aggregating the resulting discrepancies. The sliced Wasserstein
distance and its extensions are prominent examples in generative modeling, offering favorable
computational scaling by reducing multivariate comparisons to repeated 1D problems
\citep{kolouri2019gsw,wu2019swgm}. More generally, sliced probability divergences have been studied
from statistical and topological viewpoints \citep{nadjahi2020sliced}.
These ideas suggest that if one can build a robust and scalable 1D divergence estimator, then slicing can lift it to higher dimensions.

In this paper, we develop a \emph{rank-statistic} approximation of $f$-divergences. The construction
starts from a fixed reference measure $\nu$ and uses the (univariate) probability integral transform
(PIT): if $X\sim\nu$, then $F_\nu(X)$ is uniform on $[0,1]$ \citep{rosenblatt1952,S1952,K1957}.
Uniformity diagnostics based on PIT/rank histograms are standard tools in forecast calibration and
reliability assessment \citep{hamill2001rankhist,gneiting2007calibration}.
We turn this principle into a general divergence construction: we discretize the PIT into a
\emph{rank histogram} with $K$ bins and measure its deviation from uniformity via an entropic
function $f$.
For a related construction to approximate the CDF of a probability density, see \cite{L2012}.
The resulting rank-statistic divergence is bounded, depends only on \emph{order
information}, and admits simple estimators built from sorting and counting operations.
We then extend this divergence to $\R^d$ by averaging over random 1D projections, yielding \emph{sliced
rank-statistic $f$-divergences} in the spirit of \citep{kolouri2019gsw,nadjahi2020sliced,Beckmann2025a}.

\subsection{Contributions} \label{subsec:contributions}
The main results of this paper are the following:
\begin{itemize}
    \item
    We propose a rank-histogram approximation $\mathbf{D}^{(K)}_{f,\nu}(\mu)$ of the $f$-divergence, parameterized by a resolution $K$, which is an optimization-free estimator of the (sliced) $f$-divergence.
    This generalizes \cite{dFVOM2024,dFVOM2024b,dFVOM2025} to different choices of entropy function $f$, enabling us to choose differentiable functions $f$ that better interact with automatic differentiation schedules used for learning tasks.
    
    \item
    We establish basic regularity properties, and show that $\mathbf{D}^{(K)}_{f,\nu}(\mu)$ is nondecreasing in $K$ and dominated by $\mathbf{D}_{f,\nu}(\mu)$. Under mild assumptions on the density ratio, we prove consistency as $K\to\infty$ and provide
    quantitative approximation rates. We also derive finite-sample deviation bounds for the univariate estimator and prove asymptotic normality.
    
    \item
    We define sliced rank-statistic $f$-divergences in $\R^d$ by averaging the univariate construction over random 1D projections, thereby inheriting the univariate $ f$-divergence's key properties.
    
    \item
    We benchmark against classical and neural baselines on synthetic tasks, showing that rank-statistic $f$-divergences provide stable approximations of the target $f$-divergence that perform well in high dimensions and with few samples.
    Our generative transport experiments show that rank-statistic $f$-divergences can be used as sample-based transport objectives for implicit models, illustrated on two-dimensional toy problems and on image datasets.
\end{itemize}

\paragraph{Notation}
By $\N$ we denote the non-negative integers. 
For $K \in \N$ we set $[K] \coloneqq \{ 0, \ldots, K \}$.
The uniform distribution on $[K]$ is denoted by $U_K$.
The quantile function of a univariate probability measure $\mu \in \P(\R)$ is denoted by $Q_{\mu}$ and its CDF by $R_{\mu}$.
The expectation of a function $f$ under $\mu \in \P(\R^d)$ is denoted by $\E_{\mu}[f] \coloneqq \int_{\R^d} f(x) \d \mu(x)$.
The pushforward is denoted by $\#$.
We denote by $\C^k$ the $k$-times continuously differentiable functions.
 We say that $h$ is $H$-Hölder continuous of order $\alpha$ and write $h \in \C^{0, \alpha}$ if $| h(x) - h(y) | \le H | x - y |^{\alpha}$ holds for all $x, y$.
The range of a function $h$ is denoted by $\ran(h)$.

Given i.i.d. samples $X_1,\ldots,X_N \sim \mu$ and
$Y_1,\ldots,Y_M \sim \nu$, we denote the corresponding empirical
measures by $\hat{\mu}_N$, $\hat{\nu}_M$.
Throughout, unhatted symbols denote population-level objects,
whereas hatted symbols denote empirical/sample-based quantities.

\section{One-Dimensional Rank-Based Approximation of \texorpdfstring{$f$}{f}-Divergences}\label{sec:onedim}
We begin with the one-dimensional setting and introduce a rank-based approximation of $f$-divergences.
The construction relies on the probability integral transform
and a discrete rank histogram that encodes the mismatch between a distribution $\mu$ and a target $\nu$.

In this section, let $\mu,\nu \in \mathcal{P}(\mathbb{R})$ be univariate probability measures and $\nu$ be atomless.
Throughout, let
$f \colon [0,\infty) \to \mathbb{R} \cup \{+\infty\}$ always be a convex, lower
semicontinuous function with $f(1)=0$ and $\lim_{t \to \infty} \frac{1}{t} f(t) > 0$.
We then say that $f$ is an \textit{entropy function}.
Note that due to convexity, $f$ is Lipschitz on any compact set $C \subset (0, \infty)$.

The (continuous)
$f$-divergence of $\mu$ with respect to $\nu$ is
\begin{equation} 
  D_{f,\nu}(\mu)
  \coloneqq 
  \begin{cases}
    \displaystyle \int f\!\left(\frac{\d\mu}{\d\nu}(x)\right)\d\nu(x),
      & \text{if } \mu \ll \nu,\\[0.35em]
    +\infty, & \text{otherwise.}
  \end{cases}
  \label{eq:cont-f-div}
\end{equation}
Directly working with $\frac{\d\mu}{\d\nu}$ is often inconvenient. Instead, we approximate
$D_{f,\nu}(\mu)$ using a rank statistic of $\mu$ relative to $\nu$.

\begin{definition}
\label{def:rank-stat}
Fix $K \in \mathbb{N}$. Let $Y \sim \mu$ and
$(\tilde Y_j)_{j = 1}^{K} \stackrel{\text{i.i.d.}}{\sim} \nu$, independent
of $Y$. The \emph{rank statistic of order $K$} of $\mu$ with respect to $\nu$ is
\begin{equation}
  A^{(K)}_{\mu\mid\nu}
  \coloneqq 
  \#\bigl\{j \in \{1, \ldots, K\} : \tilde Y_j \le Y\bigr\}
  \in [K].
\end{equation}
We denote by $Q^{(K)}_{\mu\mid\nu}$ the probability mass function (PMF) of
$A^{(K)}_{\mu\mid\nu}$ on $[K]$.
\end{definition}

When $\mu\ll\nu$, then $\mu = \nu$ if and only if the rank statistic is
uniform, i.e., $Q^{(K)}_{\nu\mid\nu}(n) \equiv 1/(K+1)$, for all $K \in \N$, see \cref{lemma:properties_rank_statistics}.

The PMF $Q^{(K)}_{\mu\mid\nu}$ can be seen as a discrete \enquote{rank histogram} of
$\mu$ with respect to $\nu$.
It records how often a draw from $\mu$
falls below the $K$ i.i.d.\ draws from $\nu$. Departures of this histogram
from the uniform law signal discrepancies between $\mu$ and $\nu$.

\begin{figure*}[t]
  \centering
\begin{tikzpicture}[
    scale=0.95, transform shape,
    font=\sffamily,
    >=latex,
    execute at begin picture={
        \definecolor{RoyalBlue}{HTML}{00509D}
        \definecolor{BrightCyan}{HTML}{00C2CB}
        \definecolor{SoftGrey}{HTML}{F8FAFC}
        \definecolor{Slate}{HTML}{485460}
        \definecolor{AlertRed}{HTML}{E63946}
    },
    panel/.style={
        rectangle, 
        draw=Slate!15, 
        fill=SoftGrey, 
        line width=0.8pt, 
        rounded corners=8pt, 
        drop shadow={opacity=0.06, shadow xshift=2pt, shadow yshift=-2pt}, 
        inner sep=6pt, 
        align=center
    },
    mu point/.style={circle, ball color=RoyalBlue, shading=ball, inner sep=0pt, minimum size=5pt},
    nu point/.style={circle, draw=BrightCyan, fill=white, thick, inner sep=0pt, minimum size=4.5pt},
    hist bar/.style={top color=RoyalBlue!80, bottom color=RoyalBlue!50, draw=none, rounded corners=1pt},
    bin line/.style={dashed, BrightCyan!50, thin}
]

\node[panel] (p1) at (0,0) {
    \begin{minipage}{6.8cm} 
        \centering
        \vspace{0.05cm}
        \textbf{\textcolor{RoyalBlue}{(a) Ideal Case: $\mu = \nu$}} \\
        \textcolor{Slate}{\scriptsize Samples are uniformly interleaved}
        \vspace{0.2cm} 
        
        \begin{tikzpicture}[scale=0.85, baseline=(current bounding box.center)] 
            \draw[->, thick, Slate] (0, 2.8) -- (6.2, 2.8) node[above, scale=0.8] {Space $\mathcal{X}$};
            
            \foreach \x in {1.0, 2.0, 3.0, 4.0, 5.0} {
                \node[nu point] at (\x, 2.8) {};
                \draw[bin line] (\x, 2.7) -- (\x, 0); 
            }
            \draw[bin line] (0, 2.7) -- (0, 0); 
            \draw[bin line] (6, 2.7) -- (6, 0);
            
            \foreach \x in {0.5, 1.5, 2.5, 3.5, 4.5, 5.5} {
                \node[mu point] at (\x, 2.8) {};
            }
            
            \node[scale=0.8, RoyalBlue, anchor=south] at (0.5, 2.95) {$\mu$};
            \node[scale=0.8, BrightCyan!80!black, anchor=south] at (1.0, 2.95) {$\nu$};

            \draw[->, thick, Slate] (0,0) -- (6.2,0) node[below, scale=0.8] {Rank}; 
            \draw[->, thick, Slate] (0,0) -- (0,2.0) node[left, scale=0.8] {Freq};
            
            \foreach \x in {0, 1, 2, 3, 4, 5} {
                \draw[hist bar] (\x+0.1, 0) rectangle (\x+0.9, 1.2);
            }
            
            \node[scale=0.75, Slate!80, font=\bfseries] at (3.0, 1.5) {Uniform Histogram};
            \draw[dashed, Slate!50] (0, 1.2) -- (6, 1.2);
            
            \path (0,0) -- (6.4, 3.1);
        \end{tikzpicture}
        \vspace{-0.35cm} 
    \end{minipage}
};

\node[panel, right=0.8cm of p1] (p2) { 
    \begin{minipage}{6.8cm} 
        \centering
        \vspace{0.05cm}
        \textbf{\textcolor{AlertRed}{(b) Mismatch: $\mu \neq \nu$}} \\
        \textcolor{Slate}{\scriptsize Real samples cluster in tails}
        \vspace{0.2cm}
        
        \begin{tikzpicture}[scale=0.85, baseline=(current bounding box.center)]
            \draw[->, thick, Slate] (0, 2.8) -- (6.2, 2.8) node[above, scale=0.8] {Space $\mathcal{X}$};
            
            \foreach \x in {1.0, 2.0, 3.0, 4.0, 5.0} {
                \node[nu point] at (\x, 2.8) {};
                \draw[bin line] (\x, 2.7) -- (\x, 0); 
            }
            \draw[bin line] (0, 2.7) -- (0, 0); 
            \draw[bin line] (6, 2.7) -- (6, 0);
            
            \node[mu point] at (2.2, 2.8) {}; \node[mu point] at (2.5, 2.8) {}; \node[mu point] at (2.8, 2.8) {};
            \node[mu point] at (3.2, 2.8) {}; \node[mu point] at (3.5, 2.8) {}; \node[mu point] at (3.8, 2.8) {};

            \draw[->, thick, Slate] (0,0) -- (6.2,0) node[below, scale=0.8] {Rank};
            \draw[->, thick, Slate] (0,0) -- (0,2.0); 
            
            \draw[hist bar, top color=AlertRed!80, bottom color=AlertRed!50] (2.1, 0) rectangle (2.9, 1.8);
            \draw[hist bar, top color=AlertRed!80, bottom color=AlertRed!50] (3.1, 0) rectangle (3.9, 1.8);
            
            \foreach \x in {0, 1, 4, 5} {
                \draw[fill=Slate!10, draw=none] (\x+0.1, 0) rectangle (\x+0.9, 0.15);
            }
            
            \node[scale=0.75, AlertRed, font=\bfseries] at (3.0, 2.1) {High Divergence};
            \draw[dashed, Slate!50] (0, 1.2) -- (6, 1.2);
            
            \path (0,0) -- (6.4, 3.1);
        \end{tikzpicture}
        \vspace{-0.35cm}
    \end{minipage}
};

\end{tikzpicture}
\caption{\small Conceptual illustration of the rank-statistic $f$-divergence. \textbf{(a)} When $\mu = \nu$, samples are uniformly interleaved, resulting in a uniform rank histogram. \textbf{(b)} With a mismatch, samples from $\mu$ cluster in specific rank bins, creating a non-uniform histogram that indicates divergence.}
\label{fig:conceptual_divergence}
\end{figure*}
To quantify the discrepancy of the rank histogram from uniformity, we use a discrete $f$-divergence, see also \cref{remark:f-div_pushforward_reformulation}.

\begin{definition}\label[definition]{defn:rankStatistic_f_Divergence}
Let $U_K$ denote the uniform distribution on $[K]$. The
\emph{rank-statistic $f$-divergence of order $K$} of $\mu$ with respect to $\nu$ is
\begin{align}
  D^{(K)}_{f,\nu}(\mu)
  &\coloneqq 
  \D_f\!\bigl(Q^{(K)}_{\mu\mid\nu} \,\Vert\, U_K\bigr) \nonumber
  \\&=
  \frac{1}{K+1}
  \sum_{n=0}^K
    f\!\Bigl((K+1)\,Q^{(K)}_{\mu\mid\nu}(n)\Bigr),
  \label{eq:rank-f-div}
\end{align}
where $\D_f(\cdot\Vert\cdot)$ is the discrete $f$-divergence on the finite
alphabet $[K]$.
\end{definition}

\begin{example}
    For the entropy function $f_{\TV} \coloneqq | \cdot - 1 |$ of the total variation divergence, \cref{defn:rankStatistic_f_Divergence} recovers the ISL discrepancy $d_K$ from \cite{dFVOM2024} on absolutely continuous measures, up to a prefactor: $D_{f_{\TV}, \nu}^{(K)}(\mu) = (K + 1) d_K(\mu, \nu)$.
\end{example}

Now, we collect basic properties of the
rank-statistic $f$-divergence, in particular that
$D^{(K)}_{f,\nu}$ is monotone in the rank
resolution $K$ and that, similarly to $D_{f, \nu}$, the approximation $D^{(K)}_{f,\nu}$ inherits regularity properties from $f$.
 The second inequality below generalizes \citep[Thm.~2]{dFVOM2024} \citep[Thm.~2.2]{dFVOM2025}.

\begin{restatable}{theorem}{MonotoneK}
\label{thm:monotone-K}
    Let $\mu, \nu \in \P(\R)$ and $K \in \N$.
    The map $D_{f,\nu}^{(K)}$ is convex, and if $R_{\nu}$ is continuous, then it is also weakly lower semicontinuous.
    Furthermore,
    \begin{equation} \label{eq:monotonicity}
      D^{(K)}_{f,\nu}(\mu)
      \le
      D^{(K+1)}_{f,\nu}(\mu)
      \le D_{f, \nu}(\mu).
    \end{equation}
\end{restatable}
\begin{proof}
    See \cref{subsec:Proof_of_Monotonicity}.
\end{proof}

\begin{remark}[Markov kernel interpretation of $D_{f}^{(K)}$]
    Let $(b_{n, K})_{n = 0}^{K}$ be the Bernstein polynomials, see \cref{sec:well_known_results}.
    One can also prove \eqref{eq:monotonicity} by noticing that the Markov kernel $\kappa \colon \R \times [K]\to [0,1]$, $(y, \{n\}) \mapsto (b_{n, K} \circ R_{\nu})(y)$ fulfills $D_{f, \nu \kappa}(\mu \kappa) = D_{f, \nu}^{(K)}(\mu)$ and then use the data processing inequality for $f$-divergences (\cref{lemma:discrete_DP}).
    We also have $D_{f, T_\# \nu}^{(K)}(T_{\#} \mu) = D_{f, \nu}^{(K)}(\mu)$ for strictly increasing functions $T \colon \R \to \R$.
\end{remark}

\subsection{Approximation properties}

We are interested in the behavior of the increasing sequence
$\left\{D^{(K)}_{f,\nu}(\mu)\right\}_{K\in \N}$ as the resolution
parameter $K$ grows.
Increasing $K$ refines the rank histogram,
so one expects the discrete quantity $D_{f,\nu}^{(K)}(\mu)$ to approach the
continuous $f$-divergence $D_{f,\nu}(\mu)$. This is indeed the case under
a mild regularity assumption on the rank density ratio $r$, whose regularity determines the convergence rates precisely in the way that it determines the convergence rate of the Bernstein approximation
, see \cref{sec:well_known_results}.

\begin{restatable}[Convergence of the truncated divergence]{theorem}{ConvergenceKtoInfty}
\label{thm:convergence_K_to_Infty}
If $\mu \ll \nu$ and $r \coloneqq r_{\mu \mid \nu} \coloneqq \frac{\d \mu}{\d \nu} \circ Q_{\nu} \in \C([0, 1])$ and $f$ is $L_f$-Lipschitz on $\ran(r)$, then for $K \to \infty$,
\begin{equation*}
    D_{f, \nu}(\mu) - D_{f, \nu}^{(K)}(\mu)
    \begin{cases}
        \to 0, & \text{if } r \in \C([0, 1]), \\
        \in O(K^{-\frac{\alpha}{2}}), & \text{if } r \in C^{0, \alpha}([0, 1]), \\
        \in  O(K^{-1}), & \begin{array}{c}\text{if } r \in \C^2([0, 1]), \text{ or} \\
         \text{if } r \text{ is Lipschitz and} \\ f \in \C^2([0, \infty)).\end{array}
    \end{cases}
\end{equation*}
\end{restatable}
\begin{proof}
    Consider the piecewise-constant function
    \[
        r_K(u)
        \coloneqq (K + 1) Q_{\mu \mid \nu}^{(K)}(n)
    \]
    for $u\in\left[\frac{n}{K+1},\frac{n+1}{K+1}\right)$, $n \in [K]$.
    Then,
    \begin{align*}
        D^{(K)}_{f, \nu}(\mu)
        & = \frac{1}{K+1}\sum_{n=0}^K f\big((K + 1) Q_{\mu \mid \nu}^{(K)}(n)\big) \\
        & = \int_0^1 f\!\big(r_K(u)\big) \d{u}.
    \end{align*}
    In \cref{proof:convergenceKtoInfty}, we prove that $r_K$ converges uniformly to $r$ and prove the rates.
    The result then follows from
    \begin{align*}
        D_{f, \nu}(\mu) - D_{f, \nu}^{(K)}(\mu)
        & \le \int_{0}^{1} | f(r_K(u)) - f(r(u)) | \d{u} \\
        & \le L_f \| r_K - r \|_{\infty}. \qedhere
    \end{align*}
\end{proof}
We examine the applicability of \Cref{thm:convergence_K_to_Infty} to standard $f$-divergences.
\begin{example}[Applicability of Convergence Rates]
\label{ex:convergence_applicability}
Among the standard entropy functions considered here, $f_{\TV}$ is the only globally Lipschitz one (up to scalar prefactors).
\begin{itemize}
    \item 
    If $0 \not\in \ran(r)$, then the $O(K^{-\frac{\alpha}{2}})$ rate is achieved for most divergences (including KL, Jensen–Shannon, squared Hellinger, and Jeffreys) since they are Lipschitz away from zero.

    \item
    The fast rate $O(K^{-1})$ is obtained if $f \in \C^2([0, \infty))$ which excludes KL and Hellinger and $| \cdot - 1 |^{\alpha}$, for $\alpha \in (1, 2)$, but holds for the $\chi^2$-divergence (with $f_{\chi^2}(t) = \frac{1}{2}(t - 1)^2$) and other polynomial (\enquote{Tsallis})-entropy functions, and the triangular discrimination generator $f(t)=\frac{(t-1)^2}{t+1}$ \cite{L94}.
\end{itemize}
For a long list of choices of $f$, see \citep[Tab.~1]{SNRS2025}.
\end{example}

\subsection{Empirical estimation and finite-sample bounds}

We now turn to the empirical estimation of the rank-statistic
$f$-divergence. Given sample sizes $N,M \in \N$, let
$X_1,\ldots,X_N \stackrel{\mathrm{i.i.d.}}{\sim} \mu$ and
$Y_1,\ldots,Y_M \stackrel{\mathrm{i.i.d.}}{\sim} \nu$. We denote the
corresponding empirical measures by
\begin{equation}
\label{eq:empirical_measure}
  \hat\mu_N \coloneqq \frac{1}{N}\sum_{i=1}^N \delta_{X_i},
  \qquad
  \hat\nu_M \coloneqq \frac{1}{M}\sum_{j=1}^M \delta_{Y_j}.
\end{equation}

The empirical estimator is obtained by applying the rank construction to
$\hat\mu_N$ and $\hat\nu_M$.

The exact plug-in rank law is
$
\bar Q^{(K)}_{N,M}
\coloneqq 
Q^{(K)}_{\hat\mu_N\mid\hat\nu_M}$.
The corresponding plug-in estimator is
\begin{equation*}
D^{(K)}_{f,\hat\nu_M}(\hat\mu_N)
\coloneqq 
\frac{1}{K+1}
\sum_{n=0}^K
f\left((K+1)\bar Q^{(K)}_{N,M}(n)\right).
\end{equation*}
Equivalently, one may approximate $\bar Q^{(K)}_{N,M}$ by Monte Carlo resampling $K$ reference points from $\hat\nu_M$.

The following results bound the finite-sample error between the empirical
estimator and the true rank-statistic divergence for fixed $K$ and also give a high-probability concentration bound around the expectation
of the empirical estimator.

\begin{restatable}[Univariate finite sample complexity and concentration bound]{theorem}{FiniteSampleUnivariate}
\label{thm:finite_sample_univariate}
    Let $K \in \N$ be fixed, and let $\mu, \nu \in \mathcal{P}(\mathbb{R})$ and $\hat{\mu}_N$ and $\hat{\nu}_M$ be their corresponding empirical measures with sample sizes $N$ and $M$.
    Suppose that $f$ is $L_f$-Lipschitz on $[0, K + 1]$.
    The expected estimation error satisfies
    \begin{align*}
        \E &\left[ \left| D^{(K)}_{f, \hat{\nu}_M}(\hat{\mu}_N) - D^{(K)}_{f, \nu}(\mu) \right| \right] \\&\le L_f (K+1) \sqrt{2 \pi} \left( \frac{1}{\sqrt{N}} + \frac{1}{\sqrt{M}} \right).
    \end{align*}
    For any $\delta > 0$, with probability at least $1 - \delta$, we have
    \begin{align*}
        \bigl|
          D^{(K)}_{f, \hat{\nu}_M}(\hat{\mu}_N)
          &- \E\bigl[D^{(K)}_{f, \hat{\nu}_M}(\hat{\mu}_N)\bigr]
        \bigr|
        \\
        &\le\;
        L_f (K+1)
        \sqrt{
          2\log(2/\delta)
          \left( \frac{1}{N} + \frac{1}{M} \right)
        }.
    \end{align*}
\end{restatable}

\begin{proof}
    See \cref{subsec:proof_univ_finite_sample_complexity_bound} and \cref{subsec:proof_concentration_bound}.
\end{proof}

\section{Sliced Rank-Based $f$-Divergences in Higher Dimensions}
\label{sec:sliced}

We now extend  the rank-statistic $f$-divergence from
\cref{defn:rankStatistic_f_Divergence} to the $d$-dimensional setting via
\emph{slicing}. The idea is to reduce the
high-dimensional discrepancy between $\mu$ and $\nu$ to a collection of
one-dimensional discrepancies along suitably chosen projections, in the spirit
of sliced Wasserstein distances and related constructions. Throughout this
section, we work with one-dimensional projections along unit directions on the
sphere.

For $s \in \S^{d-1}$, let $\mu_s \coloneqq s_{\#}\mu$ be the one-dimensional pushforward of $\mu$ by $x \mapsto s^\top x$.
For
fixed $K \in \N$, \cref{defn:rankStatistic_f_Divergence} yields a univariate
rank-statistic $f$-divergence $D^{(K)}_{f}(\mu_s \mid \nu_s)$ describing
the mismatch between $\mu_s$ and $\nu_s$.
We assume that $\nu_s$ is atomless for almost every $s \in \S^{d - 1}$.

\begin{definition}[Sliced rank-statistic $f$-divergence]
\label[definition]{defn:sliced_rank_f_divergence}
Let $\mu,\nu \in \P(\R^d)$ with $\mu \ll \nu$. Let $\sigma$ denote the uniform probability measure on $\S^{d-1}$.
The \emph{sliced rank-statistic $f$-divergence} of order $K$ and the
\emph{sliced $f$-divergence} are, respectively, 
\begin{align*}
  \mathbf{D}^{(K)}_{f,\nu}(\mu)
  &\coloneqq
  \int_{\S^{d-1}} D^{(K)}_{f, \nu_s}(\mu_s)\d\sigma(s),\\
  \SD_{f,\nu}(\mu)
  &\coloneqq
  \int_{\S^{d-1}} D_{f, \nu_s}(\mu_s)\d\sigma(s).
\end{align*}
\end{definition}
 Note that in one dimension, $\mathbf{D}_{f, \nu}^{(K)}$ coincides with $D_{f, \nu}^{(K)}$.

The next result states that the results from the previous section carry over to the sliced construction.

\begin{restatable}{theorem}{SlicedRankFConvergence}
\label{thm:sliced-rank-f-convergence}
The map $\mathbf{D}_{f, \nu}^{(K)}$ is convex.
Let $\mu,\nu \in \P(\R^d)$ with $\mu\ll\nu$. 
Then,
\begin{equation} \label{eq:SR_inequalities}
  \mathbf{D}^{(K)}_{f,\nu}(\mu)
  \le 
  \SD_{f,\nu}(\mu)
  \le 
  D_{f, \nu}(\mu), \qquad K \in \N.
\end{equation}
If $r_{\mu_s \mid \nu_s} \in \C([0, 1])$ for almost all $s \in \S^{d - 1}$, then
\begin{align*}
  \lim_{K\to\infty} \mathbf{D}^{(K)}_{f,\nu}(\mu)
  = \SD_{f,\nu}(\mu),
\end{align*}
\end{restatable}
\begin{proof}
    See \cref{subsec:proof_sliced-rank-f-convergence}.
\end{proof}

 Now, we examine the variance of the estimator $D_{f, \nu}^{(K)}$
when estimating its input by samples.
\begin{restatable}[Asymptotic normality, sliced one-sample case]{theorem}{AsymptoticNormalitySliced}
\label{thm:asymptotic_normality_sliced}
Fix $K \in \N$ and $\mu,\nu \in \P(\R^d)$ with $\mu \neq \nu$, and let
$\hat\mu_N$ be the empirical approximation from \eqref{eq:empirical_measure}.
If $f \in \C^1([0, K + 1])$ and $f'$ is Lipschitz, then there exists a constant $\tau_{K}^2 \ge 0$ such that, in distribution,
\[
  \sqrt{N}\Bigl(
    \mathbf{D}^{(K)}_{f,\nu}(\hat\mu_N)
    - \mathbf{D}^{(K)}_{f,\nu}(\mu)
  \Bigr)
  \xrightarrow[N \to \infty]{d} \mathcal{N}(0,\tau_{K}^2).
\]
\end{restatable}

\begin{proof}
    See \cref{subsec:proof_asymptotic_normality}, where we also give a criterion ensuring $\tau_K^2 > 0$.
\end{proof}

The next result gives a simple asymptotic rule for \textit{choosing $K$} in the one-dimensional oracle-reference setting.

\begin{restatable}[Asymptotic oracle choice of the rank resolution]{theorem}{AsymptoticChoiceK}
\label{thm:asymptotic-choice-K}
Let $\mu,\nu\in\P(\R)$, assume that $\nu$ is atomless and $\mu\ll\nu$.
Assume that $r \coloneqq r_{\mu \mid \nu}\in\C^2([0,1])$ and that $f\in\C^2([0,\infty))$ with bounded second derivative. If $K=K_N\to\infty$ satisfies $\frac{\sqrt N}{K_N}\to0$ and $\frac{K_N}{N}\to0$, e.g., if $K_N=N^\beta$ with $1/2<\beta<1$,
then, for $X\sim\mu$ and $\widetilde D_N^{(K)} \coloneqq D^{(K)}_{f,\nu}(\hat\mu_N)$,
\[
  \frac{\widetilde D_N^{(K_N)} - D_{f,\nu}(\mu)}{N^{-\frac{1}{2}}}
  \xrightarrow[N\to\infty]{d}
  \mathcal N\left(
    0,
    \Var_\mu\!\left(f'(r(R_\nu(X)))\right)
  \right).
\]
\end{restatable}

\begin{proof}
    See \cref{subsec:proof_asymptotic_choice_K}.
\end{proof}

\section{Experiments}
\label{sec:experiments}

We evaluate the proposed rank-statistic $f$-divergence estimator across synthetic and high-dimensional settings.
Our experiments quantify estimation accuracy, sensitivity to the resolution parameter $K$, and the benefits of the sliced extension. We also demonstrate its practical behavior when used as a fully sample-based objective in downstream learning  on the CIFAR-10 dataset \cite{krizhevsky2009learning}.
We defer additional experiments to \cref{sec:experiments_Appendix}.

\subsection{Neural vs.\ rank-statistic divergence estimation across dimensions}

We benchmark the proposed rank-statistic $f$-divergence estimator on the same synthetic setup as the neural KL-divergence estimator of \citet{sreekumar2021non}, using their training protocol for all sample sizes (optimizer, architecture scaling, and training schedule), taking $\mu$ and $\nu$ as (suitably truncated) a standard Gaussian and a uniform distribution, respectively (details are deferred to \cref{app:Neural vs. rank-statistic divergence estimation across dimensions}).

In contrast to the neural baseline, our rank-statistic estimator involves no iterative optimization: once the samples are fixed, it is fully determined by the rank resolution $K$.

Since both $\mu$ and $\nu$ factorize over coordinates (independent Gaussian components and a product-box truncation), 
\[
\mathrm{KL}(\mu\|\nu)
=\sum_{j=1}^d \mathrm{KL}(\mu_j\|\nu_j).
\]
Accordingly, an axis-corrected rank estimator is used: compute the 1D degree-$K$ terms
$D^{(K)}_{f_\mathrm{KL},\nu_j}(\mu_j)$ and sum them up:
\[
D^{(K)}_{\mathrm{KL,axis}}(\mu\|\nu)
\coloneqq
\sum_{j=1}^d D^{(K)}_{f_\mathrm{KL},\nu_j}(\mu_j).
\]
This leverages the fact that the coordinate axes already capture the discrepancy, without averaging over random projections.

\paragraph{Evaluation and plots.}

For each dimension, Figure~\ref{fig:kl_vs_n_d2_d5_d10} reports the estimated $\mathrm{KL}(\mu\|\nu)$ as a function of the sample size $n$, for both the neural baseline and the rank-statistic estimator. The ground-truth $\mathrm{KL}(\mu\|\nu)$ (dashed horizontal line) is computed analytically (implementation details are deferred to \cref{app:Neural vs. rank-statistic divergence estimation across dimensions}).

Across all $d\in\{2,5,10\}$, the rank-statistic estimator tracks the analytic reference closely and becomes increasingly stable as $n$ grows, with the uncertainty band contracting rapidly; in particular, it is already accurate in the smaller-$n$ regime (most noticeably for $d=5$ and $d=10$), where it provides a useful signal before the neural baseline has stabilized. The neural estimator exhibits larger variability and more noticeable deviations from the reference, especially in higher dimensions, suggesting that neural $f$-divergence estimation requires a larger sample budget to become reliable in this setting. Overall, these results indicate that the rank-statistic approach is competitive on this benchmark, often offering smoother and more data-efficient estimates while avoiding iterative training.


\definecolor{niceblue}{RGB}{31,119,180}
\definecolor{niceorange}{RGB}{255,127,14}
\definecolor{nicegray}{RGB}{80,80,80}

\begin{figure*}[!t]
\centering

\begin{subfigure}[b]{0.32\linewidth}
\centering
\begin{tikzpicture}
\begin{axis}[
    width=\linewidth, height=5.0cm,
    xmode=log, log basis x=10,
    xmin=1e4, xmax=6.4e6,
    ymin=0.130, ymax=0.145,
    xtick={1e4,1e5,1e6,5e6},
    xticklabels={$10^4$,$10^5$,$10^6$,$5\cdot10^6$},
    ytick={0.130,0.133,0.136,0.139,0.142,0.145},
    yticklabel style={/pgf/number format/fixed, /pgf/number format/precision=3},
    ylabel={Estimated $\mathrm{KL}(\mu\|\nu)$},
    grid=major, grid style={dashed, gray!40},
    axis line style={black!70},
    tick label style={font=\footnotesize},
    ylabel style={font=\small, yshift=-2pt},
]
    \addplot[name path=neural_up, draw=none, forget plot] table[row sep=\\] {
    10000 0.1382\\ 20000 0.1381\\ 40000 0.1385\\ 80000 0.1386\\ 100000 0.1447\\ 200000 0.1405\\ 400000 0.1407\\ 800000 0.1402\\ 1600000 0.1406\\ 3200000 0.1390\\ 6400000 0.1382\\ };
    \addplot[name path=neural_lo, draw=none, forget plot] table[row sep=\\] {
    10000 0.1340\\ 20000 0.1361\\ 40000 0.1345\\ 80000 0.1359\\ 100000 0.1323\\ 200000 0.1354\\ 400000 0.1368\\ 800000 0.1369\\ 1600000 0.1370\\ 3200000 0.1379\\ 6400000 0.1375\\ };
    \addplot[niceblue, fill opacity=0.15, forget plot] fill between[of=neural_up and neural_lo];
    \addplot[niceblue, thick, mark=*, mark size=2pt] table[row sep=\\] {
    10000 0.1361\\ 20000 0.1371\\ 40000 0.1365\\ 80000 0.1373\\ 100000 0.1395\\ 200000 0.1380\\ 400000 0.1387\\ 800000 0.1386\\ 1600000 0.1388\\ 3200000 0.1384\\ 6400000 0.1379\\ };

    \addplot[name path=rank_up, draw=none, forget plot] table[x expr=\thisrowno{0}, y expr=\thisrowno{1}+\thisrowno{2}, row sep=\\] {
    10000 0.1379 0.0048\\ 20000 0.1371 0.0046\\ 40000 0.1372 0.0029\\ 80000 0.1377 0.0022\\ 160000 0.1365 0.0019\\ 320000 0.1369 0.0014\\ 640000 0.1372 0.0006\\ 1280000 0.1372 0.0003\\ 2560000 0.1371 0.0002\\ 5120000 0.1371 0.0004\\ };
    
    \addplot[name path=rank_lo, draw=none, forget plot] table[x expr=\thisrowno{0}, y expr=\thisrowno{1}-\thisrowno{2}, row sep=\\] {
    10000 0.1379 0.0048\\ 20000 0.1371 0.0046\\ 40000 0.1372 0.0029\\ 80000 0.1377 0.0022\\ 160000 0.1365 0.0019\\ 320000 0.1369 0.0014\\ 640000 0.1372 0.0006\\ 1280000 0.1372 0.0003\\ 2560000 0.1371 0.0002\\ 5120000 0.1371 0.0004\\ };
    
    \addplot[niceorange, fill opacity=0.15, forget plot] fill between[of=rank_up and rank_lo];
    
    \addplot[niceorange, thick, mark=square*, mark size=2pt] table[x expr=\thisrowno{0}, y expr=\thisrowno{1}, row sep=\\] {
    10000 0.1379 0.0048\\ 20000 0.1371 0.0046\\ 40000 0.1372 0.0029\\ 80000 0.1377 0.0022\\ 160000 0.1365 0.0019\\ 320000 0.1369 0.0014\\ 640000 0.1372 0.0006\\ 1280000 0.1372 0.0003\\ 2560000 0.1371 0.0002\\ 5120000 0.1371 0.0004\\ };

    \addplot[nicegray, very thick, dashed] coordinates {(1e4,0.138189) (6.4e6,0.138189)};
\end{axis}
\end{tikzpicture}
\caption{$d=2$}
\end{subfigure}\hfill%
\begin{subfigure}[b]{0.32\linewidth}
\centering
\begin{tikzpicture}
\begin{axis}[
    width=\linewidth, height=5.0cm,
    xmode=log, log basis x=10,
    xmin=1e4, xmax=6.4e6,
    ymin=0.34, ymax=0.42,
    xtick={1e4,1e5,1e6,5e6},
    xticklabels={$10^4$,$10^5$,$10^6$,$5\cdot10^6$},
    ytick={0.34, 0.36, 0.38, 0.40, 0.42},
    yticklabel style={/pgf/number format/fixed, /pgf/number format/precision=2},
    grid=major, grid style={dashed, gray!40},
    axis line style={black!70},
    tick label style={font=\footnotesize},
]
    \addplot[name path=neural_up, draw=none, forget plot] table[row sep=\\] {
    10000 0.3813\\ 20000 0.3810\\ 40000 0.3794\\ 80000 0.3813\\ 160000 0.3889\\ 320000 0.3916\\ 640000 0.3940\\ 1280000 0.3935\\ };
    \addplot[name path=neural_lo, draw=none, forget plot] table[row sep=\\] {
    10000 0.3699\\ 20000 0.3698\\ 40000 0.3714\\ 80000 0.3730\\ 160000 0.3785\\ 320000 0.3830\\ 640000 0.3809\\ 1280000 0.3850\\ };
    \addplot[niceblue, fill opacity=0.15, forget plot] fill between[of=neural_up and neural_lo];
    \addplot[niceblue, thick, mark=*, mark size=2pt] table[row sep=\\] {
    10000 0.3756\\ 20000 0.3754\\ 40000 0.3754\\ 80000 0.3772\\ 160000 0.3837\\ 320000 0.3873\\ 640000 0.3875\\ 1280000 0.3892\\ 2560000 0.3912\\ 5120000 0.3923\\ };

    \addplot[name path=rank_up, draw=none, forget plot] table[row sep=\\] {
    10000 0.3924\\ 20000 0.3927\\ 40000 0.3949\\ 80000 0.3949\\ 160000 0.3925\\ 320000 0.3915\\ 640000 0.3910\\ 1280000 0.3904\\ 2560000 0.3903\\ 5120000 0.3900\\ };
    \addplot[name path=rank_lo, draw=none, forget plot] table[row sep=\\] {
    10000 0.3850\\ 20000 0.3808\\ 40000 0.3816\\ 80000 0.3859\\ 160000 0.3854\\ 320000 0.3869\\ 640000 0.3874\\ 1280000 0.3886\\ 2560000 0.3892\\ 5120000 0.3891\\ };
    \addplot[niceorange, fill opacity=0.15, forget plot] fill between[of=rank_up and rank_lo];
    \addplot[niceorange, thick, mark=square*, mark size=2pt] table[row sep=\\] {
    10000 0.3920\\ 20000 0.3868\\ 40000 0.3882\\ 80000 0.3904\\ 160000 0.3890\\ 320000 0.3892\\ 640000 0.3892\\ 1280000 0.3895\\ 2560000 0.3897\\ 5120000 0.3896\\ };

    \addplot[nicegray, very thick, dashed] coordinates {(1e4,0.392526) (6.4e6,0.392526)};
\end{axis}
\end{tikzpicture}
\caption{$d=5$}
\end{subfigure}\hfill%
\begin{subfigure}[b]{0.32\linewidth}
\centering
\begin{tikzpicture}
\begin{axis}[
    width=\linewidth, height=5.0cm,
    xmode=log, log basis x=10,
    xmin=1e4, xmax=6.4e6,
    ymin=0.73, ymax=0.81,
    xtick={1e4,1e5,1e6,5e6},
    xticklabels={$10^4$,$10^5$,$10^6$,$5\cdot10^6$},
    ytick={0.73,0.75,0.77,0.79,0.81},
    yticklabel style={/pgf/number format/fixed, /pgf/number format/precision=2},
    grid=major, grid style={dashed, gray!40},
    axis line style={black!70},
    tick label style={font=\footnotesize},
]
    \addplot[name path=neural_up, draw=none, forget plot] table[x expr=\thisrowno{0}, y expr=\thisrowno{1}+\thisrowno{2}, row sep=\\] {
    10000 0.7435 0.0062\\ 20000 0.7476 0.0040\\ 40000 0.7432 0.0073\\ 80000 0.7494 0.0067\\ 160000 0.7524 0.0075\\ 320000 0.7539 0.0087\\ 640000 0.7602 0.0085\\ 1280000 0.7630 0.0079\\ 2560000 0.7607 0.0050\\ 5120000 0.7669 0.0098\\ };
    \addplot[name path=neural_lo, draw=none, forget plot] table[x expr=\thisrowno{0}, y expr=\thisrowno{1}-\thisrowno{2}, row sep=\\] {
    10000 0.7435 0.0062\\ 20000 0.7476 0.0040\\ 40000 0.7432 0.0073\\ 80000 0.7494 0.0067\\ 160000 0.7524 0.0075\\ 320000 0.7539 0.0087\\ 640000 0.7602 0.0085\\ 1280000 0.7630 0.0079\\ 2560000 0.7607 0.0050\\ 5120000 0.7669 0.0098\\ };
    \addplot[niceblue, fill opacity=0.15, forget plot] fill between[of=neural_up and neural_lo];
    \addplot[niceblue, thick, mark=*, mark size=2pt] table[row sep=\\] {
    10000 0.7435\\ 20000 0.7476\\ 40000 0.7432\\ 80000 0.7494\\ 160000 0.7524\\ 320000 0.7539\\ 640000 0.7602\\ 1280000 0.7630\\ 2560000 0.7607\\ 5120000 0.7669\\ };

    \addplot[name path=rank_up, draw=none, forget plot] table[x expr=\thisrowno{0}, y expr=\thisrowno{1}+\thisrowno{2}, row sep=\\] {
    10000 0.7880 0.0214\\ 20000 0.7841 0.0133\\ 40000 0.7817 0.0081\\ 80000 0.7764 0.0078\\ 160000 0.7780 0.0039\\ 320000 0.7781 0.0032\\ 640000 0.7783 0.0027\\ 1280000 0.7785 0.0014\\ 2560000 0.7781 0.0005\\ 5120000 0.7784 0.0007\\ };
    \addplot[name path=rank_lo, draw=none, forget plot] table[x expr=\thisrowno{0}, y expr=\thisrowno{1}-\thisrowno{2}, row sep=\\] {
    10000 0.7880 0.0214\\ 20000 0.7841 0.0133\\ 40000 0.7817 0.0081\\ 80000 0.7764 0.0078\\ 160000 0.7780 0.0039\\ 320000 0.7781 0.0032\\ 640000 0.7783 0.0027\\ 1280000 0.7785 0.0014\\ 2560000 0.7781 0.0005\\ 5120000 0.7784 0.0007\\ };
    \addplot[niceorange, fill opacity=0.15, forget plot] fill between[of=rank_up and rank_lo];
    \addplot[niceorange, thick, mark=square*, mark size=2pt] table[row sep=\\] {
    10000 0.7880\\ 20000 0.7841\\ 40000 0.7817\\ 80000 0.7764\\ 160000 0.7780\\ 320000 0.7781\\ 640000 0.7783\\ 1280000 0.7785\\ 2560000 0.7781\\ 5120000 0.7784\\ };

    \addplot[nicegray, very thick, dashed] coordinates {(1e4,0.785051) (6.4e6,0.785051)};
\end{axis}
\end{tikzpicture}
\caption{$d=10$}
\end{subfigure}

\vspace{0.5em}
\begin{tikzpicture}
    \node[anchor=west] at (0,0) (legend1) {
        \textcolor{niceblue}{\large$\bullet$} \small Neural estimation
    };
    \node[anchor=west] at (2.8,0) (legend2) {
        \textcolor{niceorange}{\rule{1.2ex}{1.2ex}} \small rank-statistic KL (ours)}
    ;
    \draw[nicegray, very thick, dashed] (6.4,0) -- (7,0);
    \node[anchor=west] at (7,0) (legend3) {\small True KL};
\end{tikzpicture}

\caption{\small Convergence of Kullback-Leibler divergence estimates for increasing sample size~$n$, averaged over 10 independent runs. Shaded bands denote the $\pm 1$ standard deviation interval. Results are shown for $K = 64$ across all samples.}
\label{fig:kl_vs_n_d2_d5_d10}
\end{figure*}

\subsection{Univariate empirical convergence and the influence of the resolution parameter $K$}
\label{subsec:onedim-exp}

This subsection benchmarks the one-dimensional rank-statistic estimator against standard $f$-divergences in settings where accurate reference values are available, and studies how the finite resolution parameter $K$ controls the approximation gap. We focus on three widely used discrepancies, Kullback-Leibler (KL), Jensen-Shannon (JS) \cite{L1991}, and the squared Hellinger divergence, and consider four representative mismatch families: (i) Gaussian mean shifts: $\mathcal N(0,1)$ vs.\ $\mathcal N(\Delta,1)$ with $\Delta\in\{0,0.5,1,2\}$ (JS and KL); (ii) Gaussian scale changes: $\mathcal N(0,1)$ vs.\ $\mathcal N(0,\sigma)$ with $\sigma\in\{1,1.2,1.5,2\}$ (KL and squared Hellinger); (iii) a symmetric Gaussian mixture: $\tfrac12\mathcal N(-\Delta,1)+\tfrac12\mathcal N(+\Delta,1)$ vs.\ $\mathcal N(0,1)$ (JS); and (iv) a tail-mismatch case of $\mathrm{Laplace}(0,1)$ vs.\ $\mathcal N(0,1)$ (JS).

Reference values use closed forms when available (Gaussian-Gaussian KL and squared Hellinger), and otherwise a high-accuracy one-dimensional numerical/Monte Carlo reference; full details are deferred to Appendix~\ref{app:One-Dimensional Approximation Accuracy}. Unless stated otherwise, the number of samples is $n_\mu=n_\nu=10{,}000$ and results are averaged over $10$ seeds. Table~\ref{tab:1d_results_refined_1} reports the ratio $D^{(K)}_{f,\nu}(\mu)/D_{f,\nu}(\mu)$ for $K\in\{32,64,128,256,512\}$. 
For the mean shift experiments, the Jensen-Shannon (JS) divergence outperforms the KL, while for the scale change experiments, the KL outperforms the squared Hellinger divergence.
Additional figures and sweeps over $(n,K)$ are provided in \cref{app:One-Dimensional Approximation Accuracy}.

\sisetup{
  separate-uncertainty = true,
  uncertainty-separator = {\,\pm\,},
  table-format = 1.3(3), 
  detect-all
}

\begin{table*}[!t]
\centering
\footnotesize
\renewcommand{\arraystretch}{1.00} 
\setlength{\tabcolsep}{6pt}

\begin{tabular}{@{} l l l S S S S S @{}}
\toprule
& & & \multicolumn{5}{c}{Ratio $D_{f,\nu}^{(K)}(\mu)  / D_{f, \nu}(\mu)$ for $K=$} \\
\cmidrule(lr){4-8}
Family & Scen. & Param. & {32} & {64} & {128} & {256} & {512} \\
\midrule

\multirow{6}{*}{Mean shift} 
  & JS & $\Delta=0.5$ & 0.933(40) & 0.968(41) & 0.989(42) & \bftablenum{1.003(42)} & 1.013(42) \\
  & JS & $\Delta=1.0$ & 0.928(33) & 0.961(34) & 0.981(35) & 0.992(35) & \bftablenum{0.999(35)} \\
  & JS & $\Delta=2.0$ & 0.930(08) & 0.962(08) & 0.981(09) & 0.991(09) & \bftablenum{0.997(09)} \\
\addlinespace[4pt]
  & KL & $\Delta=0.5$ & 0.946(60) & 0.987(63) & \bftablenum{1.013(65)} & 1.030(66) & 1.044(68) \\
  & KL & $\Delta=1.0$ & 0.880(24) & 0.924(25) & 0.952(26) & 0.969(27) & \bftablenum{0.980(27)} \\
  & KL & $\Delta=2.0$ & 0.775(10) & 0.844(12) & 0.895(13) & 0.933(15) & \bftablenum{0.959(16)} \\

\midrule

\multirow{6}{*}{Scale change} 
  & KL & $\sigma=1.2$ & 0.743(63) & 0.841(70) & 0.908(72) & 0.954(72) & \bftablenum{0.991(72)} \\
  & KL & $\sigma=1.5$ & 0.779(27) & 0.872(29) & 0.927(30) & 0.958(31) & \bftablenum{0.977(31)} \\
  & KL & $\sigma=2.0$ & 0.803(18) & 0.898(20) & 0.953(21) & 0.982(22) & \bftablenum{0.998(22)} \\
\addlinespace[4pt]
  & Hell$^2$ & $\sigma=1.2$ & 0.741(77) & 0.853(89) & 0.931(98) & 0.986(106) & \bftablenum{1.029(111)} \\
  & Hell$^2$ & $\sigma=1.5$ & 0.735(35) & 0.842(39) & 0.908(41) & 0.948(42) & \bftablenum{0.973(42)} \\
  & Hell$^2$ & $\sigma=2.0$ & 0.744(14) & 0.858(14) & 0.926(14) & 0.965(13) & \bftablenum{0.987(12)} \\

\midrule

\multirow{3}{*}{Multimodal} 
  & JS & $\Delta=0.5$ & 0.746(157) & 0.849(176) & 0.926(189) & \bftablenum{0.994(196)} & 1.068(199) \\
  & JS & $\Delta=1.0$ & 0.769(38) & 0.849(40) & 0.898(41) & 0.929(41) & \bftablenum{0.948(42)} \\
  & JS & $\Delta=2.0$ & 0.846(19) & 0.912(20) & 0.950(21) & 0.972(21) & \bftablenum{0.985(21)} \\

\midrule

Heavy tails & JS & -- & 0.488(28) & 0.651(36) & 0.778(41) & 0.869(44) & \bftablenum{0.933(46)} \\
\bottomrule
\end{tabular}
\caption{\small 1D divergence estimation benchmarks (10 runs). We report the ratio estimate/reference (mean $\pm$ std) for various $K$ values. Boldface highlights, for each row, the $K$ whose mean ratio is closest to $1$ (i.e., the most accurate approximation). Ratios above one arise from finite-sample Monte Carlo/numerical-reference error and do not contradict the population lower-bound property \eqref{eq:monotonicity}.}
\label{tab:1d_results_refined_1}
\end{table*}

\subsection{Sliced rank-statistic $f$-divergences: Empirical convergence}
\label{sec:exp_sliced_nd}

We consider $d$-dimensional benchmarks using the sliced estimator $D^{(K)}_{f,\nu}(\mu)$, obtained by averaging the one-dimensional rank divergence over $L$ random projections. Unless stated otherwise, $K=64$, $L=128$, $n_\mu=n_\nu=10{,}000$, and results are reported as mean$\pm$std over $R=10$ runs.
Implementation details are deferred to Appendix~\ref{app:Multivariate benchmarks via random projections}.

Three settings are considered: (i) Gaussian-Gaussian pairs, (ii) covariance mismatches (isotropic and anisotropic), and (iii) non-Gaussian pairs.
For (i), $\mathrm{KL}$ and squared Hellinger have closed-form references, while $\mathrm{JS}$ is approximated by a moment-matched Gaussian proxy.
For (iii), reference values are obtained by Monte Carlo evaluation of the divergence formula using closed-form log-densities.

Figure~\ref{fig:three_plots_wide} reports the ratio $d\,D^{(K)}_{f,\nu}(\mu)/D_{f,\nu}(\mu)$ for mean-shift benchmarks across several dimensions. Overall, the ratio stays close to one with moderate variability, indicating that
the simple $d$-scaling provides a reasonable normalization in these settings. Deviations become more noticeable in higher
dimension, especially for JS and KL, suggesting that a fixed number of projections $L$ can lead to mild under/over-estimation
as $d$ grows, while squared Hellinger remains comparatively stable. Additional benchmarks and ablations are reported in Appendix~\ref{app:Multivariate benchmarks via random projections}.

\definecolor{myBlue}{RGB}{0, 114, 178}
\definecolor{myRed}{RGB}{213, 94, 0}
\definecolor{myGreen}{RGB}{0, 158, 115}
\definecolor{myBlack}{RGB}{0, 0, 0}
\definecolor{myGray}{RGB}{128, 128, 128}

\pgfplotstableread{
d mean std
2  1.015 0.067
5  1.098 0.052
10 1.188 0.072
20 0.977 0.060
50 1.304 0.044
}\KLdHalf
\pgfplotstableread{
d mean std
2  0.991 0.030
5  1.087 0.032
10 1.170 0.031
20 0.899 0.028
50 1.113 0.042
}\KLdOne
\pgfplotstableread{
d mean std
2  1.005 0.062
5  0.935 0.048
10 0.973 0.049
20 1.002 0.037
50 1.035 0.060
}\HelldHalf
\pgfplotstableread{
d mean std
2  0.972 0.033
5  0.931 0.022
10 0.965 0.048
20 0.943 0.042
50 0.851 0.021
}\HelldOne
\pgfplotstableread{
d mean std
2  1.007 0.076
5  0.899 0.040
10 0.887 0.031
20 1.006 0.060
50 1.234 0.038
}\JSdHalf
\pgfplotstableread{
d mean std
2  0.979 0.043
5  0.892 0.013
10 0.895 0.035
20 0.951 0.041
50 1.144 0.031
}\JSdOne

\begin{figure*}[t]
\centering
\begin{tikzpicture}
    \begin{groupplot}[
      group style={
          group size=3 by 1,
          horizontal sep=1.1cm, 
      },
      width=0.32\textwidth,  
      height=5.0cm,
      xmode=log, log ticks with fixed point,
      xmin=1.5, xmax=60,
      xtick={2,5,10,20,50},
      xlabel={Dimension $d$},
      grid=major, grid style={dashed, gray!30},
      axis line style={black!70}, tick style={black!70},
      tick label style={font=\footnotesize},
      label style={font=\small},
      title style={at={(0.5,-0.35)}, anchor=north, font=\small},
    ]

    \nextgroupplot[title={(a) Mean (KL)}, ylabel={Ratio}]
    \addplot[myGray, dashed, thick, domain=1.5:60] {1};
    \addplot+[myBlue, thick, mark=*, mark size=2pt, error bars/.cd, y dir=both, y explicit] table[x=d, y=mean, y error=std] {\KLdHalf};
    \addplot+[myRed, thick, mark=square*, mark size=2pt, error bars/.cd, y dir=both, y explicit] table[x=d, y=mean, y error=std] {\KLdOne};

    \nextgroupplot[
        title={(b) Mean (Hell$^2$)}, 
        legend columns=-1,
        legend style={
            draw=none, font=\small,
            at={(0.5,-0.5)}, anchor=north, 
            /tikz/every even column/.append style={column sep=0.5cm}
        }
    ]
    \addplot[myGray, dashed, thick, domain=1.5:60] {1}; \addlegendentry{Ideal}
    \addplot+[myBlue, thick, mark=*, mark size=2pt, error bars/.cd, y dir=both, y explicit] table[x=d, y=mean, y error=std] {\HelldHalf}; \addlegendentry{$\Delta=0.5$}
    \addplot+[myRed, thick, mark=square*, mark size=2pt, error bars/.cd, y dir=both, y explicit] table[x=d, y=mean, y error=std] {\HelldOne}; \addlegendentry{$\Delta=1.0$}

    \nextgroupplot[title={(c) Mean (JS)}] 
    \addplot[myGray, dashed, thick, domain=1.5:60] {1};
    \addplot+[myBlue, thick, mark=*, mark size=2pt, error bars/.cd, y dir=both, y explicit] table[x=d, y=mean, y error=std] {\JSdHalf};
    \addplot+[myRed, thick, mark=square*, mark size=2pt, error bars/.cd, y dir=both, y explicit] table[x=d, y=mean, y error=std] {\JSdOne};

    \end{groupplot}
\end{tikzpicture}
\caption{\small Comparison of mean shift metrics across dimensions for KL, Hellinger, and JS divergences.}
\label{fig:three_plots_wide}
\end{figure*}

\paragraph{Practical choice of $K$ and $L$.}

The rank resolution $K$ controls the one-dimensional approximation--estimation trade-off. Small values of $K$ lead to a coarse but statistically stable rank histogram, while larger values provide a finer approximation at the cost of increased finite-sample variability. This mirrors the classical binning trade-off in calibration-error estimation: coarse binning can average out structured departures from uniformity, whereas overly fine binning leaves too few samples per bin and increases statistical
noise \citep{roelofs2022mitigating}. In practice, this motivates using moderate values of $K$, or progressive schedules that start from a coarse resolution and increase $K$ as the coarse rank histogram stabilizes during training.

In the oracle-reference one-dimensional setting, \cref{thm:asymptotic-choice-K} formalizes this trade-off asymptotically: under $C^2$ regularity, any $K_N=N^\beta$ with $1/2<\beta<1$ removes the Bernstein bias at the $N^{-1/2}$ scale while keeping the nonlinear plug-in remainder negligible.

The number of projections $L$ controls the Monte Carlo approximation of the
spherical average. For fixed $K$, this projection error behaves like a
standard Monte Carlo error of order $O(L^{-1/2})$ under bounded directional
discrepancies, as in general sliced probability divergences
\citep{nadjahi2020sliced}. Moderate values of $L$ can be effective when
the discrepancy is diffuse across directions, whereas localized or highly
anisotropic discrepancies may require substantially more projections. Similar
projection-complexity effects have also been observed in sliced Wasserstein
generative models and sliced-Wasserstein flows
\citep{wu2019swgm,liutkus2019sliced}. In such cases, structured or
variance-reduced directions, such as orthogonal or quasi-Monte Carlo
projections, may be preferable.

\subsection{Generative Transport Dynamics for Rank $f$-Divergences}\label{subsec:flows}

A useful way to turn a discrepancy into a learning principle is to interpret it as an
energy and derive an update rule that transports samples in data space toward a target
distribution. In our setting, the energy is the (sliced) rank $f$-divergence, and we
implement its minimization through particle transport dynamics based on one-dimensional
quantile matching.

Given particles $\{x_i\}_{i=1}^N\sim \mu$ and reference samples
$\{y_j\}_{j=1}^M\sim \nu$, we draw directions
$s_1,\dots,s_L\in \S^{d-1}$ and form the one-dimensional projections
\[
x_i^{(\ell)} \coloneqq \langle x_i,s_\ell\rangle,
\qquad
y_j^{(\ell)} \coloneqq \langle y_j,s_\ell\rangle.
\]
For each slice $\ell$, let
$\nu^{(\ell)}\coloneqq (s_\ell)_{\#}\nu$ denote the projected reference law, and let $\widehat\nu^{(\ell)} \coloneqq \frac{1}{M}\sum_{j=1}^M \delta_{y_j^{(\ell)}}$
be its empirical counterpart. We compute the initial soft rank coordinates
in slice $\ell$, where the superscript $(\ell)$ indexes the projection
direction and the subscript $0$ denotes the initial rank-space state:
\[
v^{(\ell)}_{0,i}
\approx
\widehat R_{\nu^{(\ell)},\tau}
\bigl(x_i^{(\ell)}\bigr)
\in [0,1],
\qquad i=1,\dots,N.
\]
Equivalently, we write $\mathbf v^{(\ell)}_0
=
\bigl(v^{(\ell)}_{0,1},\dots,v^{(\ell)}_{0,N}\bigr)
\in [0,1]^N$.
Here, $\widehat R_{\nu^{(\ell)},\tau}$ denotes the smoothed empirical CDF
constructed from the projected reference samples
$\{y_j^{(\ell)}\}_{j=1}^M$, obtained by replacing the indicator functions in
the empirical CDF by logistic soft indicators, as in standard differentiable
rank approximations.

To refine these rank coordinates, we associate with any vector
$\mathbf v=(v_1,\dots,v_N)\in [0,1]^N$ a discrete PMF on
$\{0,\dots,K\}$ via a Bernstein-smoothed histogram, which can be viewed as a
smooth sample-based analogue of $Q_{\mu \mid \nu}^{(K)}$:
\begin{equation}
\label{eq:bern_hist_PMF}
\widehat Q^{(K)}(\mathbf v)(n)
\;\coloneqq\;
\frac{1}{N}\sum_{i=1}^N b_{n,K}(v_i),
\qquad n=0,\dots,K.
\end{equation}
We then measure its deviation from the discrete uniform law $U_K$, as in
Definition~\ref{defn:rankStatistic_f_Divergence}:
\begin{equation}
\label{eq:bern_hist_divergence}
\D_f\bigl(\widehat Q^{(K)}(\mathbf v)\parallel U_K\bigr)
=
\frac{1}{K+1}
\sum_{n=0}^K
f\Bigl((K+1)\widehat Q^{(K)}(\mathbf v)(n)\Bigr).
\end{equation}
The rank-space refinement in slice $\ell$ is then defined by the proximal step
\begin{equation}
\label{eq:rank_prox}
\begin{split}
\mathbf v^{(\ell)}_{1}
\in
\argmin_{\mathbf v\in[0,1]^N}
\Bigg\{
&\D_f\!\bigl(\widehat Q^{(K)}(\mathbf v)\,\Vert\,U_K\bigr)
\\
&+
\frac{1}{2\eta}
\bigl\|\mathbf v-\mathbf v^{(\ell)}_{0}\bigr\|_2^2
\Bigg\},
\end{split}
\end{equation}
where $\eta>0$ controls the trust in the current rank coordinates. In
practice, \eqref{eq:rank_prox} can be approximated by deterministic updates
(e.g.\ SGD) or by Langevin-type inner samplers such as ULA or MALA.

The updated rank coordinates are mapped back to the projection axis through
the empirical quantile map of the reference slice:
\[
z_i^{(\ell)}
\approx
Q_{\widehat\nu^{(\ell)}}
\bigl(v^{(\ell)}_{1,i}\bigr),
\qquad
\Delta_i^{(\ell)}
\coloneqq
z_i^{(\ell)} - x_i^{(\ell)}.
\]
Here, $Q_{\widehat\nu^{(\ell)}}$ is evaluated by sorting the projected
reference samples and linearly interpolating between adjacent order statistics,
following the sliced-Wasserstein flow construction
\citep[Sec.~3.3 and Alg.~1]{liutkus2019sliced}. The scalar correction
$\Delta_i^{(\ell)}$ is therefore a one-dimensional monotone transport
correction in slice $\ell$. Finally, we lift these corrections back to
$\R^d$ by aggregating over slices:
\begin{equation}
\label{eq:particle_transport_update}
x_i
\;\leftarrow\;
x_i
+
\varepsilon\,\frac{d}{L}
\sum_{\ell=1}^L
\Delta_i^{(\ell)}\,s_\ell,
\end{equation}
with step size $\varepsilon>0$ and, optionally, per-slice clipping for
stability. Iterating \eqref{eq:particle_transport_update} yields a practical
transport dynamics that moves the particle cloud toward $\nu$ while being
driven by a bounded, rank-based energy. The full pseudocode is given in
Appendix~\ref{app:Generative Transport Dynamics for Rank $f$-divergence}.

\begin{remark}
    This particle algorithm resembles an explicit time discretization of a Wasserstein gradient flow of a Moreau envelope of $D_{f, \nu}^{(K)}$ (similar to \cite{SNRS2025}), the difference being that here the Moreau envelope is taken in quantile space.
\end{remark}

\subsubsection{Two-dimensional toy examples}
\label{subsubsec:2d-toys}

We illustrate the induced particle dynamics on four $2$D toy targets:
(i) a checkerboard distribution, (ii) a noisy ring, (iii) a two-spirals dataset, and
(iv) a two-component Gaussian mixture (two blobs). In each case, we draw $M$ reference samples
from the target $\nu$ and initialize $N$ particles from an isotropic Gaussian. We then iterate
\eqref{eq:particle_transport_update} for a fixed number of outer steps and report snapshots at
$t\in\{0,1,5,10,20,40,100,200,400\}$.

Figure~\ref{fig:mosaic_nice} uses an SGD approximation of the rank-proximal refinement \eqref{eq:rank_prox} with the KL generator, $L=10$ projection directions, and trust-region parameter
$\eta=0.5$. We use a moderate outer step size by starting from
$\varepsilon=0.20$ and linearly annealing it to $0.15$ over training; in parallel we anneal the
rank smoothing temperature (cf. \cref{alg:rank-prox-transport-compact}) from $\tau=0.30$ to $0.10$ and increase the rank resolution from
$K=80$ to $K=128$.

Qualitatively, the dynamics rapidly match the target geometry across very different structures.
On the checkerboard, particles populate multiple disconnected cells without degenerating to a
single region; on the ring, they expand radially and then redistribute along the angular
directions; on spirals, the cloud progressively aligns with the nonlinear manifold; and on the
two-blobs mixture, it splits and concentrates around both modes.

\graphicspath{{./}{./imgs/}{./imgs/2d_toy_examples/chess_board/}{./imgs/2d_toy_examples/two_blobs/}}

\newcommand{\imgcell}[2][]{\fbox{\includegraphics[#1]{#2}}}

\begin{figure}[t]
  \centering

  \setlength{\fboxsep}{0pt}
  \setlength{\fboxrule}{0.35pt}
  \setlength{\tabcolsep}{0pt}
  \renewcommand{\arraystretch}{0}

  \newcommand{\axiscell}[1]{%
    \begin{minipage}{1.55cm}%
      \centering
      \hrule width 1.55cm height 0.4pt depth 0pt
      \hbox to 1.55cm{\hss\vrule width 0.5pt height 3pt\hss}
      \vspace{2pt}
      {\sffamily\tiny #1}
    \end{minipage}%
  }

  \begin{adjustbox}{max width=\columnwidth,center}
    \begin{tabular}{@{} *{9}{c} @{}}

    \imgcell[width=1.55cm]{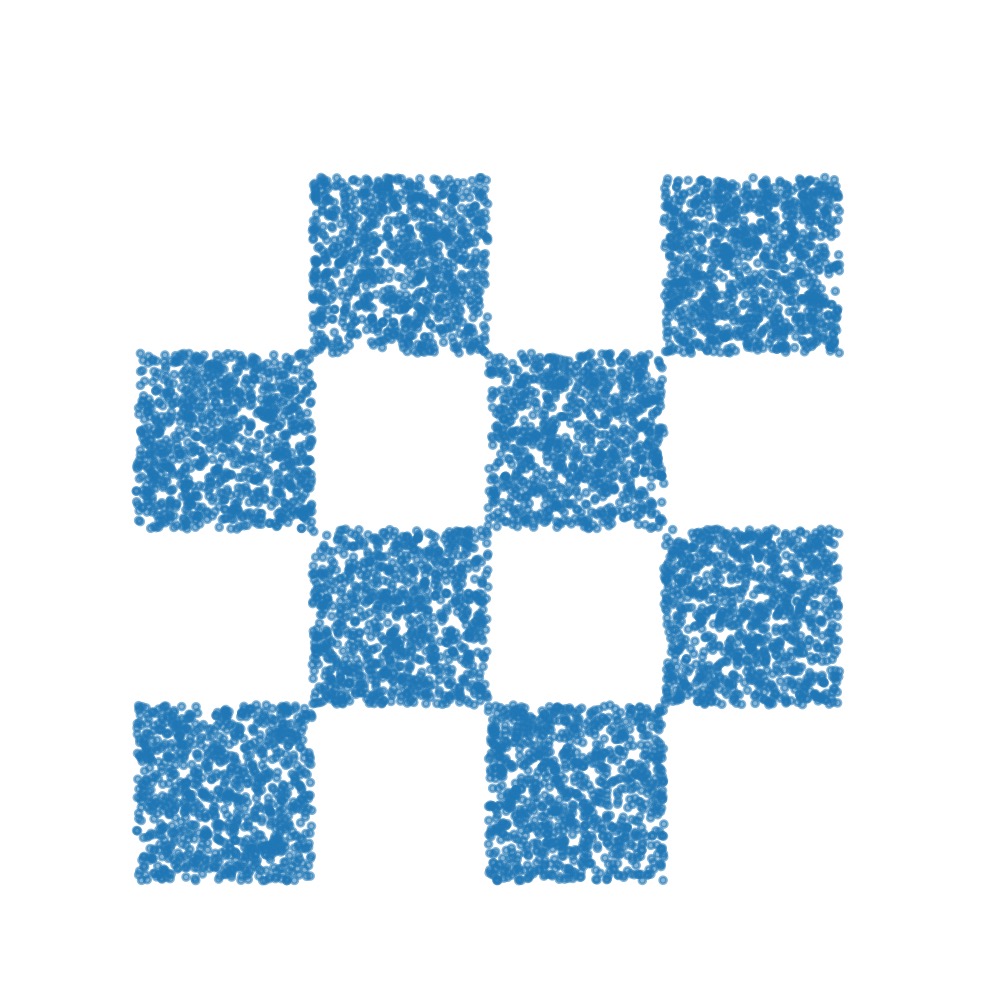} &
    \imgcell[width=1.55cm]{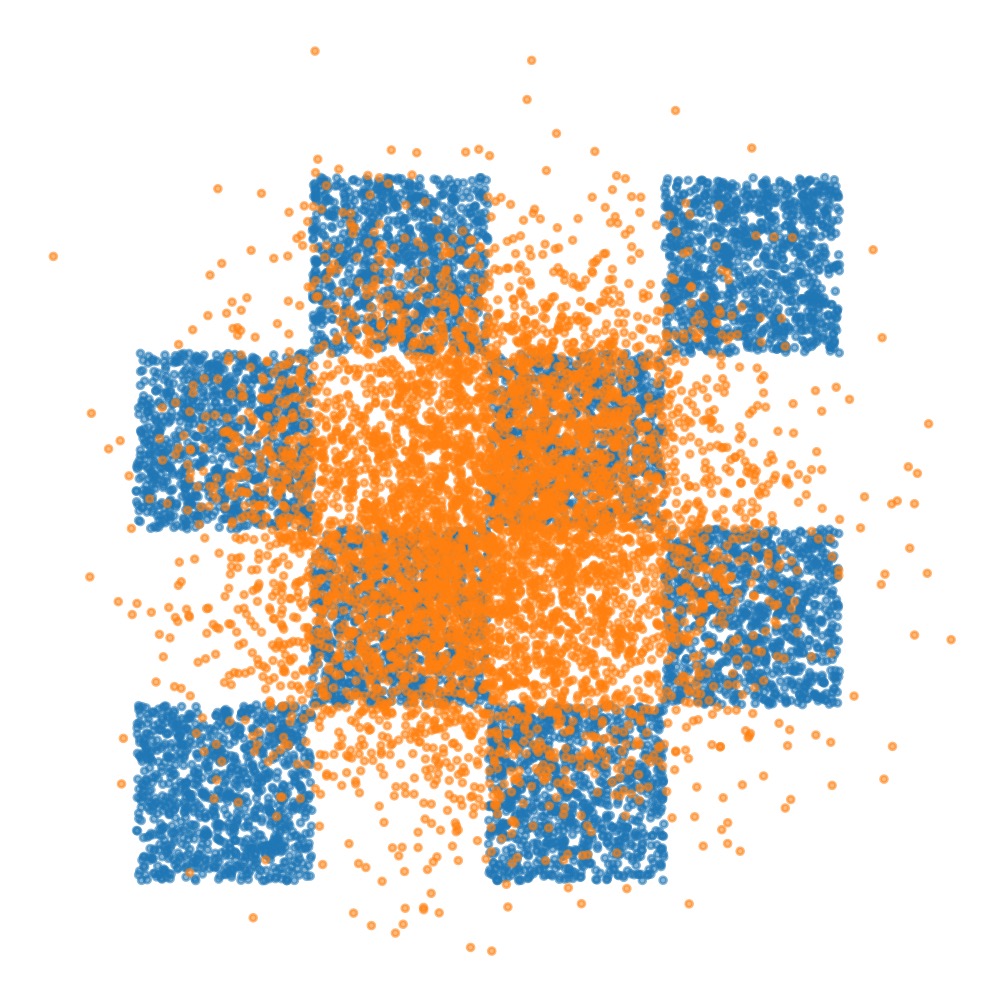} &
    \imgcell[width=1.55cm]{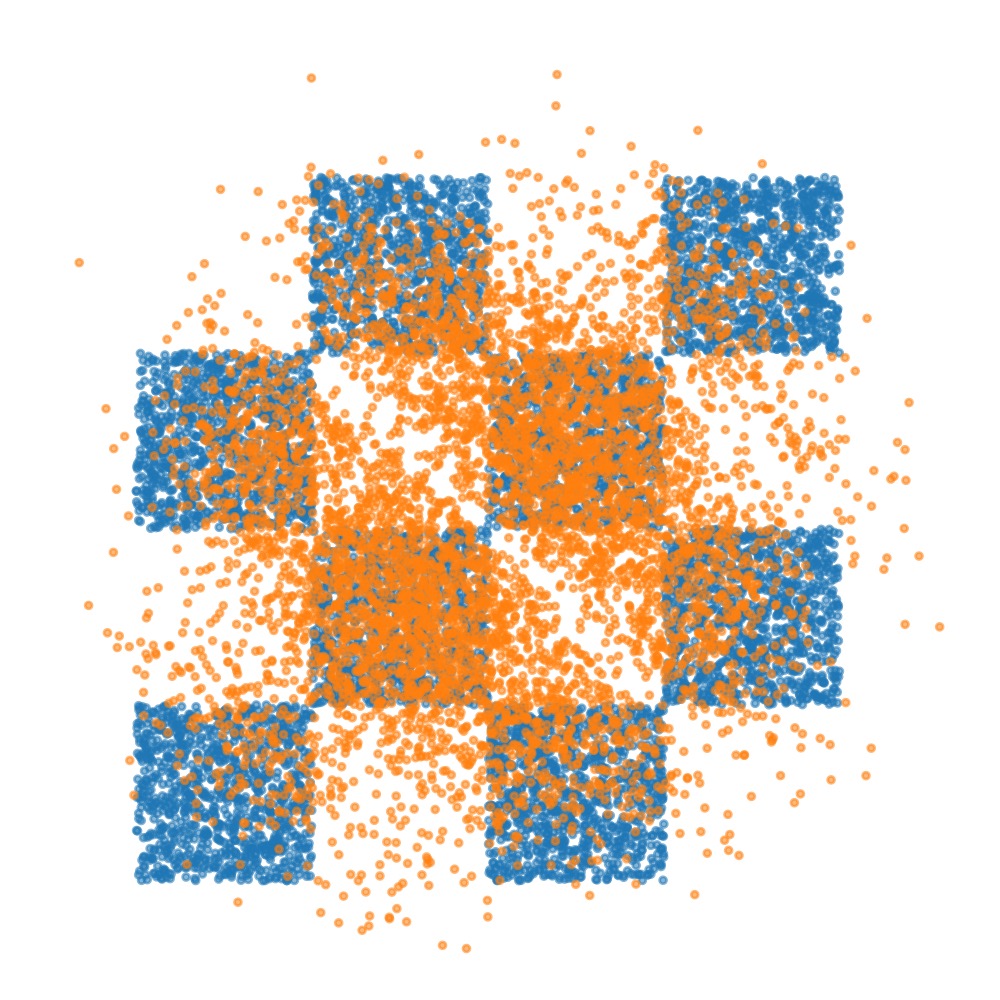} &
    \imgcell[width=1.55cm]{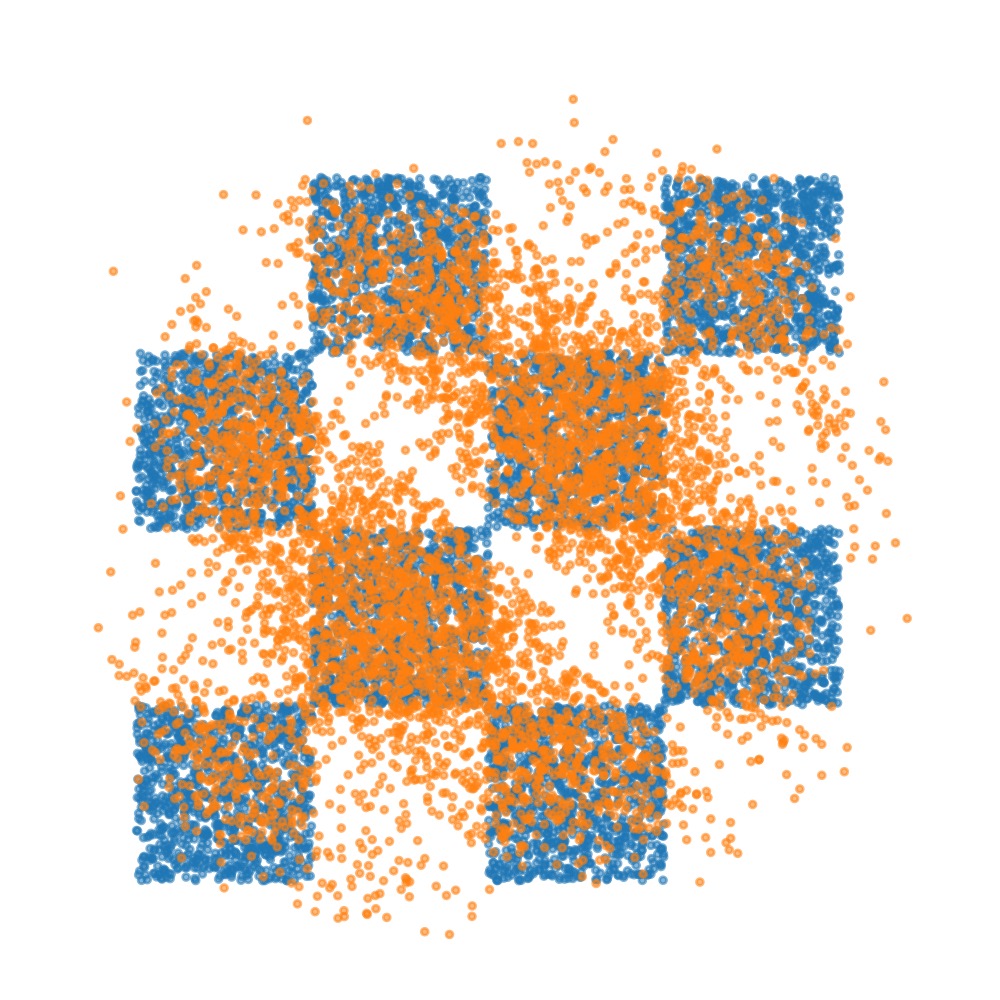} &
    \imgcell[width=1.55cm]{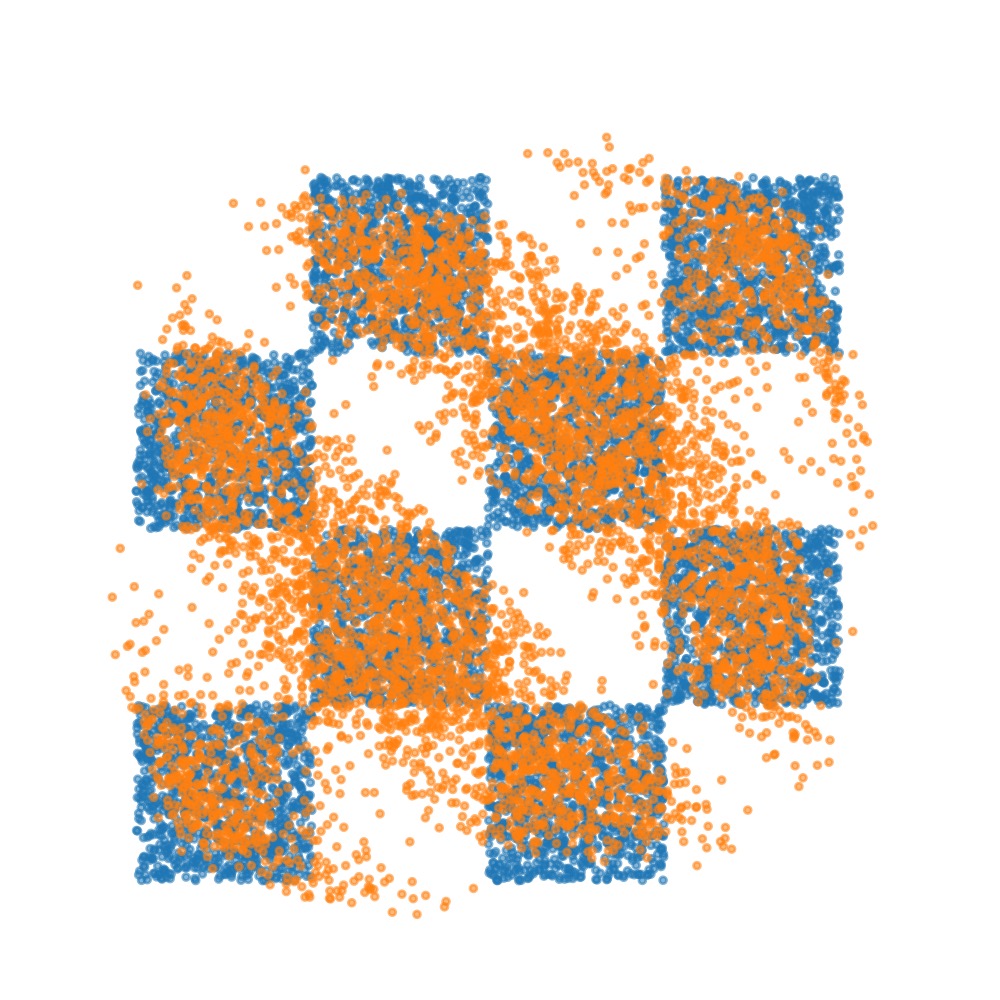} &
    \imgcell[width=1.55cm]{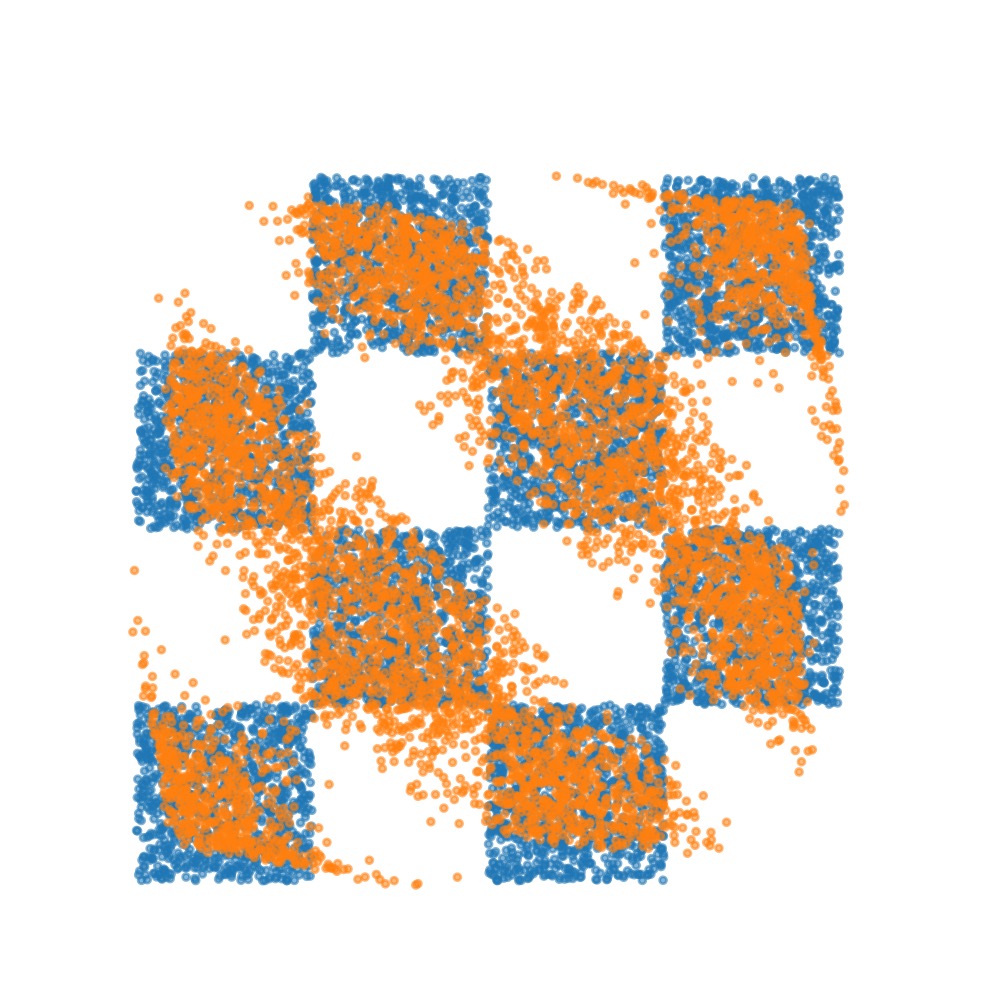} &
    \imgcell[width=1.55cm]{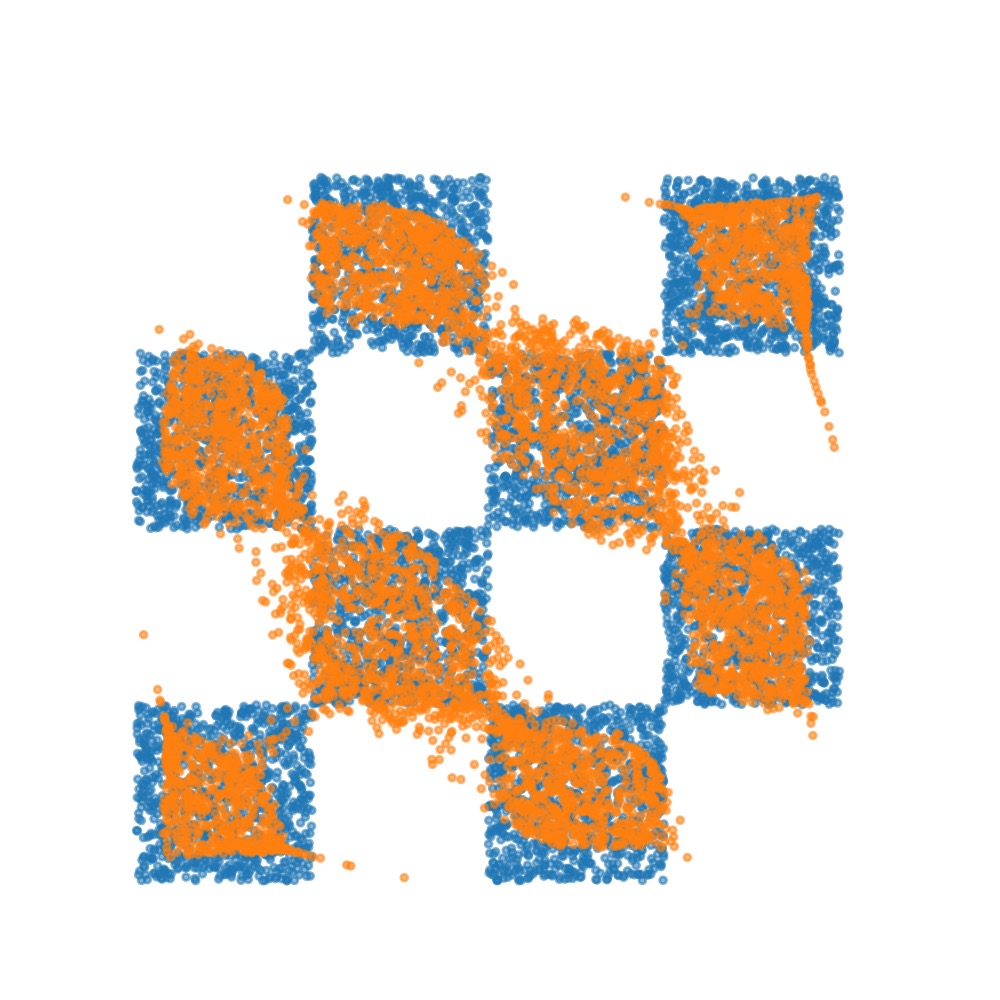} &
    \imgcell[width=1.55cm]{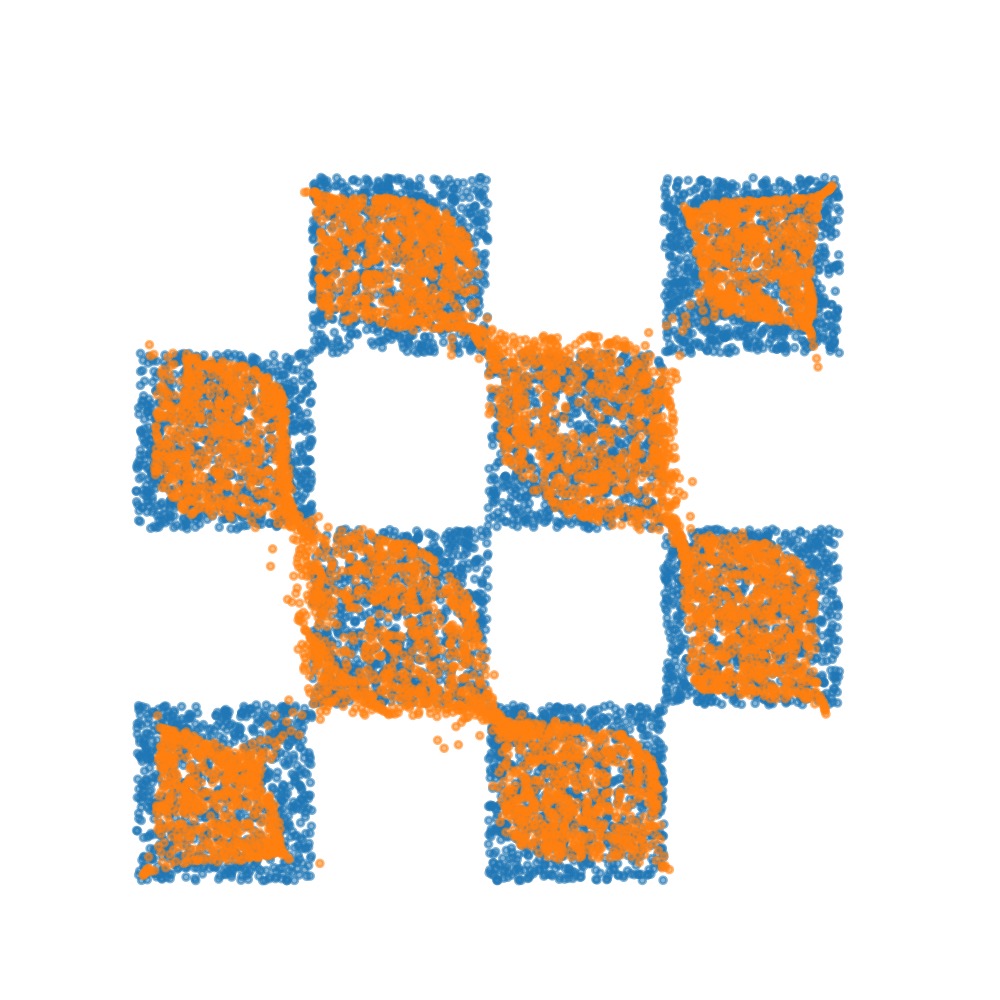} &
    \imgcell[width=1.55cm]{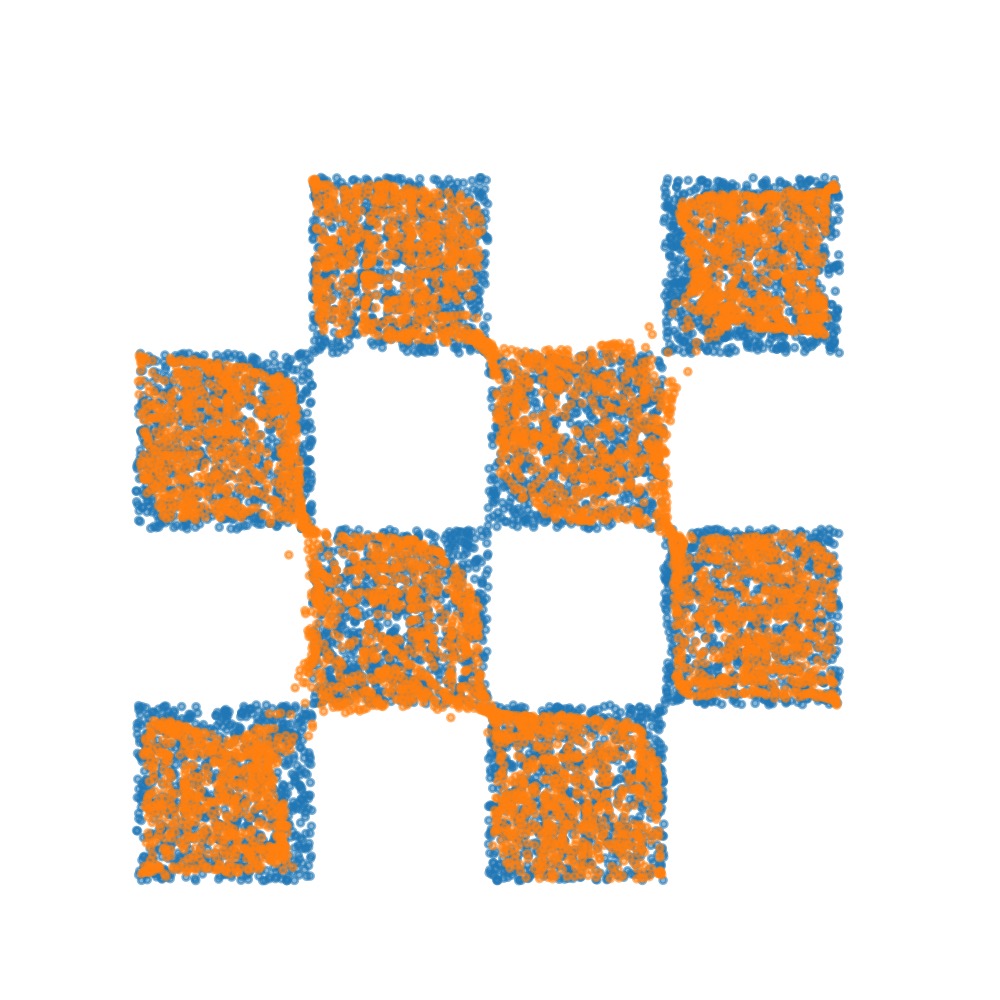} \\

    \imgcell[width=1.55cm]{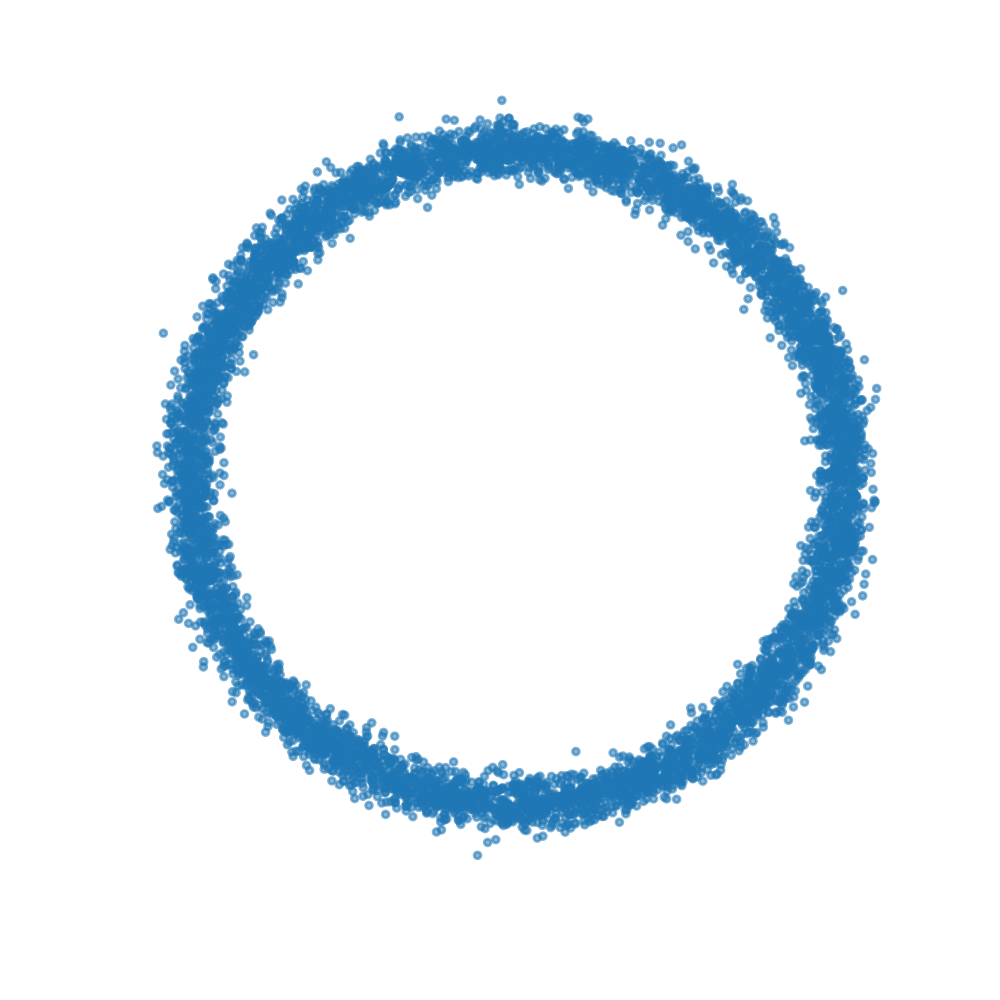} &
    \imgcell[width=1.55cm]{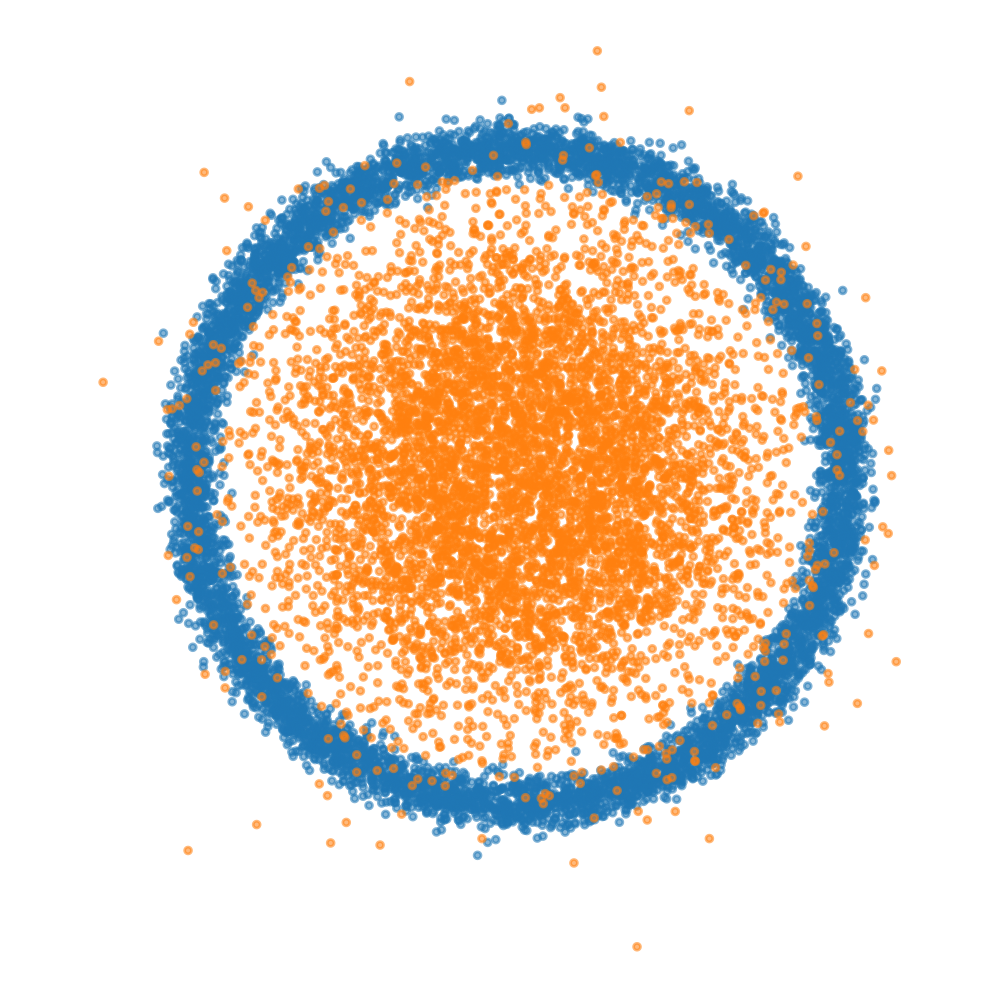} &
    \imgcell[width=1.55cm]{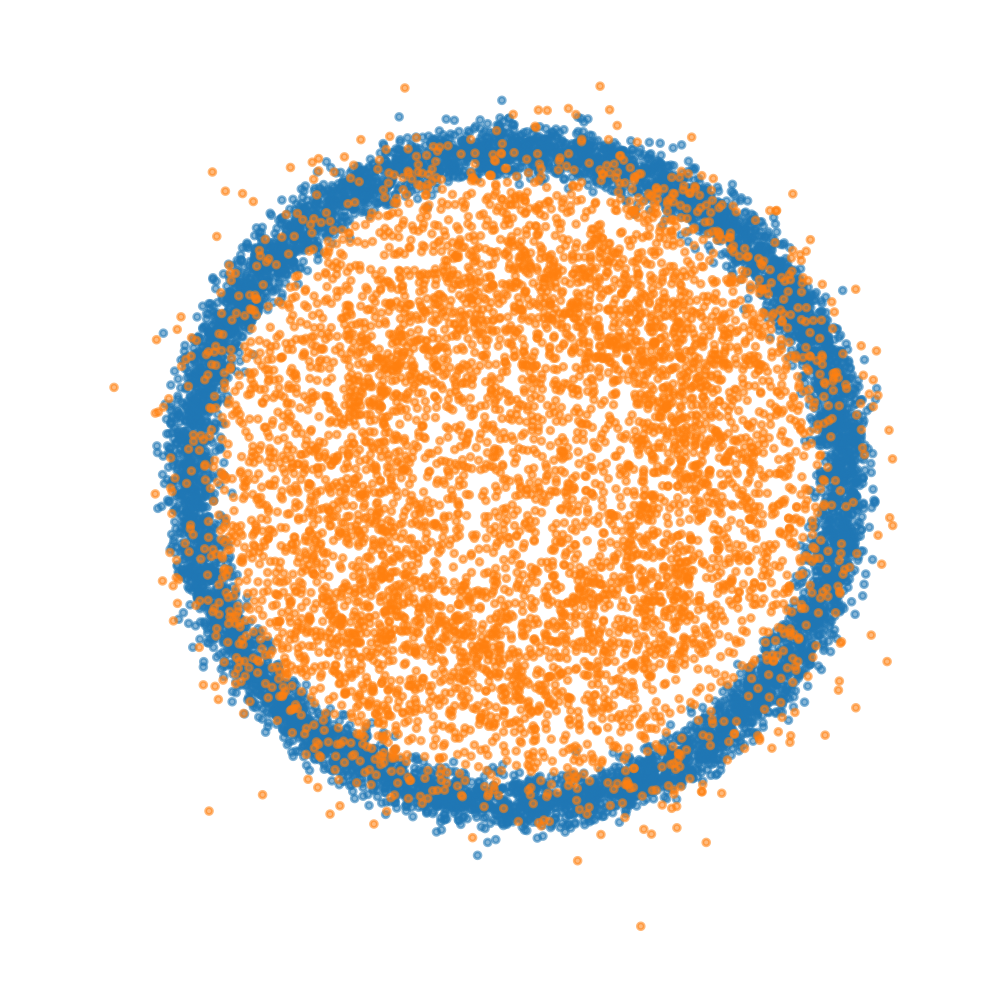} &
    \imgcell[width=1.55cm]{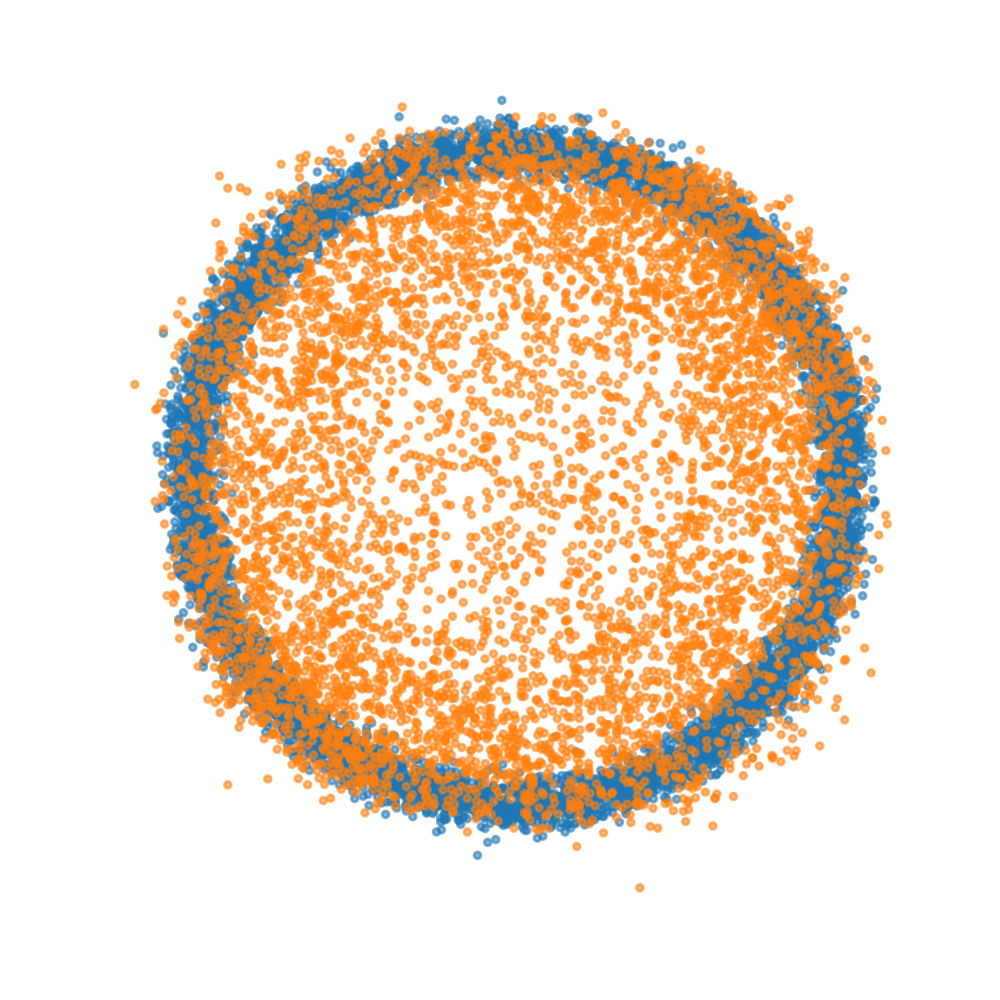} &
    \imgcell[width=1.55cm]{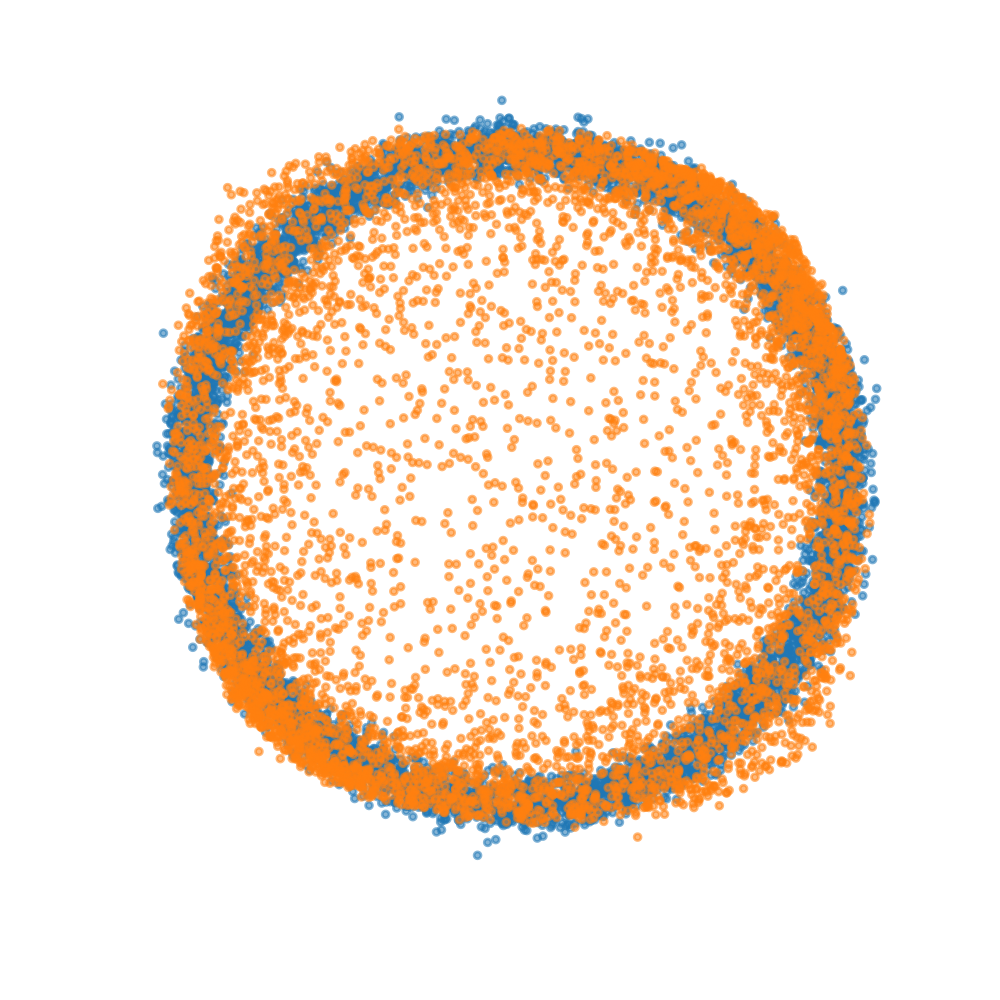} &
    \imgcell[width=1.55cm]{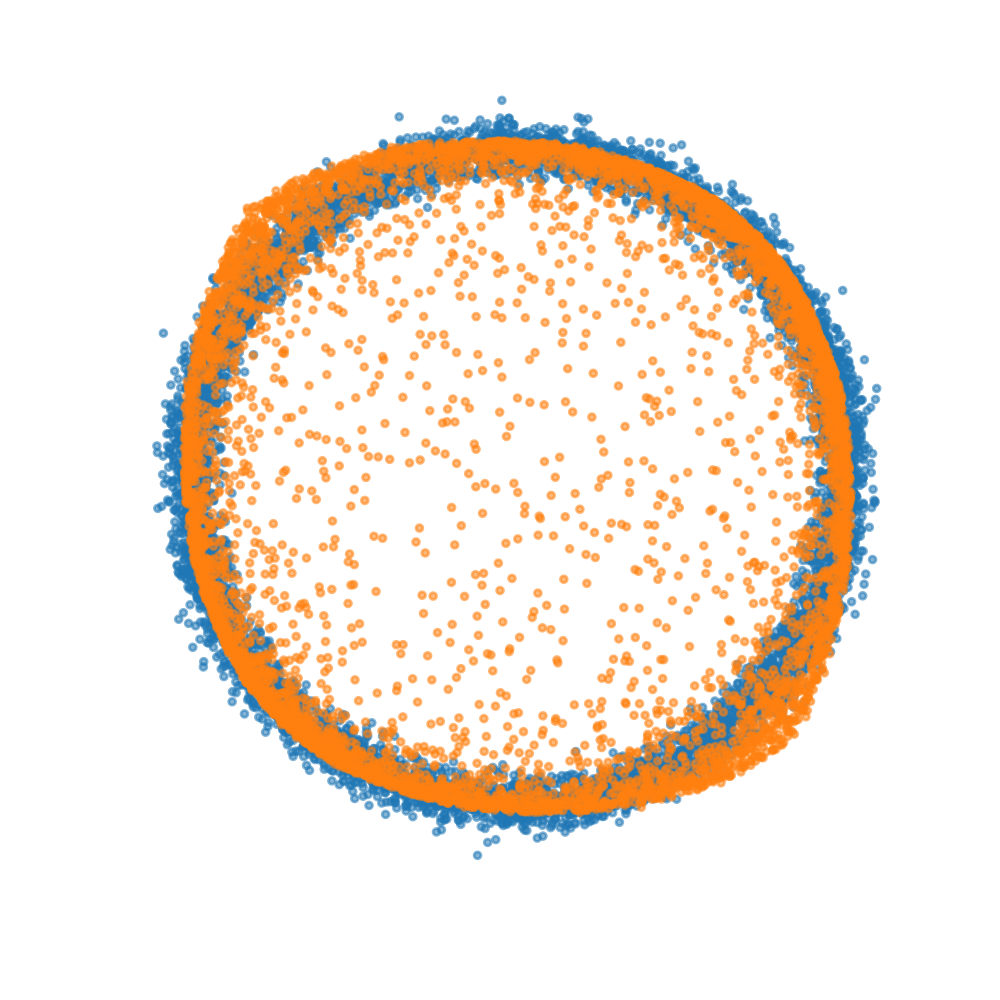} &
    \imgcell[width=1.55cm]{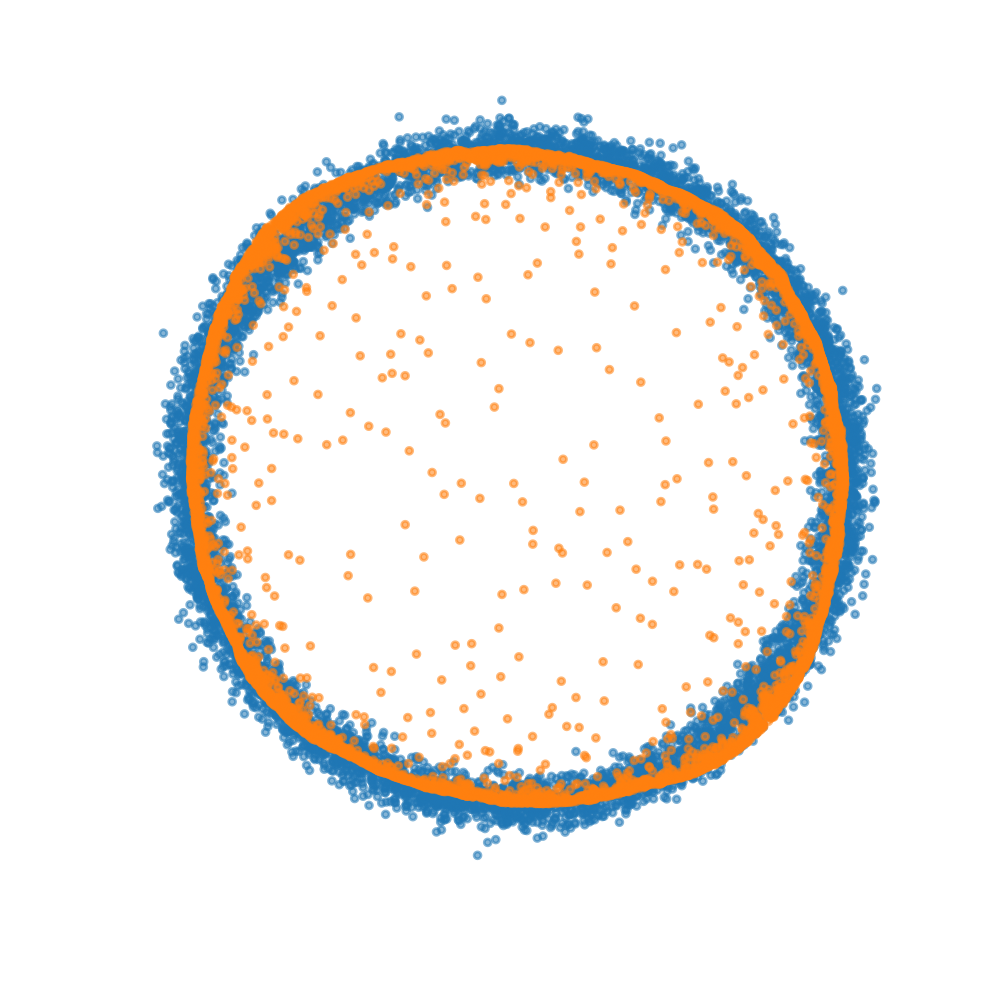} &
    \imgcell[width=1.55cm]{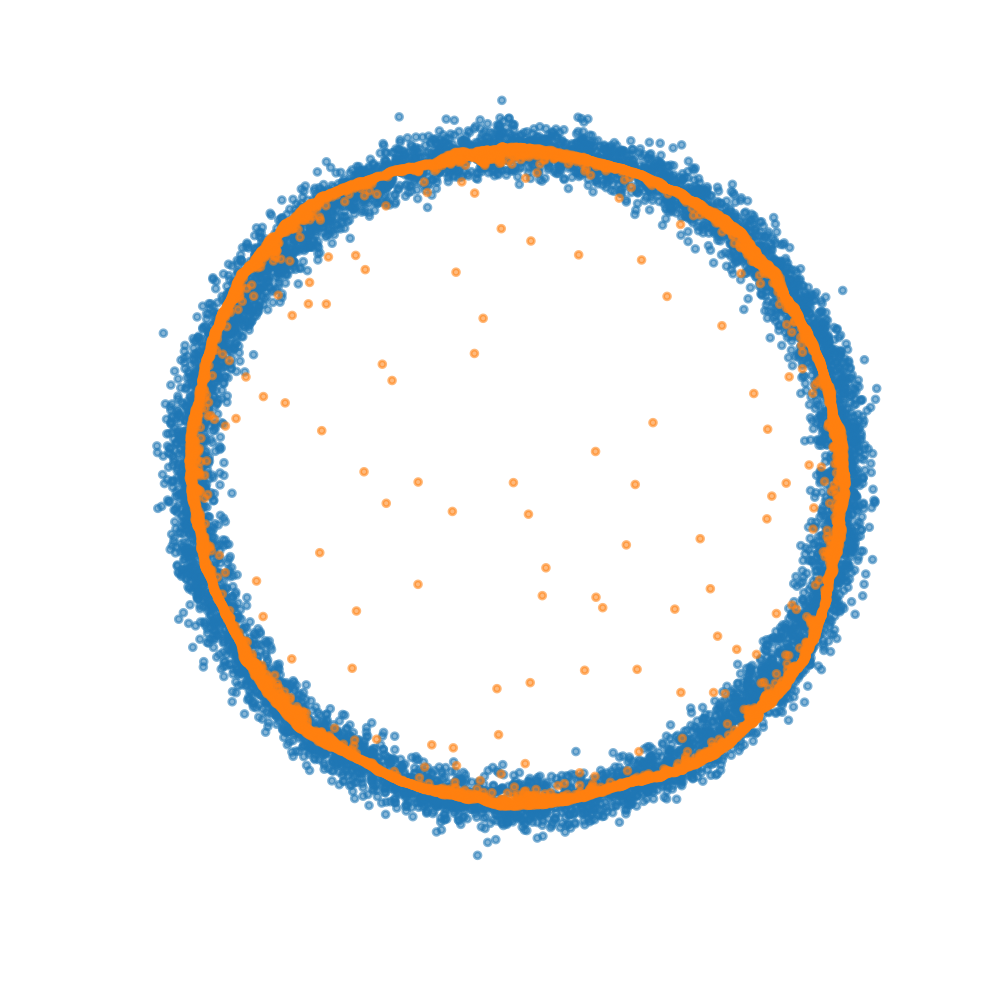} & 
    \imgcell[width=1.55cm]{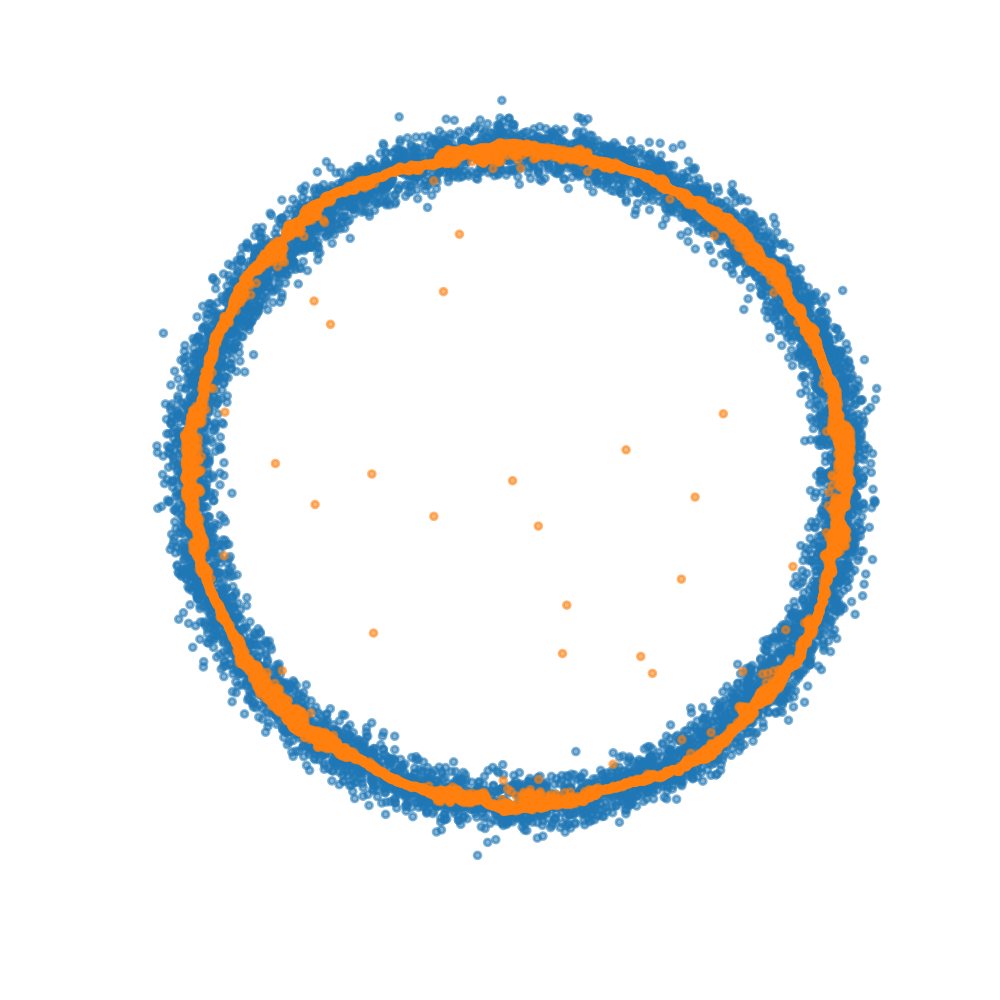} \\

    \imgcell[width=1.55cm]{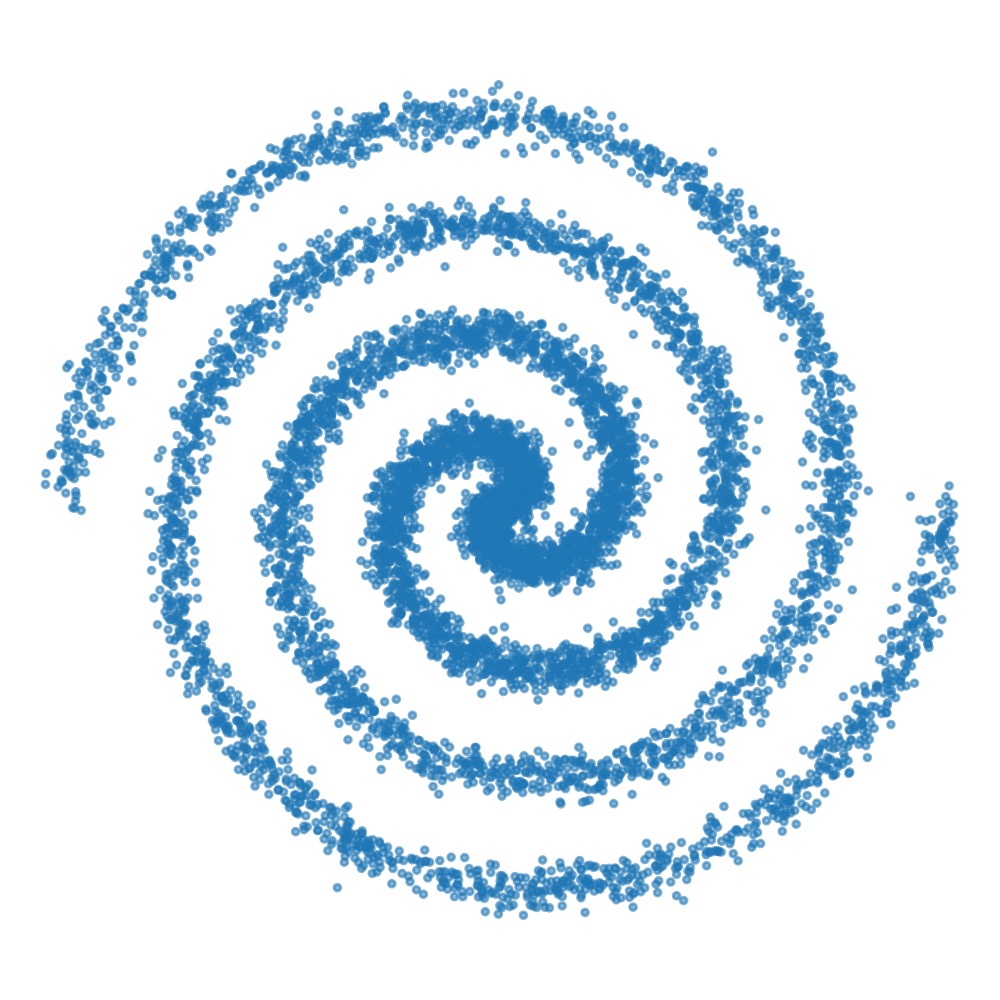} &
    \imgcell[width=1.55cm]{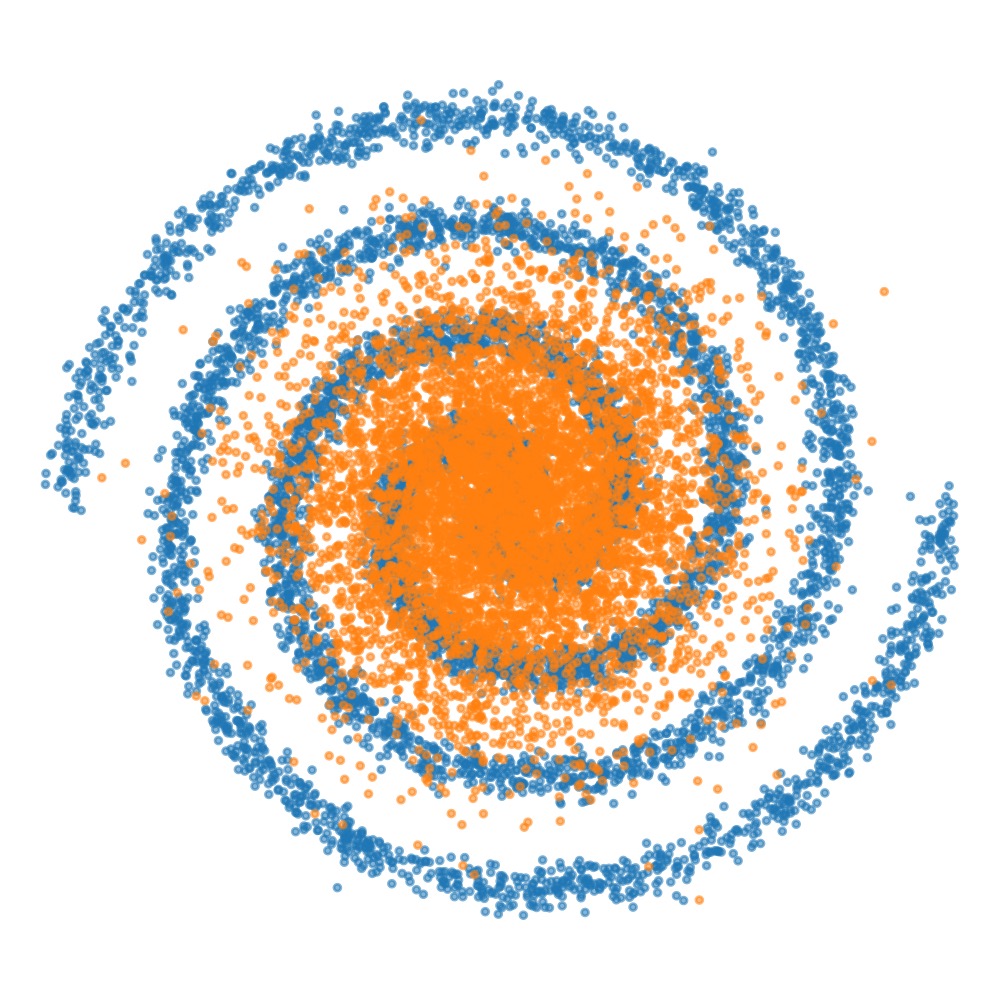} &
    \imgcell[width=1.55cm]{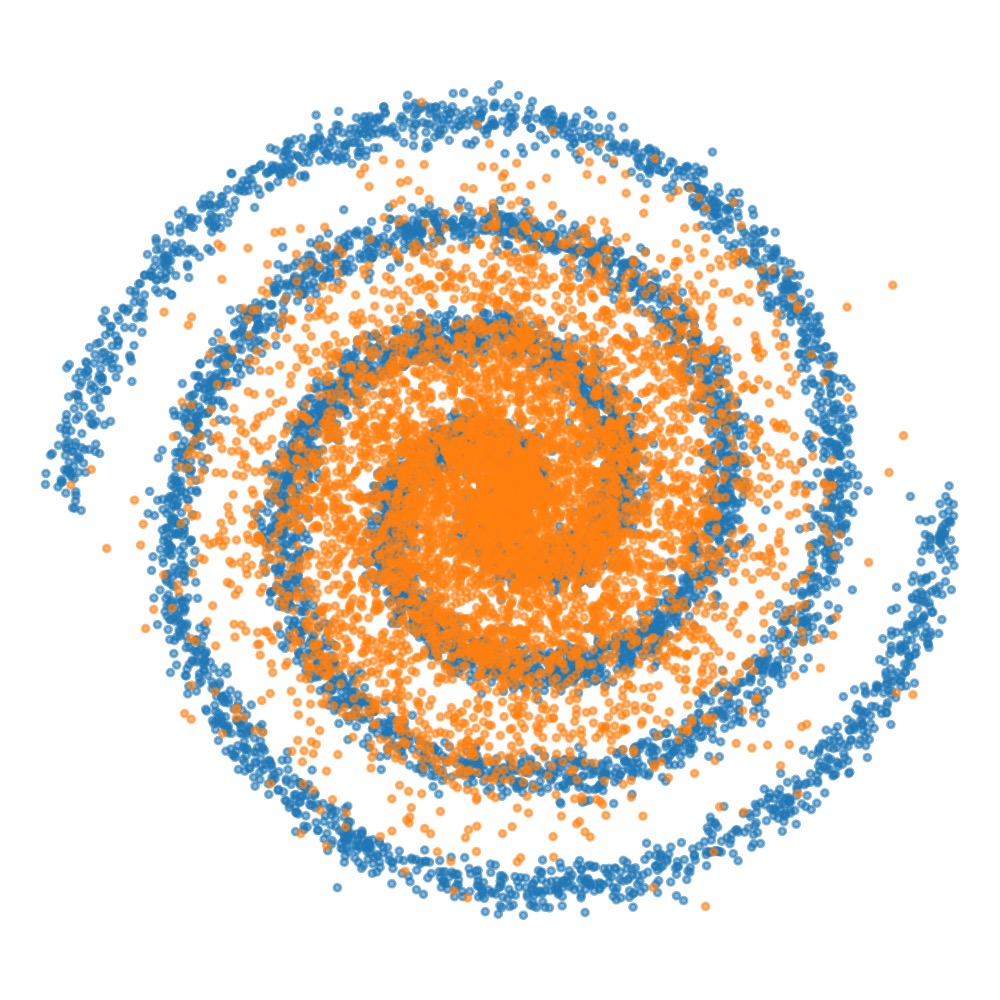} &
    \imgcell[width=1.55cm]{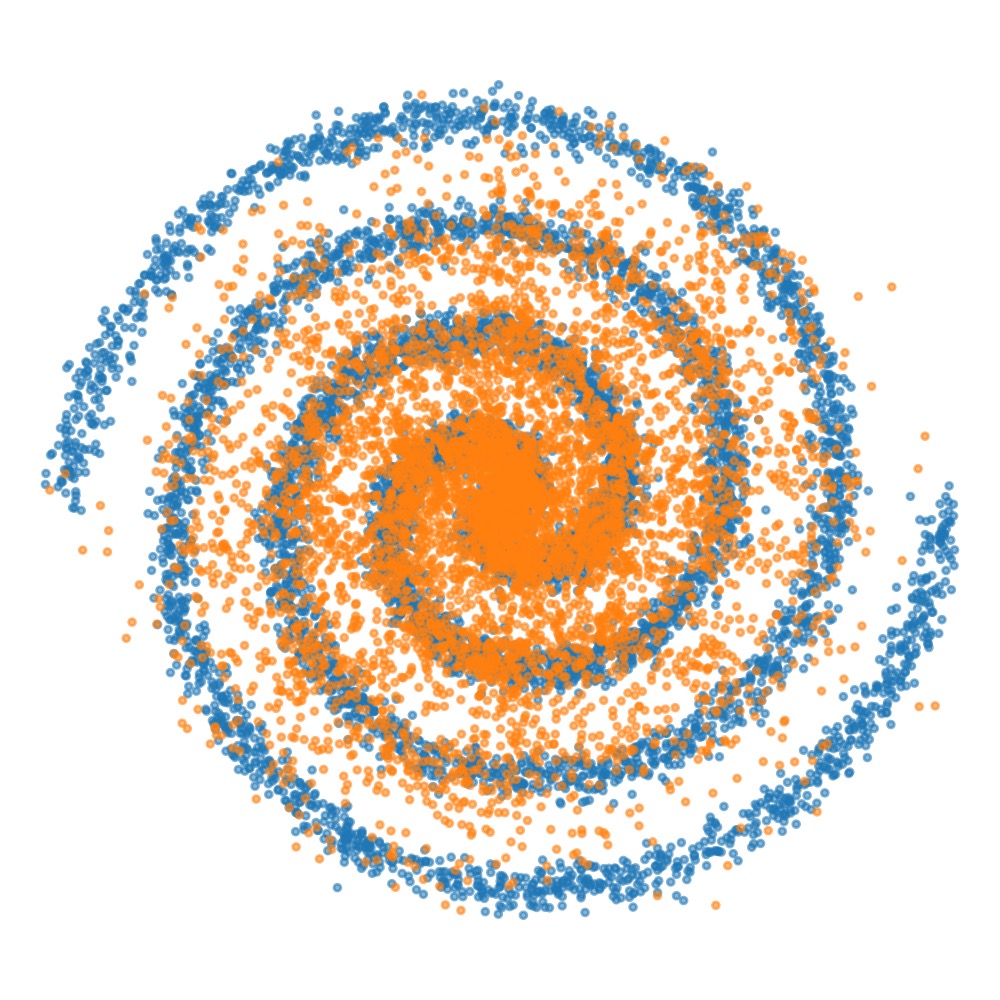} &
    \imgcell[width=1.55cm]{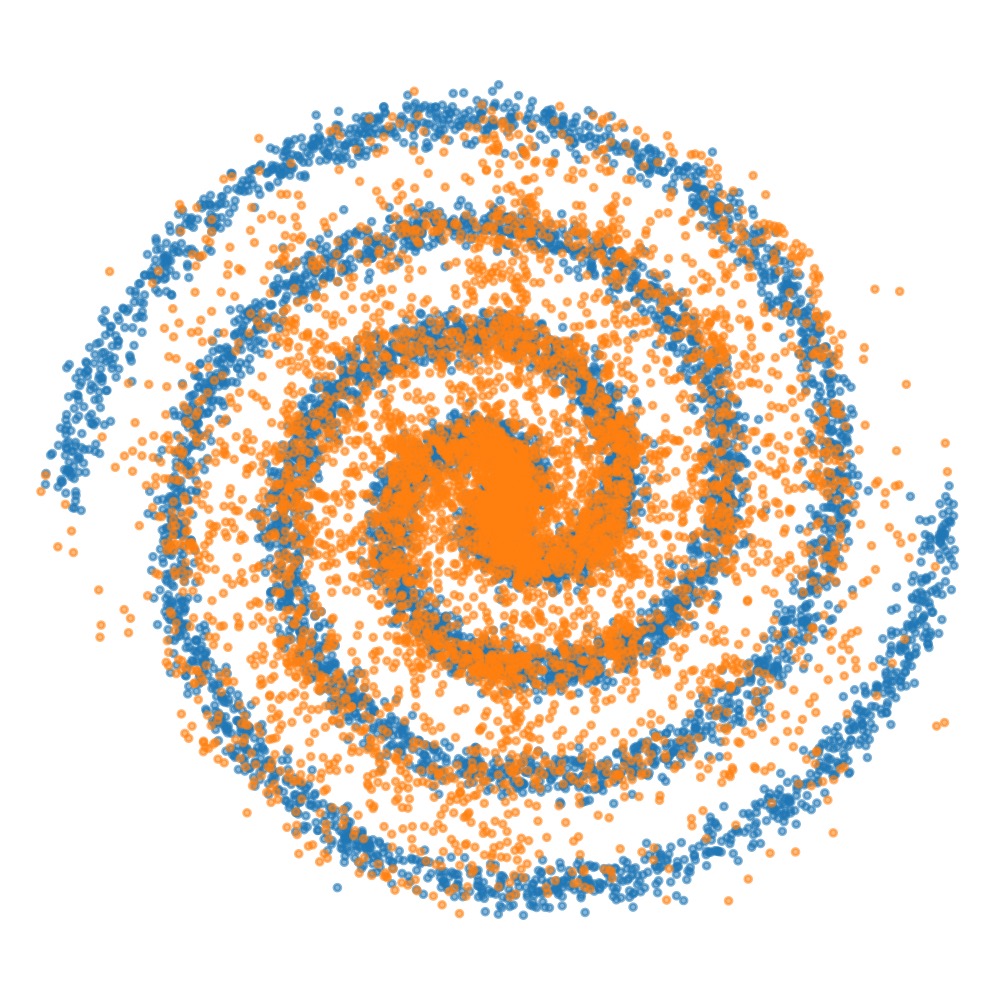} &
    \imgcell[width=1.55cm]{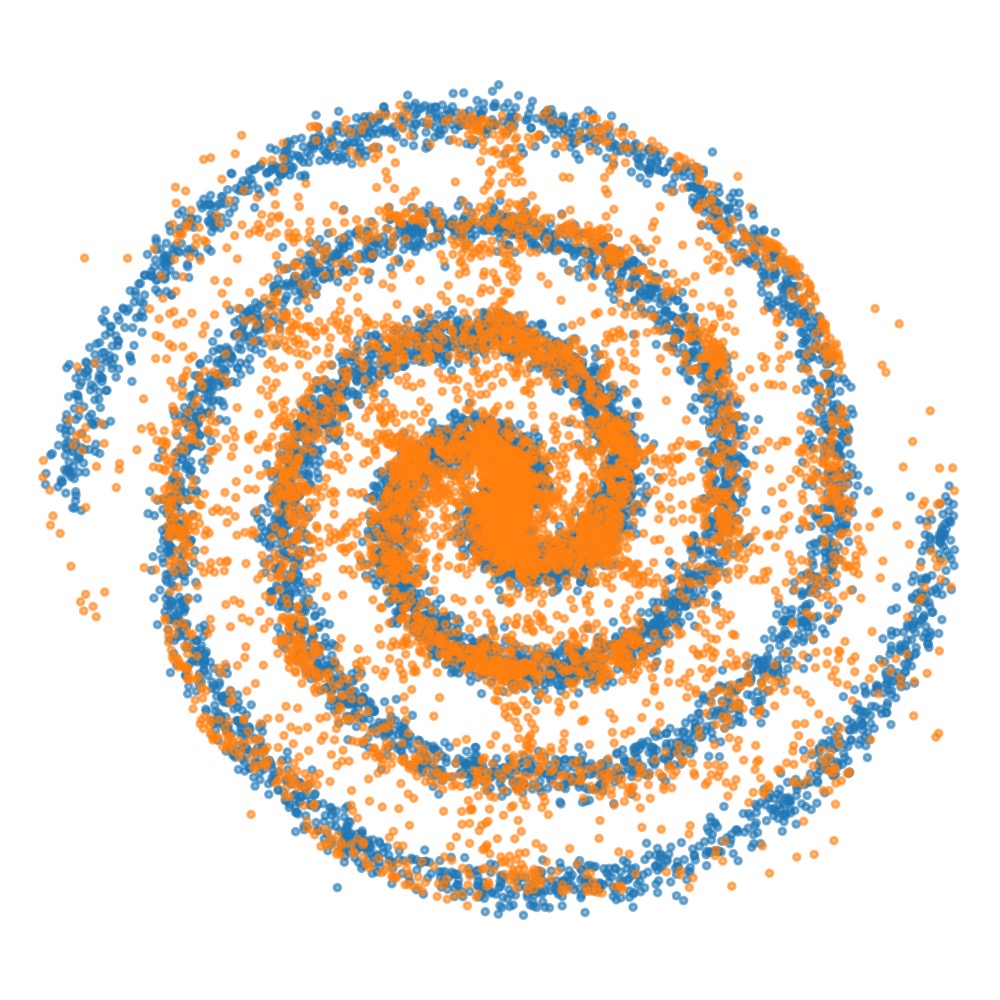} &
    \imgcell[width=1.55cm]{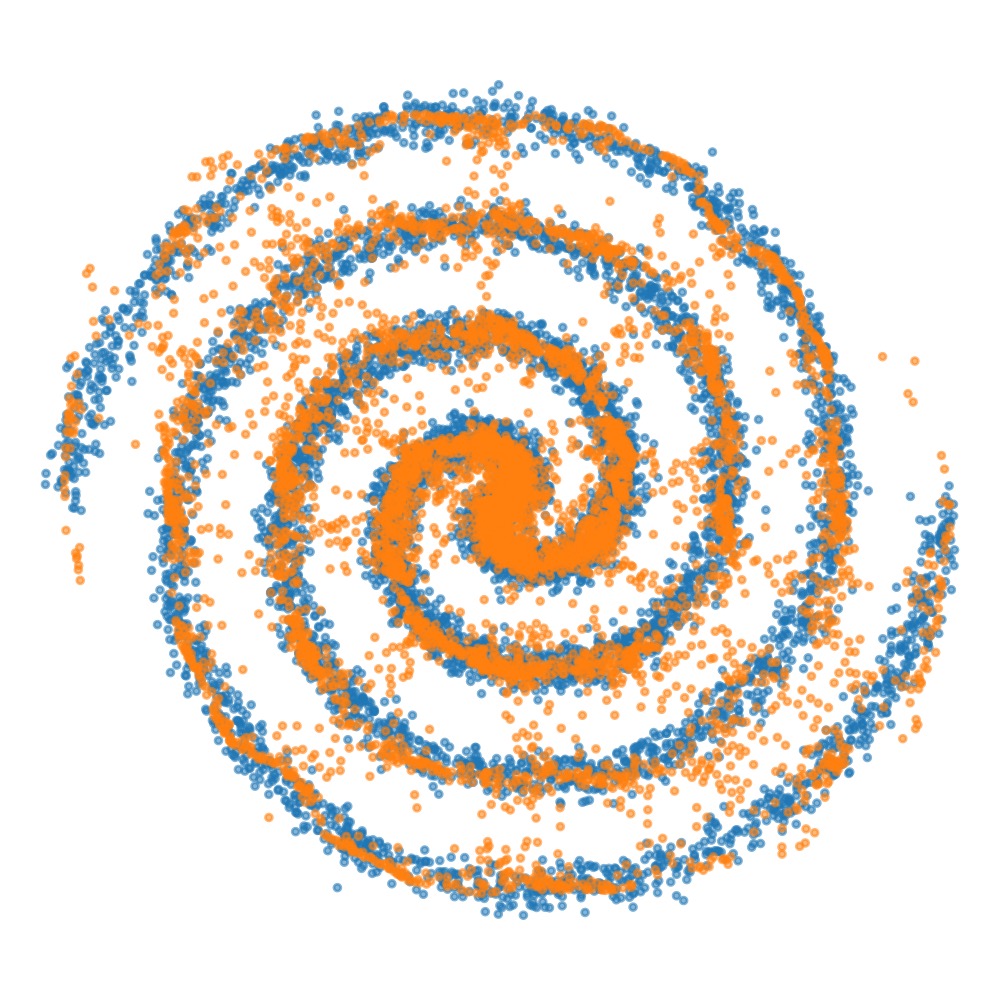} &
    \imgcell[width=1.55cm]{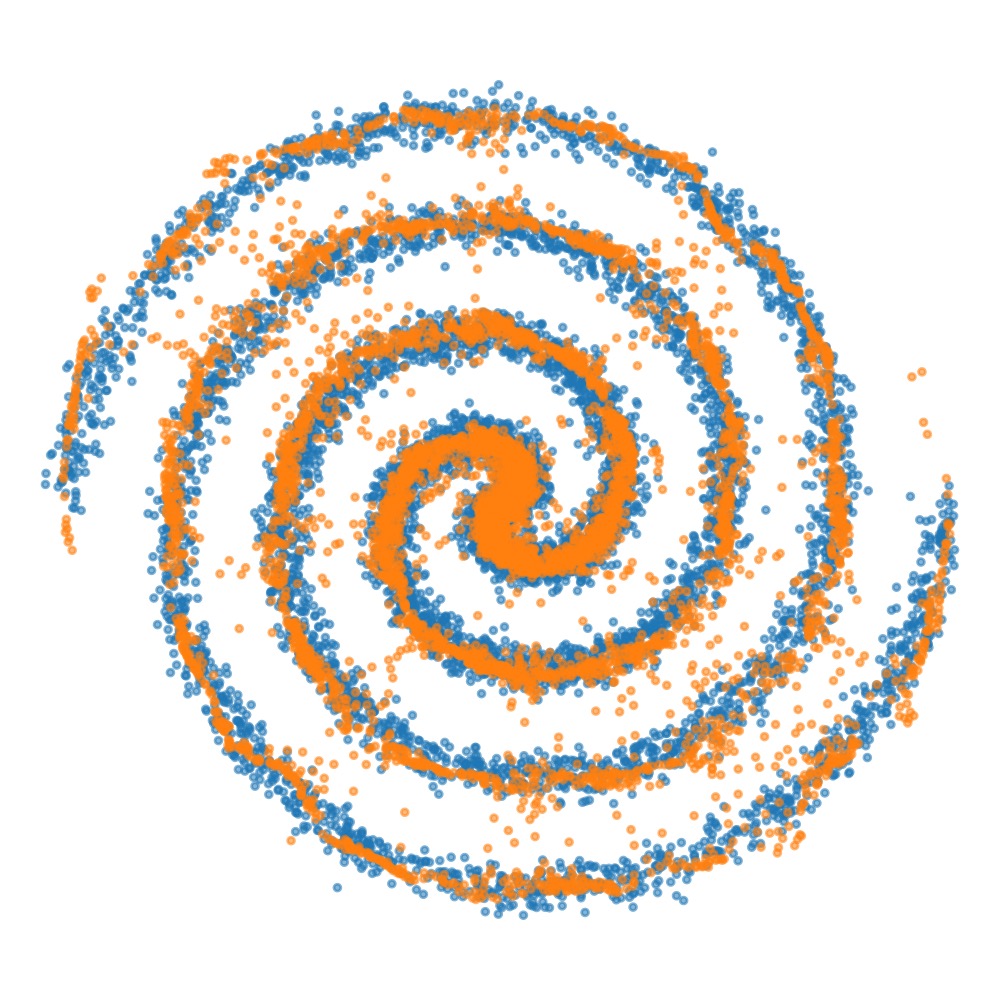} &
    \imgcell[width=1.55cm]{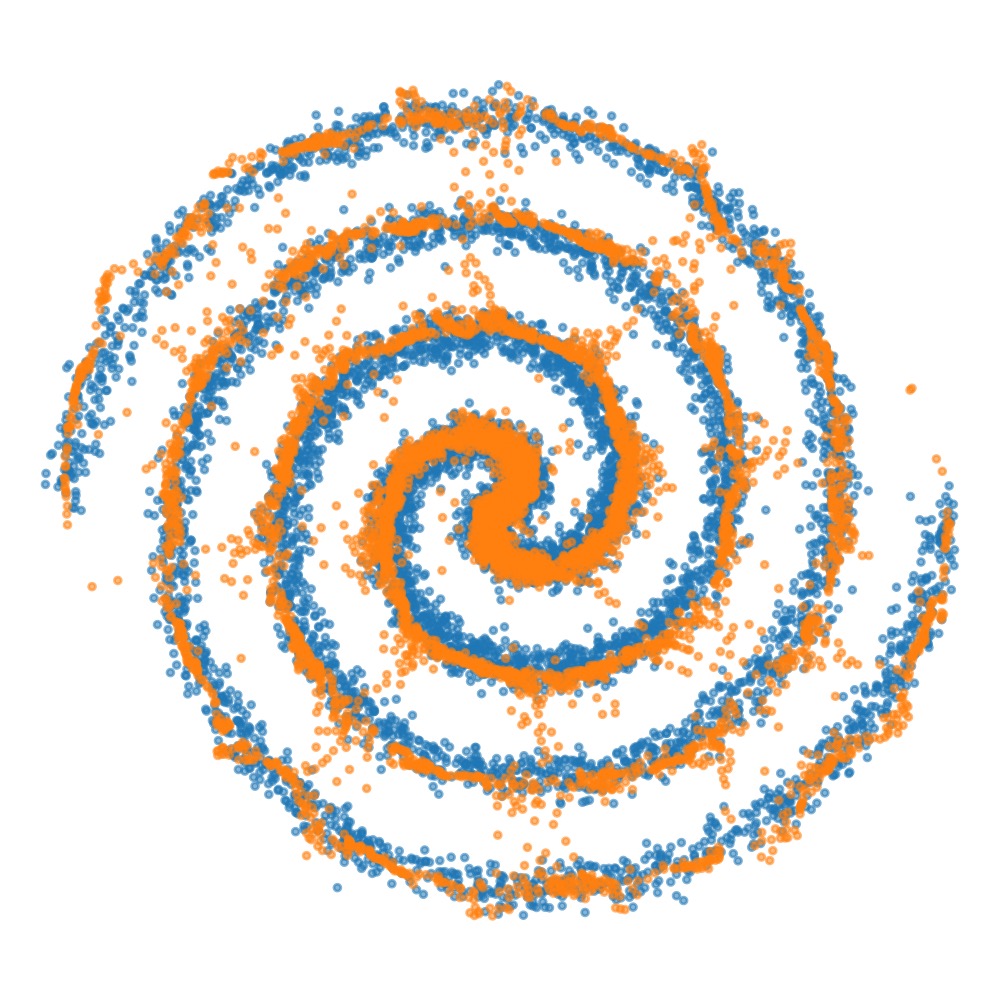} \\

    \imgcell[width=1.55cm]{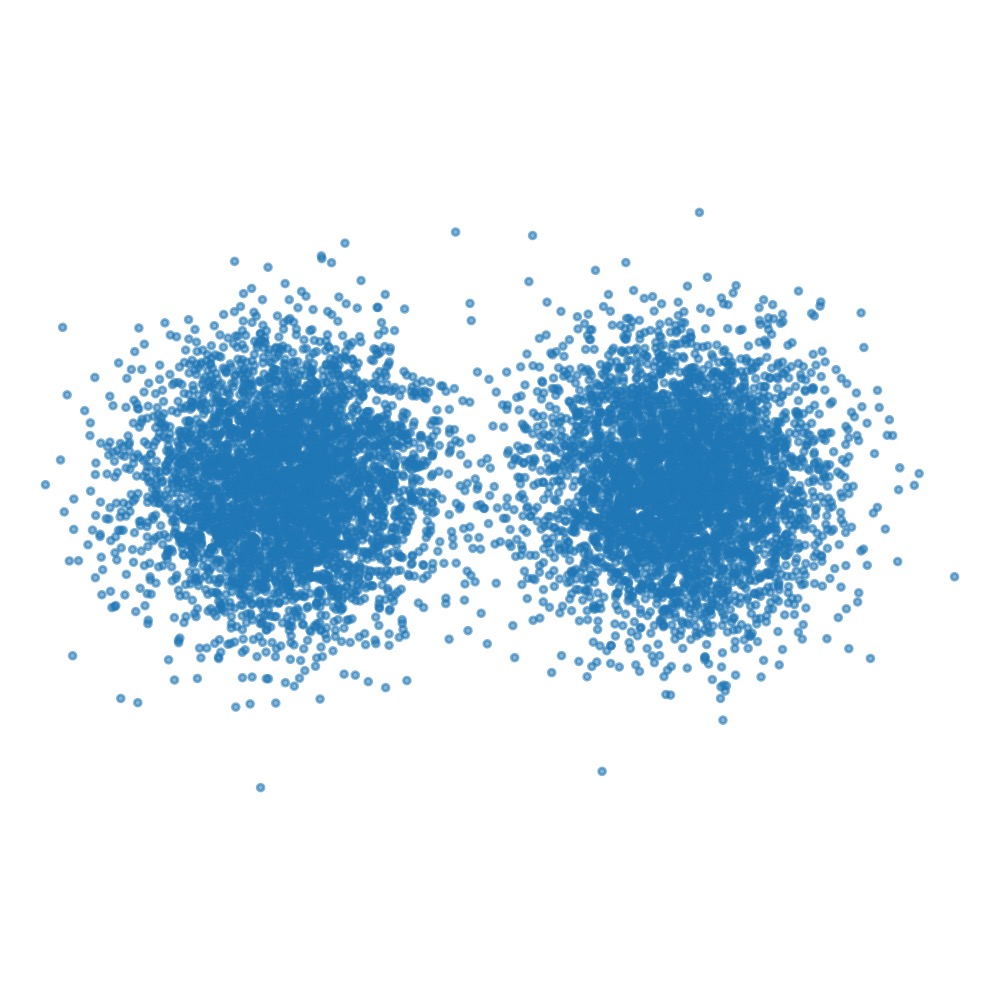} &
    \imgcell[width=1.55cm]{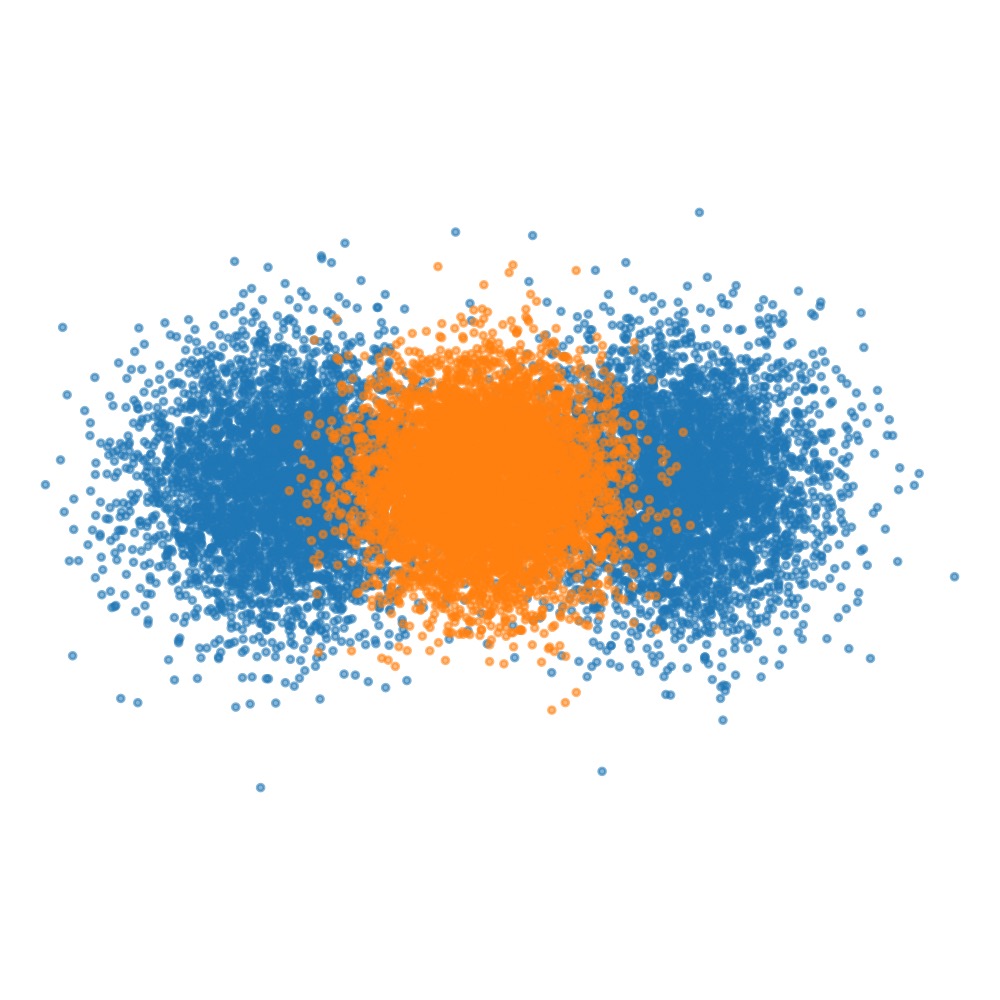} &
    \imgcell[width=1.55cm]{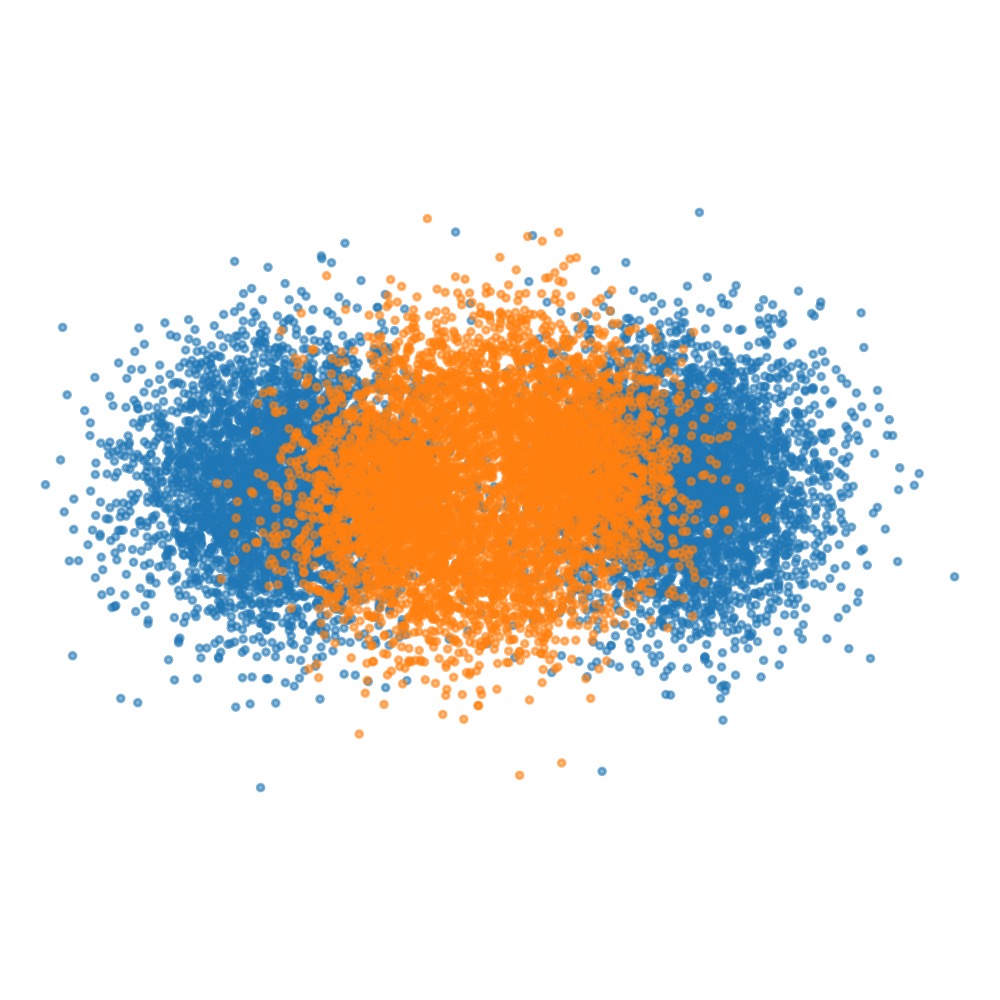} &
    \imgcell[width=1.55cm]{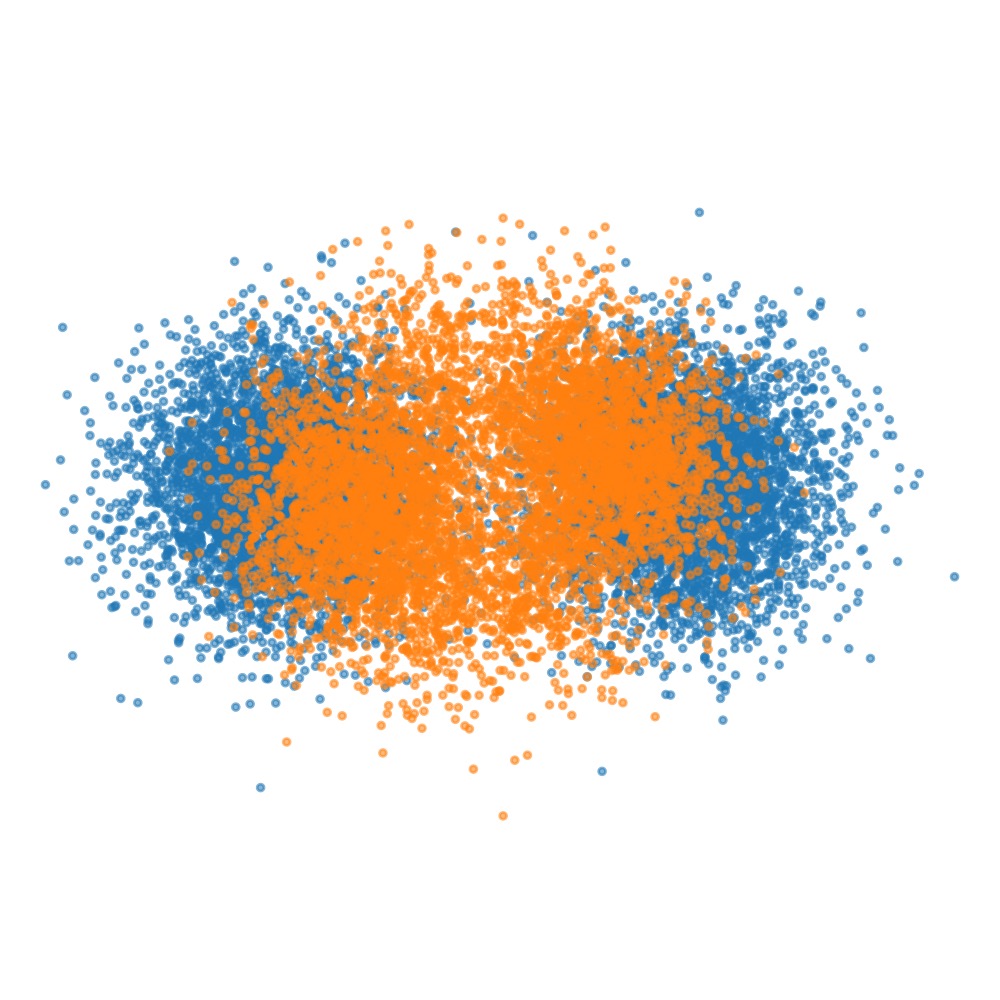} &
    \imgcell[width=1.55cm]{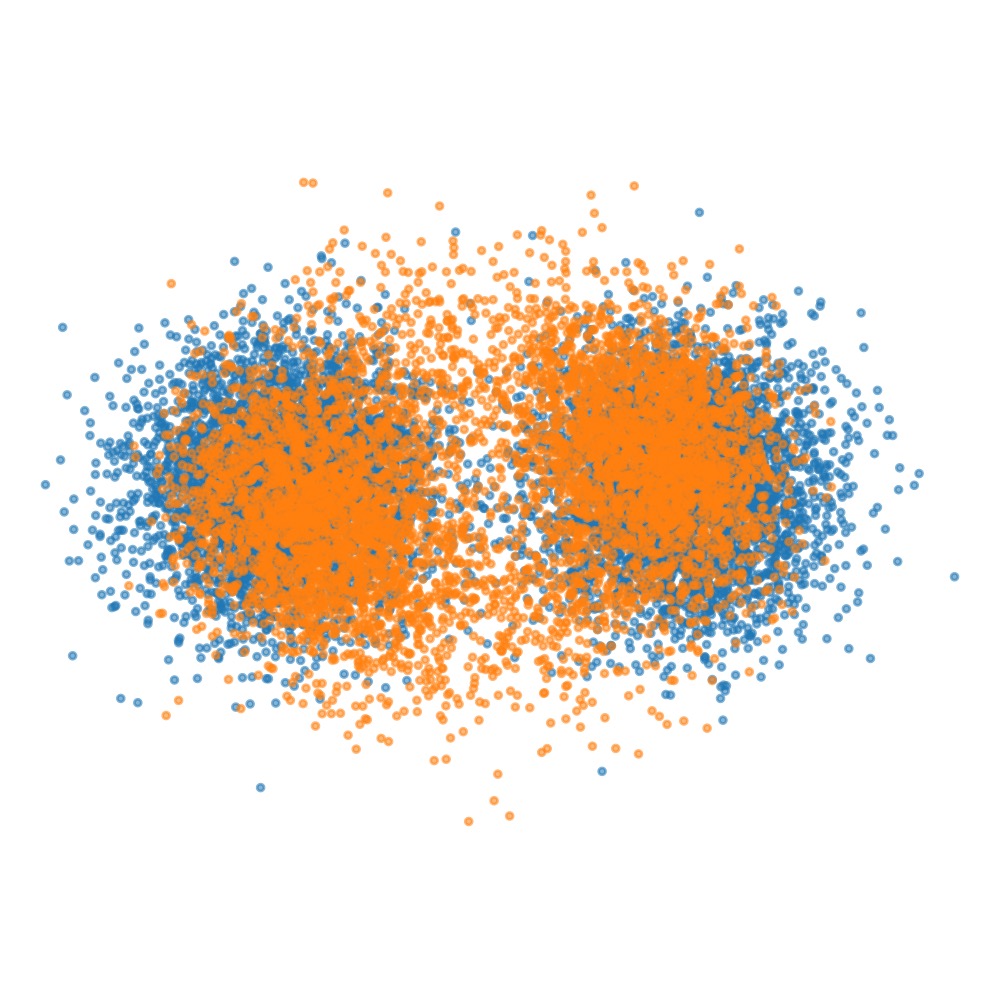} &
    \imgcell[width=1.55cm]{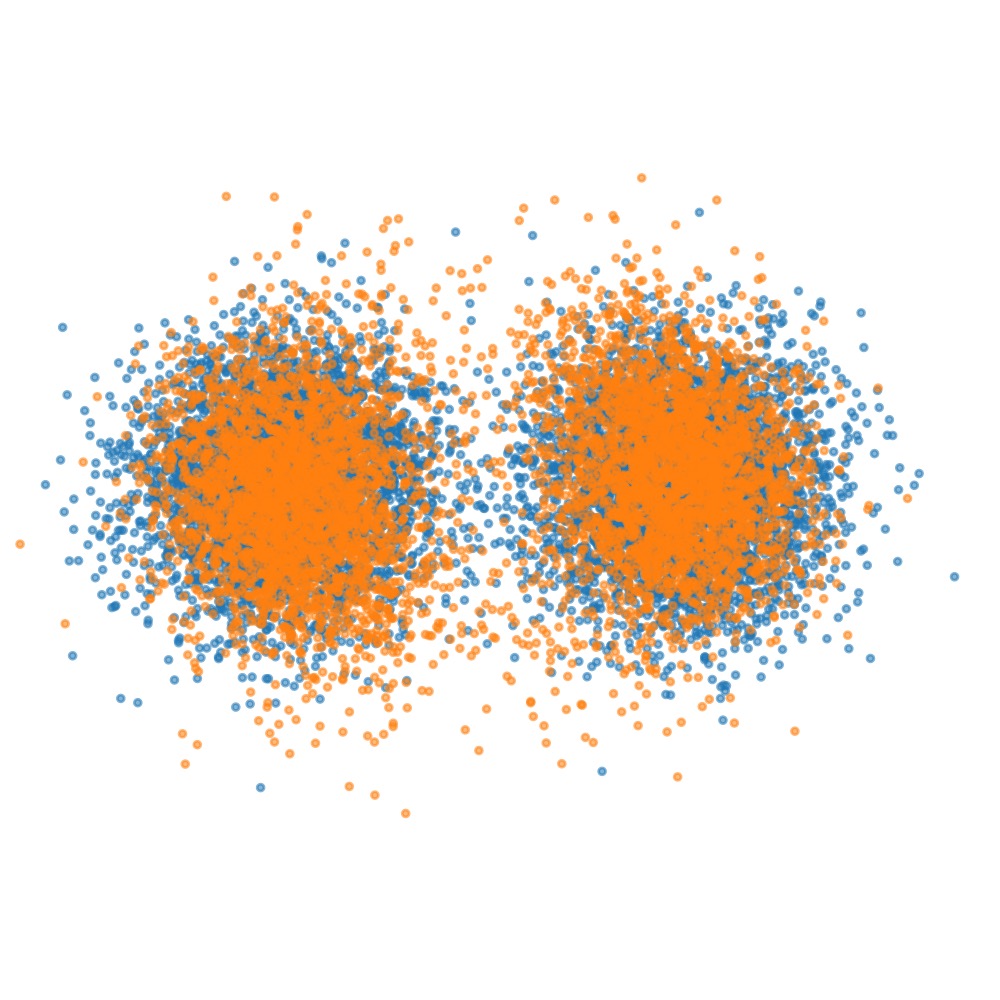} &
    \imgcell[width=1.55cm]{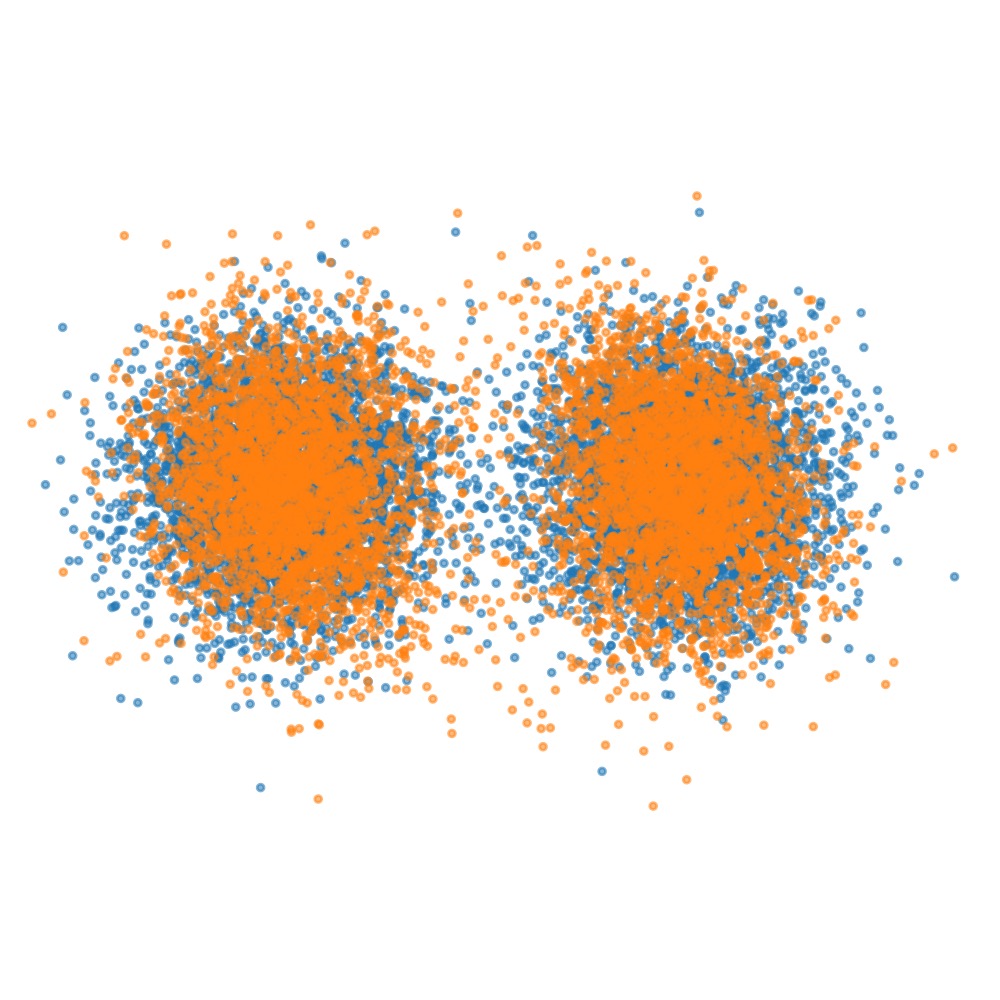} &
    \imgcell[width=1.55cm]{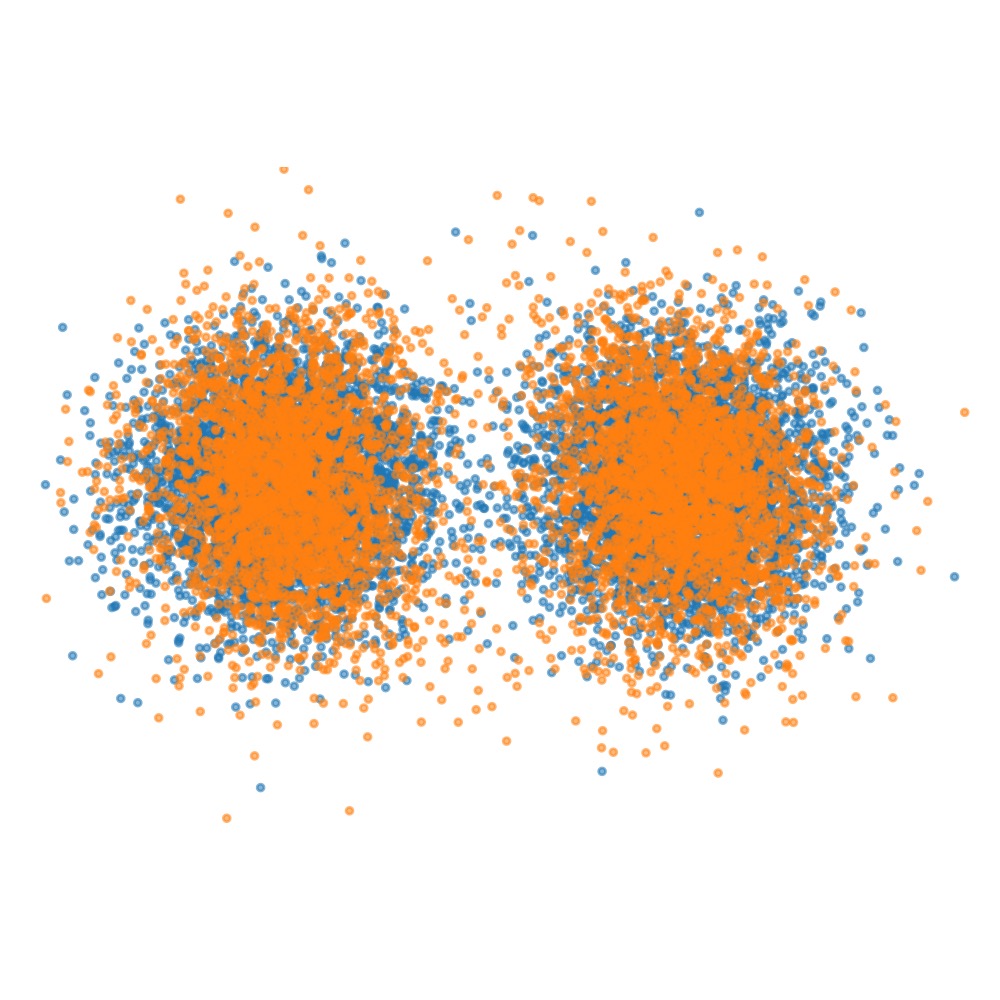} &
    \imgcell[width=1.55cm]{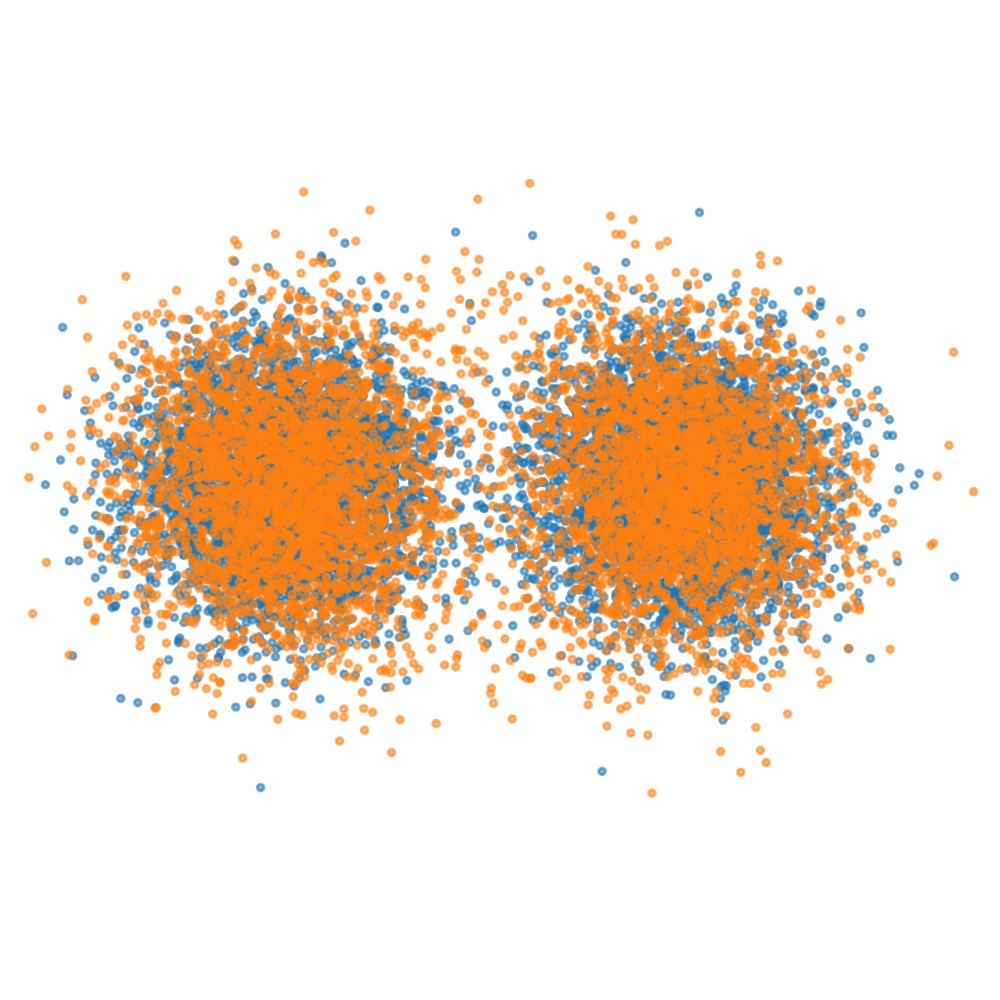} \\

    \axiscell{0} & \axiscell{1} & \axiscell{5} & \axiscell{10} & \axiscell{20} & \axiscell{40} & \axiscell{100} & \axiscell{200} & \axiscell{400} \\

    \end{tabular}
  \end{adjustbox}

    \caption{\small Rank-Proximal Transport on 2D toy targets. Using SGD to minimize the rank-statistic KL with $L=10$ random projections, particles (orange) evolve from a Gaussian start ($t=0$) to match the target support (blue).}
  \label{fig:mosaic_nice}
\end{figure}

\subsubsection{CIFAR-10 experiments}
\label{subsubsec:celeba}

We next illustrate the induced particle dynamics on \emph{CIFAR-10}
\citep{krizhevsky2009learning} using a center-outward rank-proximal
transport (CO-RPT) update (see details of the algorithm in Appendix
\ref{app:cifar10-corpt}).
Although CIFAR-10 is natively $32\times32$, we use a high-quality
bicubic upsampled $64\times64$ version with antialiasing, and treat
each particle as an RGB image flattened to $\mathbb{R}^{3\cdot 4096}$.
We use $M$ real CIFAR-10 training images as reference samples from the
target distribution $\nu$. Starting from $N$ i.i.d.\ Gaussian particles,
we iterate the center-outward transport update for a fixed number of
outer steps (see \cref{app:CIFAR10} for full schedules).

In our runs, we use the Jensen--Shannon generator with trust parameter
$\eta=0.5$ and a small number of inner prox steps per iteration. Since
the transport is performed directly in a high-dimensional pixel space,
we use a moderate outer step size, linearly annealing $\varepsilon$
from $0.16$ to $0.10$ over training. In parallel, we anneal the
rank-smoothing temperature from $\tau=0.30$ to $\tau=0.07$ and increase
the rank resolution from $K=64$ to $K=160$. To avoid overly large
updates in the ambient pixel space, we clip the per-particle correction
with a cap of $0.30$. We additionally use a small angular blending
parameter $\beta_{\mathrm{angle}}=0.01$.

Qualitatively, the dynamics progressively transforms the initial noise
cloud into structured images that match low- and mid-level statistics
of CIFAR-10, such as global color balance, coarse spatial layout, and
local texture; see Figure~\ref{fig:cifar10_corpt_qual}. Further
implementation details, and qualitative results on MNIST, are deferred
to \cref{app:CIFAR10}.

\begin{figure}[!ht]
    \centering
    \includegraphics[width=0.75\linewidth]{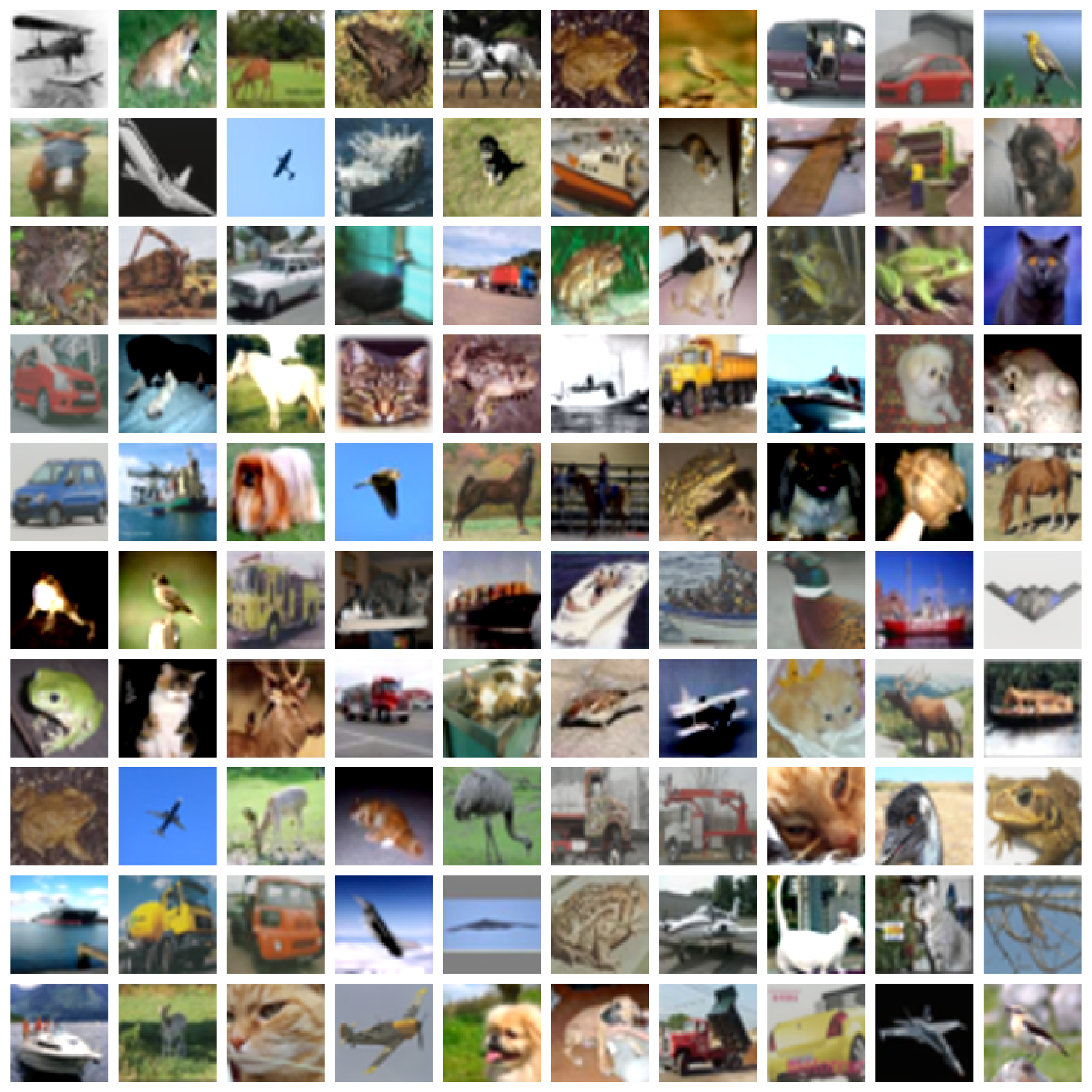}
    \caption{\small CO-RPT samples on CIFAR-10 ($64\times64$) after $T=20{,}000$ outer steps.} \label{fig:cifar10_corpt_qual}
\end{figure}

As an additional pretraining experiment, \cref{subsec:MNIST_mode_collapse} shows that rank-statistic pretraining improves precision/recall trade-offs for DCGAN on MNIST; Table 5 reports that TV/KL pretraining improves precision after DCGAN fine-tuning, while JS and Hellinger emphasize recall.

\section{Conclusions, Future Work, and Limitations}
\label{sec:conclusion}

We proposed \emph{rank-statistic} approximations of $f$-divergences that replace density-ratio estimation with simple rank counting in a discrete histogram. The resulting surrogate $D^{(K)}_{f,\nu}(\mu)$ has a clean variational structure (convexity and weak lower semicontinuity), is nondecreasing in the resolution $K$, and remains dominated by the target divergence; under mild regularity of the density ratio we proved consistency as $K\!\to\!\infty$ with quantitative rates, and established finite-sample deviation guarantees for practical estimators.
For multivariate data, we introduced sliced rank-statistic divergences by averaging the univariate construction over random 1D projections, inheriting its key properties, and we validated the approach empirically on synthetic benchmarks and as a sample-based transport and pretraining objective for implicit generative modeling.

We generalized the rank-statistic approximation of the TV-divergence from \cite{dFVOM2025}. Since the TV-divergence is the only $f$-divergence which is also an integral probability metric (IPM) \cite{M1997}, it would be interesting to see if IPMs like maximum mean discrepancy \cite{BGRKSS06} or the Wasserstein-1 metric can be approximated by rank statistics as well.
It could also be promising to replace the Bernstein polynomials by another family, like B-splines or general non-linear filters.
Investigating the geodesic convexity properties of $D_{f, \nu}^{(K)}$ in the Wasserstein geometry could yield convergence rates of the generative transport dynamics.
Lastly, it would be interesting to identify a joint regime for $(K, N) \to \infty$ yielding the best convergence rate. We leave these questions for future work.

Our results also highlight limitations: projection complexity in high dimensions when discrepancies are strongly anisotropic or concentrated in dependencies that are hard to detect from 1D views, and the reliance on random directions for capturing such effects efficiently. Future work includes variance-reduced and structured projection schemes (e.g., orthogonal or quasi-Monte Carlo directions), improved anisotropy calibration beyond simple $d\times$ normalizations, tighter dimension-dependent guarantees, and scaling the objective inside modern large-scale generative pipelines. We expand the discussion of limitations and future directions in Appendix~\ref{app:limitations_future_work}.

\clearpage
\section*{Impact Statement}

This paper presents work whose goal is to advance the field of Machine Learning, in particular through rank-based approximations of $f$-divergences and sample-based objectives for generative modeling. The proposed methods are generic and are not designed for identification, surveillance, or applications involving sensitive personal data.

Our empirical evaluation focuses on standard natural-image benchmarks and avoids datasets centered on identifiable facial images. This choice keeps the experiments focused on the methodological properties of the proposed objective while reducing privacy-related concerns. We do not identify any additional societal consequences that require specific highlighting beyond these general considerations.

\section*{Acknowledgements}

The authors thank the reviewers for their comments, which improved this submission.
VS thanks his advisor, Gabriele Steidl, for her support and guidance throughout, as well as the organizers of the \href{https://sites.google.com/view/msml2025/home?authuser=0}{\textit{MSML 2025} conference} in Naples.

This project has received funding from the European Research Council (ERC) under the European Union’s Horizon Europe research and innovation programme (grant agreement No. 101198055, project acronym NEITALG). This work has also been partially supported by the Office of Naval Research (award N00014-22-1-2647) and Spain’s Agencia Estatal de Investigación (ref. PID2024-158181NB-I00 NISA and PID2021-123182OB-I00 EPiCENTER) funded by MCIN/AEI/10.13039/501100011033 and by “ERDF A way of making Europe."

Views and opinions expressed are, however, those of the author(s) only and do not necessarily reflect those of the European Union, the European Research Council Executive Agency, the U.S. Office of Naval Research, or the Spanish Agencia Estatal de Investigación. Neither the European Union nor any of the aforementioned granting authorities can be held responsible for them.

\bibliography{Bibliography}

@article{SNRS2025,
author = {Stein, Viktor and Neumayer, Sebastian and Rux, Nicolaj and Steidl, Gabriele},
title = {Wasserstein gradient flows for {M}oreau envelopes of $f$-divergences in reproducing kernel {H}ilbert spaces},
journal = {Analysis and Applications},
volume = {24},
number = {01},
pages = {21-65},
year = {2026},
doi = {10.1142/S0219530525500162},
publisher={World Scientific}
}

@inproceedings{dFVOM2025,
  title     = {Explicit Density Approximation for Neural Implicit Samplers Using a {B}ernstein-Based Convex Divergence},
  author    = {de Frutos, Jos{\'e} Manuel and V{\'a}zquez, Manuel Alberto and Olmos, Pablo M. and M{\'\i}guez, Joaqu{\'\i}n},
  booktitle = {Proceedings of The 29th International Conference on Artificial Intelligence and Statistics},
  series    = {Proceedings of Machine Learning Research},
  publisher = {PMLR},
  year      = {2026},
  note      = {To appear}
}

@inproceedings{dFVOM2024,
  title={Training implicit generative models via an invariant statistical loss},
  author={de Frutos, Jos{\'e} Manuel and Olmos, Pablo M and Lopez, Manuel Alberto Vazquez and M{\'\i}guez, Joaqu{\'\i}n},
  booktitle={International Conference on Artificial Intelligence and Statistics},
  pages={2026--2034},
  year={2024},
  organization={PMLR}
}

@unpublished{dFVOM2024b,
  title={Robust training of implicit generative models for multivariate and heavy-tailed distributions with an invariant statistical loss},
  author={de Frutos, Jos{\'e} Manuel and V{\'a}zquez, Manuel A and Olmos, Pablo and M{\'\i}guez, Joaqu{\'\i}n},
  note={arXiv preprint arXiv:2410.22381},
  year={2024}
}

@unpublished{RRGDP2012,
  title={Tighter variational representations of {$f$}-divergences via restriction to probability measures},
  author={Ruderman, Avraham and Reid, Mark and Garc{\'\i}a-Garc{\'\i}a, Dar{\'\i}o and Petterson, James},
  note={arXiv preprint arXiv:1206.4664},
  year={2012}
}

@article{B1912,
  title={Démonstration du théorème de {W}eierstrass fondée sur le calcul des probabilités},
  author={Bernstein, Serge},
  journal={Communications of the Kharkiv Mathematical Society},
  volume={13},
  number={1},
  pages={1--2},
  year={1912},
  publisher={Kharkiv Imperial University},
  language = {French}
}

@article{EMD2021,
  title={On the performance of particle filters with adaptive number of particles},
  author={Elvira, V{\'\i}ctor and Miguez, Joaqu{\'\i}n and Djuri{\'c}, Petar M},
  journal={Statistics and Computing},
  volume={31},
  number={6},
  pages={81},
  year={2021},
  publisher={Springer}
}

@article{massart1990tight,
  title={The tight constant in the {D}voretzky-{K}iefer-{W}olfowitz inequality},
  author={Massart, Pascal},
  journal={The Annals of Probability},
  pages={1269--1283},
  year={1990},
  publisher={JSTOR}
}

@book{boucheron2013concentration,
  author    = {St{\'e}phane Boucheron and G{\'a}bor Lugosi and Pascal Massart},
  title     = {Concentration Inequalities: {A} Nonasymptotic Theory of Independence},
  publisher = {Oxford University Press},
  series    = {Oxford Series in Probability and Its Applications},
  address   = {Oxford, UK},
  year = {2013},
  month = {02},
  isbn = {9780199535255},
  doi = {10.1093/acprof:oso/9780199535255.001.0001},
}

@book{vandervaart_wellner_1996,
  author    = {van der Vaart, Aad W. and Wellner, Jon A.},
  title     = {Weak Convergence and Empirical Processes: With Applications to Statistics},
  year      = {1996},
  publisher = {Springer},
  series    = {Springer Series in Statistics},
  address   = {New York}
}

@inproceedings{sreekumar2021non,
  title={Non-asymptotic performance guarantees for neural estimation of {$f$}-divergences},
  author={Sreekumar, Sreejith and Zhang, Zhengxin and Goldfeld, Ziv},
  booktitle={International Conference on Artificial Intelligence and Statistics},
  pages={3322--3330},
  year={2021},
  organization={PMLR}
}

@article{ali1966,
  title={A general class of coefficients of divergence of one distribution from another},
  author={Ali, Syed Mumtaz and Silvey, Samuel D},
  journal={Journal of the Royal Statistical Society: Series B (Methodological)},
  volume={28},
  number={1},
  pages={131--142},
  year={1966},
  publisher={Wiley Online Library}
}

@article{L94,
author = {Bruce G. Lindsay},
title = {Efficiency Versus Robustness: The Case for Minimum {H}ellinger Distance and Related Methods},
volume = {22},
journal = {The Annals of Statistics},
number = {2},
publisher = {Institute of Mathematical Statistics},
pages = {1081 -- 1114},
year = {1994},
doi = {10.1214/aos/1176325512},
}

@article{L1991,
  author={Lin, J.},
  journal = {IEEE Transactions on Information Theory}, 
  title={Divergence measures based on the {S}hannon entropy}, 
  year={1991},
  volume={37},
  number={1},
  pages={145-151},
  doi={10.1109/18.61115}
}

@article{K2025,
  title={On the tensorization of the variational distance},
  author={Kontorovich, Aryeh},
  journal={Electronic Communications in Probability},
  volume={30},
  pages={1--10},
  year={2025},
  publisher={The Institute of Mathematical Statistics and the Bernoulli Society}
}

@unpublished{K2026,
  title={{TV} homogenization inequalities},
  author={Kontorovich, Aryeh},
  note={arXiv preprint arXiv:2601.04079},
  year={2026}
}

@article{csiszar1967,
  author    = {Csisz{\'a}r, Imre},
  title     = {Information-Type Measures of Difference of Probability Distributions and Indirect Observations},
  journal   = {Studia Scientiarum Mathematicarum Hungarica},
  volume    = {2},
  pages     = {299--318},
  year      = {1967},
  mrnumber  = {0219345}
}

@book{L1968,
  author    = {Lorentz, George G.},
  title     = {Bernstein Polynomials},
  series    = {AMS Chelsea Publishing},
  volume    = {323},
  edition   = {2},
  publisher = {Chelsea Publishing Company},
  address   = {New York, NY},
  year      = {1986},
  pages     = {134},
  isbn      = {9780828403238}
}

@book{PW2025,
  author    = {Polyanskiy, Yury and Wu, Yihong},
  title     = {Information Theory: From Coding to Learning},
  publisher = {Cambridge University Press},
  year      = {2025},
  edition   = {1st},
  doi       = {10.1017/9781108966351},
  isbn      = {9781108832908}
}

@techreport{S1952,
  title={A note on non-parametric methods},
  author={Savage, I Richard},
  year={1952},
  institution={U.S. Department of Commerce National Bureau of Standards},
  note = {NBS project 1103-11-1107, NBS report 1699}
}

@article{BGRKSS06,
    author = {Borgwardt, Karsten M. and Gretton, Arthur and Rasch, Malte J. and Kriegel, Hans-Peter and Schölkopf, Bernhard and Smola, Alex J.},
    title = {Integrating structured biological data by Kernel Maximum Mean Discrepancy},
    journal = {Bioinformatics},
    fjournal = {Bioinformatics},
    volume = {22},
    number = {14},
    pages = {e49-e57},
    year = {2006},
    month = {07},
    issn = {1367-4803},
    doi = {10.1093/bioinformatics/btl242},
}

@InProceedings{Beckmann2025a,
  author = {Beckmann, Matthias and Beinert, Robert and Bresch, Jonas},
  title     = {Max-normalized {R}adon cumulative distribution transform for limited data classification},
  booktitle = {International Conference on Scale Space and Variational Methods in Computer Vision (SSVM)},
  location = {Dartington, UK},
  series    = {Lect. Notes Comput. Sci.},
  year      = {2025},
  volume= {{15667}},
  pages = {241--254},
  publisher = {Springer},
  address   = {Cham, Switzerland},
  doi    = {10.1007/978-3-031-92366-1\_19},
}

@article{M1997,
  title={Integral probability metrics and their generating classes of functions},
  author={M{\"u}ller, Alfred},
  journal={Advances in applied probability},
  volume={29},
  number={2},
  pages={429--443},
  year={1997},
  publisher={Cambridge University Press}
}

@article{L2012,
  title={On estimating distribution functions using {Bernstein} polynomials},
  author={Leblanc, Alexandre},
  journal={Annals of the Institute of Statistical Mathematics},
  volume={64},
  number={5},
  pages={919--943},
  year={2012},
  publisher={Springer}
}

@article{K1957,
  title={Contributions to the theory of convex bodies},
  author={Knothe, Herbert},
  journal={Michigan Mathematical Journal},
  volume={4},
  number={1},
  pages={39--52},
  year={1957},
  publisher={University of Michigan, Department of Mathematics}
}

@article{nguyen2010divergence,
  title={Estimating divergence functionals and the likelihood ratio by convex risk minimization},
  author={Nguyen, XuanLong and Wainwright, Martin J and Jordan, Michael I},
  journal={IEEE Transactions on Information Theory},
  volume={56},
  number={11},
  pages={5847--5861},
  year={2010},
  publisher={IEEE}
}

@inproceedings{goodfellow2014gan,
  author    = {Goodfellow, Ian J. and Pouget-Abadie, Jean and Mirza, Mehdi and Xu, Bing and Warde-Farley, David and Ozair, Sherjil and Courville, Aaron and Bengio, Yoshua},
  title     = {Generative Adversarial Nets},
  booktitle = {Advances in Neural Information Processing Systems},
  volume    = {27},
  pages     = {2672--2680},
  year      = {2014}
}

@article{nowozin2016fgan,
  title={{$f$}-{GAN}: {T}raining generative neural samplers using variational divergence minimization},
  author={Nowozin, Sebastian and Cseke, Botond and Tomioka, Ryota},
  journal={Advances in Neural Information Processing Systems},
  volume={29},
  year={2016}
}

@inproceedings{gutmann2010nce,
  title={Noise-contrastive estimation: {A} new estimation principle for unnormalized statistical models},
  author={Gutmann, Michael and Hyv{\"a}rinen, Aapo},
  booktitle={Proceedings of the Thirteenth International Conference on Artificial Intelligence and Statistics},
  pages={297--304},
  year={2010},
  organization={JMLR Workshop and Conference Proceedings}
}

@inproceedings{belghazi2018mine,
  title={Mutual information neural estimation},
  author={Belghazi, Mohamed Ishmael and Baratin, Aristide and Rajeshwar, Sai and Ozair, Sherjil and Bengio, Yoshua and Courville, Aaron and Hjelm, Devon},
  booktitle={International Conference on Machine Learning},
  pages={531--540},
  year={2018},
  organization={PMLR}
}

@inproceedings{arjovsky2017wgan,
  title={Wasserstein generative adversarial networks},
  author={Arjovsky, Martin and Chintala, Soumith and Bottou, L{\'e}on},
  booktitle={International Conference on Machine Learning},
  pages={214--223},
  year={2017},
  organization={PMLR}
}

@article{gulrajani2017wgangp,
  title={Improved training of {W}asserstein {GAN}s},
  author={Gulrajani, Ishaan and Ahmed, Faruk and Arjovsky, Martin and Dumoulin, Vincent and Courville, Aaron C},
  journal={Advances in Neural Information Processing Systems},
  volume={30},
  year={2017}
}

@article{rosenblatt1952,
  title={Remarks on a multivariate transformation},
  author={Rosenblatt, Murray},
  journal={The Annals of Mathematical Statistics},
  volume={23},
  number={3},
  pages={470--472},
  year={1952},
  publisher={JSTOR}
}

@article{hamill2001rankhist,
  title={Interpretation of rank histograms for verifying ensemble forecasts},
  author={Hamill, Thomas M},
  journal={Monthly Weather Review},
  volume={129},
  number={3},
  pages={550--560},
  year={2001}
}

@article{gneiting2007calibration,
  title={Probabilistic forecasts, calibration and sharpness},
  author={Gneiting, Tilmann and Balabdaoui, Fadoua and Raftery, Adrian E},
  journal={Journal of the Royal Statistical Society Series B: Statistical Methodology},
  volume={69},
  number={2},
  pages={243--268},
  year={2007},
  publisher={Oxford University Press}
}

@article{kolouri2019gsw,
  title={Generalized sliced {W}asserstein distances},
  author={Kolouri, Soheil and Nadjahi, Kimia and Simsekli, Umut and Badeau, Roland and Rohde, Gustavo},
  journal={Advances in Neural Information Processing Systems},
  volume={32},
  year={2019}
}

@article{nadjahi2020sliced,
  title={Statistical and topological properties of sliced probability divergences},
  author={Nadjahi, Kimia and Durmus, Alain and Chizat, L{\'e}na{\"\i}c and Kolouri, Soheil and Shahrampour, Shahin and Simsekli, Umut},
  journal={Advances in Neural Information Processing Systems},
  volume={33},
  pages={20802--20812},
  year={2020}
}

@inproceedings{wu2019swgm,
  title={Sliced {W}asserstein generative models},
  author={Wu, Jiqing and Huang, Zhiwu and Acharya, Dinesh and Li, Wen and Thoma, Janine and Paudel, Danda Pani and Gool, Luc Van},
  booktitle={Proceedings of the IEEE/CVF Conference on Computer Vision and Pattern Recognition},
  pages={3713--3722},
  year={2019}
}

@article{rubenstein2019practical,
  title={Practical and consistent estimation of {$f$}-divergences},
  author={Rubenstein, Paul and Bousquet, Olivier and Djolonga, Josip and Riquelme, Carlos and Tolstikhin, Ilya O},
  journal={Advances in Neural Information Processing Systems},
  volume={32},
  year={2019}
}

@article{moon2014multivariate,
  title={Multivariate {$f$}-divergence estimation with confidence},
  author={Moon, Kevin R and Hero, Alfred O},
  journal={Advances in Neural Information Processing Systems},
  volume={27},
  year={2014}
}

@article{sobol1967distribution,
  title={Distribution of points in a cube and approximate evaluation of integrals},
  author={Sobol, Ilya M},
  journal={USSR Computational Mathematics and Mathematical Physics},
  volume={7},
  pages={86--112},
  year={1967}
}

@book{dick2010digital,
  title={Digital nets and sequences: discrepancy theory and quasi--{M}onte {C}arlo integration},
  author={Dick, Josef and Pillichshammer, Friedrich},
  year={2010},
  publisher={Cambridge University Press}
}

@inproceedings{deshpande2019max,
  title={Max-sliced {W}asserstein distance and its use for {GAN}s},
  author={Deshpande, Ishan and Hu, Yuan-Ting and Sun, Ruoyu and Pyrros, Ayis and Siddiqui, Nasir and Koyejo, Sanmi and Zhao, Zhizhen and Forsyth, David and Schwing, Alexander G},
  booktitle={Proceedings of the IEEE/CVF Conference on Computer Vision and Pattern Recognition},
  pages={10648--10656},
  year={2019}
}

@inproceedings{paty2019subspace,
  title={Subspace robust {Wasserstein} distances},
  author={Paty, Fran{\c{c}}ois-Pierre and Cuturi, Marco},
  booktitle={International conference on machine learning},
  pages={5072--5081},
  year={2019},
  organization={PMLR}
}

@inproceedings{nietert2022statistical,
  title={Statistical, robustness, and computational guarantees for sliced {Wasserstein} distances},
  author={Nietert, Sloan and Goldfeld, Ziv and Sadhu, Ritwik and Kato, Kengo},
  booktitle={Advances in Neural Information Processing Systems},
  volume={35},
  pages={28179--28193},
  year={2022}
}

@inproceedings{liutkus2019sliced,
  title={Sliced-{W}asserstein flows: {N}onparametric generative modeling via optimal transport and diffusions},
  author={Liutkus, Antoine and Simsekli, Umut and Majewski, Szymon and Durmus, Alain and St{\"o}ter, Fabian-Robert},
  booktitle={International Conference on Machine Learning},
  pages={4104--4113},
  year={2019},
  organization={PMLR}
}

@article{sajjadi2018assessing,
  title={Assessing generative models via precision and recall},
  author={Sajjadi, Mehdi SM and Bachem, Olivier and Lucic, Mario and Bousquet, Olivier and Gelly, Sylvain},
  journal={Advances in Neural Information Processing Systems},
  volume={31},
  year={2018}
}

@article{radford2015unsupervised,
  title={Unsupervised Representation Learning with Deep Convolutional Generative Adversarial Networks},
  author={Radford, Alec and Metz, Luke and Chintala, Soumith},
  journal={arXiv preprint arXiv:1511.06434},
  year={2015}
}

@article{choi2022mcl,
  title={{MCL-GAN}: {G}enerative adversarial networks with multiple specialized discriminators},
  author={Choi, Jinyoung and Han, Bohyung},
  journal={Advances in Neural Information Processing Systems},
  volume={35},
  pages={29597--29609},
  year={2022}
}

@article{durugkar2016generative,
  title={Generative multi-adversarial networks},
  author={Durugkar, Ishan and Gemp, Ian and Mahadevan, Sridhar},
  journal={arXiv preprint arXiv:1611.01673},
  year={2016}
}

@article{hendrycks2019benchmarking,
  title={Benchmarking neural network robustness to common corruptions and perturbations},
  author={Hendrycks, Dan and Dietterich, Thomas},
  journal={arXiv preprint arXiv:1903.12261},
  year={2019}
}

@techreport{krizhevsky2009learning,
  title={Learning multiple layers of features from tiny images},
  author={Krizhevsky, Alex},
  institution={University of Toronto},
  year={2009}
}

@inproceedings{roelofs2022mitigating,
  title     = {Mitigating Bias in Calibration Error Estimation},
  author    = {Roelofs, Rebecca and Cain, Nicholas and Shlens, Jonathon and Mozer, Michael C.},
  booktitle = {International Conference on Artificial Intelligence and Statistics},
  pages     = {4036--4054},
  year      = {2022},
  organization = {PMLR}
}

@inproceedings{VNPC2023,
  title={Precision-Recall Divergence Optimization for Generative Modeling with GANs and Normalizing Flows},
  author={V{\'e}rine, Alexandre and N{\'e}grevergne, Benjamin and Pydi, Muni Sreenivas and Chevaleyre, Yann},
  booktitle={Advances in Neural Information Processing Systems},
  volume={36},
  pages={32539--32573},
  year={2023}
}

@article{SWSR2023,
  title={On the Theoretical Equivalence of Several Trade-Off Curves Assessing Statistical Proximity},
  author={Siry, Rodrigue and Webster, Ryan and Simon, Lo{\"i}c and Rabin, Julien},
  journal={Journal of Machine Learning Research},
  volume={24},
  number={185},
  pages={1--34},
  year={2023}
}

@inproceedings{SWR2019,
  title={Revisiting Precision and Recall Definition for Generative Model Evaluation},
  author={Simon, Lo{\"i}c and Webster, Ryan and Rabin, Julien},
  booktitle={Proceedings of the 36th International Conference on Machine Learning},
  pages={5799--5808},
  year={2019},
  publisher={PMLR}
}

@article{gretton2012kernel,
  title={A Kernel Two-Sample Test},
  author={Gretton, Arthur and Borgwardt, Karsten M. and Rasch, Malte J. and Sch{\"o}lkopf, Bernhard and Smola, Alexander},
  journal={Journal of Machine Learning Research},
  volume={13},
  pages={723--773},
  year={2012}
}

@inproceedings{lopezpaz2017revisiting,
  title={Revisiting Classifier Two-Sample Tests},
  author={Lopez-Paz, David and Oquab, Maxime},
  booktitle={International Conference on Learning Representations},
  year={2017}
}

@article{hotelling1931generalization,
  title={The Generalization of Student's Ratio},
  author={Hotelling, Harold},
  journal={The Annals of Mathematical Statistics},
  volume={2},
  number={3},
  pages={360--378},
  year={1931}
}

@article{cressie1984multinomial,
  title={Multinomial Goodness-of-Fit Tests},
  author={Cressie, Noel A. C. and Read, Timothy R. C.},
  journal={Journal of the Royal Statistical Society: Series B (Methodological)},
  volume={46},
  number={3},
  pages={440--464},
  year={1984}
}
\bibliographystyle{icml2026}

\newpage
\appendix
\onecolumn

\section*{Appendix}
In this appendix, we first recall well-known results about Bernstein polynomials and $f$-divergences in \cref{sec:well_known_results}.
In \cref{sec:proofs} we prove the theorems from the main text, and in \cref{sec:experiments_Appendix} we provide supplementary explanations and experiments.
Finally, in \cref{app:limitations_future_work}, we elaborate on limitations and future work.

\begin{table*}[ht]
\centering

\label{tab:notation}
\renewcommand{\arraystretch}{1.12}
\begin{tabularx}{\textwidth}{@{}lX@{}}
\toprule
Symbol & Meaning \\
\midrule
\(\mathbb{N}\) & Set of non-negative integers. \\
\([K]\) & Discrete set \(\{0,\ldots,K\}\). \\
\(\mathcal{P}(\mathbb{R}^d)\) & Set of probability measures on \(\mathbb{R}^d\). \\
\(\mu,\nu\) & Probability measures to be compared; typically \(\nu\) is the reference distribution. \\
\(\hat\mu_N,\hat\nu_M\) & Empirical measures built from \(N\) samples from \(\mu\) and \(M\) samples from \(\nu\). \\
\(Q_\mu\), \(R_\mu\) & Quantile function and cumulative distribution function of a univariate measure \(\mu\). \\
\(T_\#\mu\) & Pushforward of \(\mu\) by a map \(T\). \\
$\mu \kappa$ & Pushforward of $\mu\in \P(\R)$ by a Markov kernel $\kappa \colon [K] \times \R \to [0, 1]$, i.e. $(\mu \kappa)(n) = \int_{\R} \kappa(n, x) \d \mu(x)$ for $n \in [K]$. \\
\(f\) & Entropy function/generator of an \(f\)-divergence. \\
\(D_{f,\nu}(\mu)\) & Continuous \(f\)-divergence of \(\mu\) with respect to \(\nu\). \\
\(K\) & Rank resolution parameter. \\
\(U_K\) & Uniform distribution on \([K]\). \\
\(A^{(K)}_{\mu|\nu}\) & Rank statistic of order \(K\) of \(\mu\) with respect to \(\nu\). \\
\(Q^{(K)}_{\mu|\nu}\) & Probability mass function of \(A^{(K)}_{\mu|\nu}\), interpreted as a rank histogram. \\
\(D^{(K)}_{f,\nu}(\mu)\) & Rank-statistic approximation of \(D_{f,\nu}(\mu)\) at resolution \(K\). \\
\(r_{\mu|\nu}\) & Quantile-domain density ratio \(r_{\mu|\nu} = \frac{d\mu}{d\nu}\circ Q_\nu\), when \(\mu\ll\nu\). \\
\(r_K\) & Piecewise-constant approximation of \(r_{\mu|\nu}\) induced by the rank histogram. \\
\(b_{n,K}\) & Bernstein basis polynomial of degree \(K\). \\
\(s\in\mathbb{S}^{d-1}\) & One-dimensional projection direction. \\
\(\mu_s,\nu_s\) & One-dimensional projected measures along direction \(s\). \\
\(\sigma\) & Uniform probability measure on the sphere \(\mathbb{S}^{d-1}\). \\
\(\mathbf{SD}_{f,\nu}(\mu)\) & Sliced \(f\)-divergence obtained by averaging one-dimensional divergences over projections. \\
\(\mathbf{D}^{(K)}_{f,\nu}(\mu)\) & Sliced rank-statistic \(f\)-divergence of order \(K\). \\
\(L\) & Number of random projection directions used in empirical sliced estimators. \\
\(\tau\) & Smoothing temperature used in differentiable rank/CDF approximations. \\
\bottomrule
\end{tabularx}
\caption{Frequently used notation.}
\end{table*}

\section{Well-known results} \label{sec:well_known_results}

Here, we recall results about Bernstein polynomials and $f$-divergences.

Bernstein polynomials were introduced in \cite{B1912} to prove the Weierstraß approximation theorem in a simple way.

    \begin{lemma}[Properties of Bernstein polynomials] \label{lemma:Bernstein_Facts}
        The Bernstein polynomials $b_{n, K} \colon [0, 1] \to [0, \infty)$, $u \mapsto \binom{K}{n} u^{n} (1 - u)^{K - n}$ have the following properties.
        \begin{enumerate}
            \item
            We have $\sum_{n = 0}^{K} b_{n, K}(s) = 1$ for all $s \in [0, 1]$, and  $\int_{0}^{1} b_{n, K}(s) \d{s} = \frac{1}{K + 1}$ for $n \in [K]$. 

            \item 
            For $K \in \N$ and $n \in [K - 1]$ we have
            \begin{equation*}
                b_{n, K - 1}
                = \frac{K - n}{K} b_{n, K}
                + \frac{n + 1}{K} b_{n + 1, K}.
            \end{equation*}

            \item 
            The function $(K + 1) b_{n, K}$ is the probability density function of the $\Beta(n + 1, K - n + 1)$ distribution, whose mean and variance are $\frac{n + 1}{K + 2}$ and $\frac{(n + 1)(K - n + 1)}{(K + 2)^2 (K + 3)}$, respectively.

            \item 
            Let $f \in \C([0, 1])$. 
            For the Bernstein operator $B_K \colon \C([0, 1]; \R) \to \Pi_K$, $f \mapsto \sum_{n = 0}^{K} f\left(\frac{n}{K}\right) b_{n, K}$ we have $\| B_{K}[f] - f \|_{\infty} \to 0$ for $K \to \infty$.

            \item 
            If $g$ is Lipschitz, then $\| B_K[g] - g \|_{\infty} \in O(K^{-\frac{1}{2}})$ and if $g \in \C^2([0, 1])$, then $\| B_K[g] - g \|_{\infty} \in O(K^{-1})$.
        \end{enumerate}
    \end{lemma}

    \begin{proof}
        See \citep[Chp.~1]{L1968}.
    \end{proof}

\begin{lemma}[Data processing inequality for discrete $f$-divergences] \label{lemma:discrete_DP}
    For finite sets $X$ and $Y$ and a matrix $W \in [0, 1]^{\# Y \times \# X}$ fulfilling $W \1_{\# Y} = \1_{\# X}$ and $P,Q \in \P(X)$ we have
    \[
      \D_f(W^{\tT} P \mid W^{\tT} Q)
      \le \D_f(P \mid Q).
    \]
\end{lemma}

\begin{proof}
    See \citep[Subsec.~7.2]{PW2025}.
\end{proof}

\section{Proofs of theorems} \label{sec:proofs}

First, we repeat some important properties of the rank statistics \citep[App.~A]{dFVOM2024} \citep[Thm.~4.1]{dFVOM2025}.

\subsection{Properties of the rank statistic} \label{appendix:rank-representation}

\begin{lemma}[Properties of the rank statistic] \label{lemma:properties_rank_statistics}
    Let $\mu, \nu \in \P(\R)$ with $\mu \ll \nu$ and $K \in \N_{> 0}$, and let $(b_{n, K})_{n \in [K]}$ be the Bernstein polynomials.
    Suppose that $\nu$ is atomless.
    \begin{enumerate}
        \item 
        We have
        \begin{equation} \label{eq:Q^K_expression}
            Q_{\mu \mid \nu}^{(K)}(n)
            = \int_{\R} b_{n, K}(R_{\nu}(y)) \d{\mu}(y)
            = \int_{0}^{1} b_{n, K}(s) \left(\frac{\d \mu}{\d \nu} \circ Q_{\nu}\right)(s) \d{s}, \qquad \forall n \in [K],
        \end{equation}
        where $R_{\nu}$ and $Q_{\nu}$ are the cumulative distribution function (CDF) and the quantile function of $\nu$, respectively.

        \item 
        Furthermore, we have $\mu = \nu$ if and only if for all $K \in \N$ we have $A_{\mu \mid \nu}^{(K)} \sim U_K$, i.e., $Q_{\mu \mid \nu}^{(K)}(n) = \frac{1}{K + 1}$ for all $n \in [K]$.
    \end{enumerate}    
\end{lemma}

\begin{proof}
    \begin{enumerate}
        \item 
        Given $y \sim \mu$, the random variable $A_{\mu \mid \nu}^{(K)} \mid y$ follows a $\Bin(K, R_{\nu}(y))$ distribution, whose probability mass function is $[K] \ni n \mapsto b_{n, K}(R_{\nu}(y))$.
        By the law of total probability, we thus have
        \begin{equation*}
            Q_{\mu \mid \nu}^{(K)}(n)
            = \PP(A_{\mu \mid \nu}^{(K)} = n)
            = \int_{\R} \PP\left(A_{\mu \mid \nu}^{(K)} = n \mid y\right) \d{\mu}(y)
            = \int_{\R} (b_{n, K} \circ R_{\nu})(y) \d{\mu}(y)
        \end{equation*}
        as in \citep[Eq.~(B.2)]{EMD2021}. 
        The second equation follows from the change of variables formula for the pushforward measure.

        \item 
        If $\mu = \nu$ and $K \in \N$, then 
        \begin{equation*}
            Q_{\mu \mid \nu}^{(K)}(n)
            = \int_{0}^{1} b_{n, K}(s) \left(\frac{\d \mu}{\d \nu} \circ Q_{\nu}\right)(s) \d{s}
            = \int_{0}^{1} b_{n, K}(s) \d{s}
            = \frac{1}{K + 1}, \qquad \forall n \in [K].
        \end{equation*}
        The converse direction follows as in the proof of \citep[Thm.~3]{dFVOM2024}, the assumption that $\mu$ and $\nu$ admit densities is not needed.\qedhere
    \end{enumerate}
\end{proof}

\begin{figure}[H]
    \centering
    \includegraphics[width=0.33\linewidth,page=1]{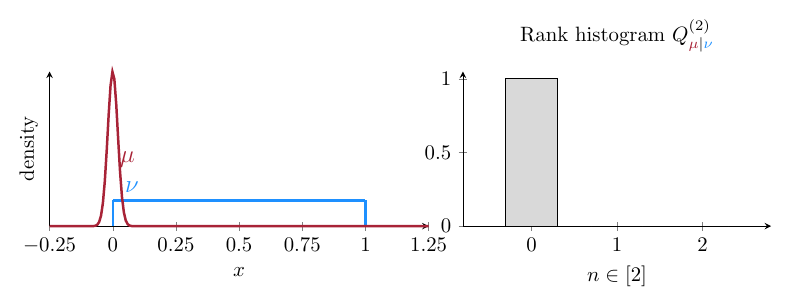}%
    \includegraphics[width=0.33\linewidth,page=3]{Bernstein_plot.pdf}%
    \includegraphics[width=0.33\linewidth,page=5]{Bernstein_plot.pdf}%
    \caption{Illustration of the rank histogram $Q_{\mu \mid \nu}^{(K)}$ for $K = 2$, $\nu \sim U([0, 1])$ and $\mu$ being a Gaussian with varying mean.}
    \label{fig:Bernstein_illustration}
\end{figure}

In the next remark, we illustrate why measuring the deviation of the rank histogram $Q_{\mu \mid \nu}^{(K)}$ using an $f$-divergence is meaningful.

\begin{remark}[Suitability of $f$-divergences] \label{remark:f-div_pushforward_reformulation}
    For $\mu, \nu \in \P(\R)$ with $\mu \ll \nu$, we have 
    \begin{equation*} \label{eq:uniform_reformulation}
        D_{f, \nu}(\mu)
        = \int_{\R} f\left(\frac{\d \mu}{\d \nu}(x)\right) \d{[Q_{\nu} \# \lambda_{(0, 1)}]}(x)
        = \int_{0}^{1} f\left(\left(\frac{\d \mu}{\d \nu} \circ Q_{\nu}\right)(s)\right) \d{s}
        = D_{f, \lambda_{(0, 1)}}\left( \frac{\d \mu}{\d \nu} \circ Q_{\nu} \cdot \lambda_{(0, 1)}\right),
    \end{equation*}
    where $\lambda_{(0, 1)}$ is the Lebesgue measure on $(0, 1)$.
    Hence, the $f$-divergence between $\mu$ and $\nu$ can be rewritten as the $f$-divergence between two densities on the (bounded) unit interval.
    The same holds for the $\alpha$-Rényi divergences, and we leave the exploration of rank-statistic approximations of Rényi divergences for future work.

    This clean reformulation is not possible for other discrepancies, like integral probability metrics or Wasserstein distances.
\end{remark}

\subsection{Proof of \cref{thm:monotone-K}} \label{subsec:Proof_of_Monotonicity}

\MonotoneK*

\begin{proof}
    \begin{enumerate}
        \item 
        First, we show that $D_{f, \nu}^{(K)} \le D_{f, \nu}$.
    
        Let $U_{n,K}\sim \mathrm{Beta}(n+1,K{-}n+1)$ and $r \coloneqq \frac{\d \mu}{\d \nu} \circ Q_{\nu}$.
        Then,
        \[
            c^{(K)}_n
            \coloneqq
            (K+1)\,Q^{(K)}_{\mu \mid\nu}(n)
            = \int_0^1 (K+1)b_{n,K}(u)\,r(u) \d{u}
            = \mathbb{E}\big[r(U_{n,K})\big],
        \]
        By Jensen's inequality, we have
        \begin{align*}
        D^{(K)}_f(\mu \mid \nu)
        & =\frac{1}{K+1}\sum_{n=0}^K f\!\big(c^{(K)}_n\big)
        \le \frac{1}{K+1}\sum_{n=0}^K \mathbb{E}\big[f(r(U_{n,K}))\big] \\
        & = \frac{1}{K + 1} \sum_{n = 0}^{K} \int_{0}^{1} f(r(u)) (K + 1) b_{n, K}(u) \d{u}
        =\int_0^1 f(r(u))\d{u},
        \end{align*}
        where in the last step we used
        $\frac{1}{K+1}\sum_{n=0}^K (K+1)b_{n,K}\equiv 1$.

        \item 
        Now, we prove the monotonicity with respect to $K$.

        We want to use the data-processing inequality for discrete $f$-divergences (\cref{lemma:discrete_DP}) with $X \coloneqq [K + 1]$, $Y \coloneqq [K]$, $P \coloneqq Q^{(K+1)}_{\mu\mid\nu}$, and $Q \coloneqq \U([K+1])$ and construct $W$ such that $W^{\tT} P = Q^{(K+1)}_{\mu\mid\nu}$ and $W^{\tT} Q = U_K$.
        Then,
        \begin{align*}
          D^{(K)}_{f,\nu}(\mu)
          &= \D_f\big(Q^{(K)}_{\mu\mid\nu} \mid U_K\big)
           = \D_f\big(Q^{(K+1)}_{\mu\mid\nu} W \mid U_{K+1} W\big)
          \\
          &\le \D_f\big(Q^{(K+1)}_{\mu\mid\nu} \,\Vert\, U_{K+1}\big)
           = D^{(K+1)}_{f,\nu}(\mu).
        \end{align*}
        We set
        \begin{equation*}
            W_{n, m}
            \coloneqq
            \begin{cases}
                \frac{m}{K+1}, & \text{if } n = m-1,\\
                \frac{K+1-m}{K+1}, & \text{if } n = m,\\
                0, & \text{otherwise.}
            \end{cases}
            \qquad n \in [K], \ m \in [K + 1].
        \end{equation*}
        Then,
        \begin{equation*}
           W \1_{K + 1}
           = \left(\sum_{n = 0}^{K} W_{n, m} \right)_{m \in [K + 1]}
           = \left(W_{m - 1, m} + W_{m, m} \right)_{m \in [K + 1]}
           = \left(\frac{m}{K + 1} + \frac{K + 1 - m}{K + 1}\right)_{m \in [K + 1]}
           = \1_{K + 2}
        \end{equation*}
        and
        \begin{align*}
            W^{\tT} Q
            & = W^{\tT} \U([K + 1])
            = \left(\sum_{m = 0}^{K} W_{n, m} \frac{1}{K + 2}\right)_{n \in [K]}
            = \left(\frac{1}{K + 2} \left( W_{n, n} + W_{n, n + 1}\right)\right)_{n \in [K]} \\
            & = \frac{1}{K + 2} \left( \frac{K + 1 - n}{K + 1} + \frac{n + 1}{K + 1}\right)_{n \in [K]}
            = \frac{1}{K + 1} \1_{K + 1}
            = U_K.
        \end{align*}
        Next,
        \begin{align*}
            W^{\tT} P
            & = W^{\tT} Q_{\mu \mid \nu}^{(K + 1)}
            = \left(W_{n, n} Q_{\mu \mid \nu}^{(K + 1)}(n) + W_{n, n + 1} Q_{\mu \mid \nu}^{(K + 1)}(n + 1) \right)_{n \in [K]} \\
            & = \left(\frac{K + 1 - n}{K + 1} \int_{\R} b_{n, K + 1} \circ R_{\nu} \d{\mu}
            + \frac{n + 1}{K + 1} \int_{\R} b_{n, K + 1} \circ R_{\nu} \d{\mu}\right)_{n \in [N]} \\
            & = \left(\int_{\R} b_{n, K} \circ R_{\nu} \d{\mu}\right)_{n \in [K]}
            = Q_{\mu \mid \nu}^{(K)}.
        \end{align*}
        The identity about Bernstein polynomials we use in the last line expresses that starting from a count $n$ of ``successes'' among $K+1$
        trials ($n$ of the $K+1$ samples that are drawn from $\nu$ are lower than $Y\sim \mu$), we delete one trial uniformly at random. With probability
        $n/(K+1)$ we delete a success and the new count is $n-1$, while
        with probability $(K+1-n)/(K+1)$ we delete a failure and the new
        count remains $n$.
        
        \item 
        By \eqref{eq:Q^K_expression}, $Q_{\mu \mid \nu}^{(K)}(n)$ is linear in $\mu$.
        Since $f$ is convex, the functional $D_{f, \nu}^{(K)}$ is convex as well.

        \item 
        If $Q_{\nu}$ is continuous, then by \eqref{eq:Q^K_expression}, $Q_{\mu \mid \nu}^{(K)}$ is weakly continuous, since the integrand is continuous.
        Hence, $D_{f, \nu}^{(K)}$ is weakly lower semicontinuous by the lower semicontinuity of $f$.
        \qedhere
    \end{enumerate}
\end{proof}

\subsection{Proof of \cref{thm:convergence_K_to_Infty}}

First, we prove the convergence if $r \in \C([0, 1])$.

\begin{proof} \label{proof:convergenceKtoInfty}
    The only point not proved in the proof sketch in the main text is the uniform convergence of $r_K$ to $r$ for $K \to \infty$.

    Indeed, for $u \in \left[\frac{n}{K+1},\frac{n+1}{K+1}\right)$ we have
    \begin{equation*}
        \left| r_K(u) - r(u) \right|
        = \left| c_n^{(K)} - r(u) \right|
        \le \left| c_n^{(K)} - r\left(\frac{n}{K}\right)\right| + \left| r\left(\frac{n}{K}\right) - r(u) \right|.
    \end{equation*}
    The second summand is bounded above as follows:
    \begin{equation*}
        \left| r\left(\frac{n}{K}\right) - r(u) \right|
        \le \sup\left\{ | r(x) - r(y) |: | x - y | \le \frac{1}{K + 1}\right\},
        \eqqcolon \omega_r\left(\frac{1}{K + 1}\right),
    \end{equation*}
    which is the \textit{modulus of continuity} of $r$ at $\frac{1}{K + 1}$.
    Hence,
    \begin{equation} \label{eq:binning}
        \| r_K - r \|_{\infty}
        \le \max_{n \in [K]} \left| c_n^{(K)} - r\left(\frac{n}{K}\right)\right|
        + \omega_r\left(\frac{1}{K + 1}\right).
    \end{equation}
    Since $r$ is uniformly continuous, the second summand vanishes for $K \to \infty$.
    
    We now upper bound the first summand.
    Let $U_{n, K} \sim \Beta(n + 1, K + 1 - n)$.
    The density of $U_{n, K}$ is $(K + 1) b_{n, K}$ and its mean is $m_{n, K} \coloneqq \frac{n + 1}{K + 2}$.
    For $n \in [K]$, we have
    \begin{equation*}
        \left| \E[r(U_{n, K})] - r\left(\frac{n}{K}\right) \right|
        \le \E\left[\left| r(U_{n, K}) - r(m_{n, K}) \right|\right] + \left| r(m_{n, K}) - r\left(\frac{n}{K}\right)\right|.
    \end{equation*}
    Furthermore,
    \begin{equation} \label{eq:deviation_of_mean}
        \left| m_{n, K} - \frac{n}{K} \right|
        = \left| \frac{K - 2 n}{K(K + 2)} \right|
        \le \frac{1}{K + 2}
        < \frac{1}{K} \to 0, \qquad K \to \infty,
    \end{equation}
    We have
    \begin{equation} \label{eq:V_U_nK}
        \V[U_{n, K}]
        = \frac{(n + 1)(K - n + 1)}{(K + 2)^2 (K + 3)}
        \le \frac{1}{4(K + 3)}
        < \frac{1}{K}.
    \end{equation}
    If the modulus of continuity $\omega_r$ is concave, then by Jensen's inequality, we have
        \begin{align*}
            \left| \E[r(U_{n, K})] - r\left(\frac{n}{K}\right)\right|
            & \le \E\left[ \left| r(U_{n, K}) - r\left(\frac{n}{K}\right) \right|\right]
            \le \E\left[ \omega_r\left(U_{n, K} - \frac{n}{K}\right) \right] \\
            & \le \omega_r\left(\E\left[\left|U_{n, K} - \frac{n}{K}\right|\right] \right)
            \le \omega_r\left(\frac{1}{K + 2} + \frac{1}{2\sqrt{K+3}} \right),
        \end{align*}
        where we use \cref{eq:deviation_of_mean,eq:V_U_nK} in the last inequality.
        Hence, by \eqref{eq:binning},
        \begin{equation*}
            \| r_K - r \|_{\infty}
            \le \omega_r\left(\frac{1}{K + 2} + \frac{1}{2\sqrt{K+3}} \right)
            + \omega_r\left(\frac{1}{K + 1}\right)
            \xrightarrow{K \to \infty} 0.
        \end{equation*}
\end{proof}

Now, let us prove the convergence rates.

\begin{lemma}
    If $f$ is Lipschitz on $\ran(r)$, then
    \begin{equation*}
        D_{f, \nu}(\mu) - D_{f, \nu}^{(K)}
        \in \begin{cases}
            O(K^{-\frac{\alpha}{2}}), & \text{if } r \in C^{0, \alpha}([0, 1]), \\
            O(K^{-1}), & \text{if } r \in \C^2([0, 1]), \text{or if } r \text{ is Lipschitz}, f \in \C^2([0, \infty)).
        \end{cases}
    \end{equation*}
\end{lemma}

\begin{proof}
    \begin{enumerate}

        \item 
        If $r$ is $H_r$-Hölder continuous with exponent $\alpha \in (0, 1]$, then $\omega_r(\delta) = H_r \delta^{\alpha}$, so the bound becomes
        \begin{equation*}
            D_{f, \nu}(\mu) - D_{f, \nu}^{(K)}(\mu)
            \le L_f H_r \left( \left(\frac{1}{K + 2} + \frac{1}{2\sqrt{K+1}} \right)^{\alpha}
            + \left(\frac{1}{K + 1}\right)^{\alpha}\right).
        \end{equation*}

        \item 
        If $r \in \C^2([0, 1])$.
        By Taylor's theorem, there exists a $\xi$ between $U_{n, K}$ and $\frac{n}{K}$ such that
        \begin{equation*}
            r(U_{n, K})
            = r\left(\frac{n}{K}\right) + r'\left(\frac{n}{K}\right) \left( U_{n, K} -\frac{n}{K}\right)
            + \frac{1}{2} r''(\xi) \left(U_{n, K} - \frac{n}{K}\right)^2,
        \end{equation*}
        so that
        \begin{equation*}
            \left| c_K(n) - r\left(\frac{n}{K}\right) \right|
            \le \| r' \|_{\infty} \left| m_{n, K} - \frac{n}{K}\right| + \frac{1}{2} \| r'' \|_{\infty} \E\left[\left(U_{n, K} -\frac{n}{K}\right)^2\right],
        \end{equation*}
        so by the estimates from \cref{eq:deviation_of_mean,eq:V_U_nK} we obtain again
        \begin{equation*}
            \max_{n \in [K]} \left| c_K(n) - r\left(\frac{n}{K}\right) \right|
            \le \frac{\| r' \|_{\infty}}{K + 2}
            + \frac{\| r'' \|_{\infty}}{8(K + 3)} + \frac{\| r'' \|_{\infty}}{2 (K + 2)^2}.
        \end{equation*}
        Combined with $\omega_r\left(\frac{1}{K + 1}\right) \le \frac{\| r' \|_{\infty}}{K + 1}$, we obtain $\| r_K - r \|_{\infty} \in O(K^{-1})$.

        \item 
        Now assume that $r$ is $L_r$–Lipschitz and $f\in C^2$ with $\|f''\|_\infty\le M_f$.
        For brevity, set
        \[
          Z_{n,K} \coloneqq r(U_{n,K}), \qquad m_{n,K} \coloneqq \E[Z_{n,K}].
        \]
        By Taylor's theorem with remainder, for every $x\in\R$ there exists
        $\xi_x$ on the line segment between $x$ and $m_{n,K}$ such that
        \[
          f(x) = f(m_{n,K}) + f'(m_{n,K})(x-m_{n,K})
                 + \frac{1}{2} f''(\xi_x)(x-m_{n,K})^2.
        \]
        Hence,
        \[
          \big| f(x) - f(m_{n,K}) - f'(m_{n,K})(x-m_{n,K}) \big|
          \le \frac{M_f}{2} (x-m_{n,K})^2.
        \]
        Applying this Taylor expansion with $x = Z_{n,K}$, and using that
        $m_{n,K} = \E[Z_{n,K}]$ together with the bound $\lvert f''\rvert \le M_f$ on the remainder, we obtain
        \begin{align*}
          J_{n,K}
          &= \E\big[f(Z_{n,K})\big] - f\big(\E[Z_{n,K}]\big)
           = \E\big[f(Z_{n,K}) - f(m_{n,K})\big] \\
          &= \E\Big[
                f(Z_{n,K}) - f(m_{n,K})
                - f'(m_{n,K})(Z_{n,K}-m_{n,K})
                + f'(m_{n,K})(Z_{n,K}-m_{n,K})
              \Big] \\
          &= \E\big[
                f(Z_{n,K}) - f(m_{n,K})
                - f'(m_{n,K})(Z_{n,K}-m_{n,K})
              \big]
             \;+\; f'(m_{n,K})\,\E[Z_{n,K}-m_{n,K}] \\
          &= \E\big[
                f(Z_{n,K}) - f(m_{n,K})
                - f'(m_{n,K})(Z_{n,K}-m_{n,K})
              \big]\\
          &\le \E\Big[
                \big|f(Z_{n,K}) - f(m_{n,K}) - f'(m_{n,K})(Z_{n,K}-m_{n,K})\big|
              \Big] \\
          &\le \frac{M_f}{2}\,\E\big[(Z_{n,K}-m_{n,K})^2\big]
           = \frac{M_f}{2}\,\Var(Z_{n,K}).
        \end{align*}
        Using the characterization
        \[
          \V(Z_{n,K}) = \inf_{c\in\R} \E\big[(Z_{n,K}-c)^2\big],
        \]
        and taking $c=r(\E[U_{n,K}])$, we get
        \begin{align*}
          \V(Z_{n,K})
          &= \inf_{c\in\R} \E\big[(r(U_{n,K})-c)^2\big] \\
          &\le \E\big[(r(U_{n,K})-r(\E[U_{n,K}]))^2\big]
           \le L_r^2\,\E\big[(U_{n,K}-\E[U_{n,K}])^2\big]
           = L_r^2\,\V(U_{n,K}).
        \end{align*}
        Combining the two bounds yields
        \[
          J_{n,K} \le \frac{M_f}{2}\,L_r^2\,\V(U_{n,K}).
        \]
        As in \eqref{eq:V_U_nK}, for $U_{n,K}\sim\mathrm{Beta}(n+1,K-n+1)$ we have $\V(U_{n,K}) < \frac{1}{K}$. Hence
        \[
          J_{n,K} \le \frac{M_f}{2}\,L_r^2\,\frac{1}{K}
          = \frac{M_f L_r^2}{2K}.
        \]
        Finally, averaging over $n$,
        \begin{align*}
          0
          &\le D_{f, \nu}(\mu)-D^{(K)}_{f, \nu}(\mu)
           = \frac{1}{K+1}\sum_{n=0}^K J_{n,K}
           \le \frac{1}{K+1}\sum_{n=0}^K \frac{M_f L_r^2}{2K}
           = \frac{M_f L_r^2}{2K}.\qedhere
        \end{align*}
    \end{enumerate}
\end{proof}

\subsection{Proof of \cref{thm:sliced-rank-f-convergence}} \label{subsec:proof_sliced-rank-f-convergence}

\SlicedRankFConvergence*

\begin{proof}
    The convexity follows as in the proof of \cref{thm:monotone-K}, because the pushforward is a linear operation.
    First, note that the absolute continuity $\mu \ll \nu$ implies that the projected measures satisfy $\mu_s \ll \nu_s$ for all $s \in \mathbb{S}^{d - 1}$. By \cref{thm:convergence_K_to_Infty} we have the direction-wise convergence $\lim_{K \to \infty} D^{(K)}_{f, \nu_s}(\mu_s) = D_{f, \nu_s}(\mu_s)$.
    By \cref{thm:monotone-K} and by applying the data–processing inequality for $f$–divergences to the measurable map
    $x\mapsto s^\top x$, we have $D^{(K)}_{f, \nu_s}(\mu_s)\le D_{f, \nu_s}(\mu_s) \le D_{f, \nu}(\mu)$.
    Integrating over the sphere yields \eqref{eq:SR_inequalities}.
    Furthermore, we can thus apply the dominated convergence theorem and obtain
    \[
    \lim_{K\to\infty} \int_{\mathbb S^{d-1}} D^{(K)}_{f, \nu_s}(\mu_s) \d \sigma(s)
    = \int_{\mathbb S^{d-1}} \lim_{K\to\infty} D^{(K)}_{f, \nu_s}(\mu_s) \d \sigma(s)
    = \int_{\mathbb S^{d-1}} D_{f, \nu_s}(\mu_s) \d \sigma(s).\qedhere
    \]
\end{proof}

\subsection{Proof of \Cref{thm:finite_sample_univariate}}

\begin{lemma}[Sampling $\mu$ only]
\label{lem:mu-sampling}
Fix $\nu\in\P(\R)$ and $K\in\N$.
Let $\mu\in\P(\R)$ and let $\hat\mu_N$ be the empirical measure based on
$N$ i.i.d.\ samples from $\mu$.
Then, denoting $\mathbf Q_{\mu \mid \nu}^{(K)} \coloneqq (Q_{\mu \mid \nu}^{(K)}(n))_{n \in [K]}$, we have
\[
  \E\Bigl[
    \bigl\|\mathbf Q_{\hat\mu_N\mid\nu}^{(K)}
            - \mathbf Q_{\mu\mid\nu}^{(K)}\bigr\|_1
  \Bigr]
  \le \frac{K+1}{2\sqrt N}.
\]
\end{lemma}

\begin{proof}
Fix $n\in[K]$ and define, for $x\in\R$,
\begin{equation} \label{eq:phi_n}
  \phi_n(x)
  \coloneqq
  \PP_{\tilde Y_1,\dots,\tilde Y_K\sim\nu}
  \Bigl(\#\{j \in \{1, \ldots, K \}: \tilde Y_j \le x\} = n\Bigr).
\end{equation}
Then, $0\le \phi_n(x)\le 1$ for all $x$.
By \cref{eq:Q^K_expression},
\[
  Q^{(K)}_{\mu\mid\nu}(n)
  = \PP\bigl(A^{(K)}_{\mu\mid\nu}=n\bigr)
  = \E_{\mu}[\phi_n]
\qquad
\text{and}
\qquad 
  Q^{(K)}_{\hat\mu_N\mid\nu}(n)
  = \E_{\hat\mu_N}[\phi_n]
  = \int_{\R} \phi_n(x)\d\hat\mu_N(x)
  = \frac{1}{N}\sum_{i=1}^N \phi_n(X_i).
\]
Thus, for each $n$,
\[
  Q^{(K)}_{\hat\mu_N\mid\nu}(n) - Q^{(K)}_{\mu\mid\nu}(n)
  = \frac{1}{N}\sum_{i=1}^N
      \bigl(\phi_n(X_i)-\E_\mu[\phi_n]\bigr).
\]
Since $X_1,\dots,X_N \stackrel{\text{i.i.d.}}{\sim} \mu$ and
$\phi_n \colon \R\to[0,1]$ is a fixed deterministic function, the random variables
$\phi_n(X_1),\dots,\phi_n(X_N)$ are also i.i.d.\ and take values in $[0,1]$ with $\Var(\phi_n(X_1)) \le \frac{1}{4}$.
Set
\[
  Z_i \coloneqq \phi_n(X_i)-\E_\mu[\phi_n], \qquad i \in \{ 1,\dots,N \}.
\]
Then $(Z_i)_{i=1}^N$ are i.i.d.\ and centered, so by independence
\begin{align*}
  \Var\bigl(Q^{(K)}_{\hat\mu_N\mid\nu}(n)
            - Q^{(K)}_{\mu\mid\nu}(n)\bigr)
  = \Var\!\left(\frac{1}{N}\sum_{i=1}^N Z_i\right) 
  = \frac{1}{N^2}\sum_{i=1}^N \Var(Z_i) 
  = \frac{1}{N}\Var(Z_1)
    = \frac{1}{N}\Var(\phi_n(X_1)) 
  \le \frac{1}{4N}.
\end{align*}
Applying Cauchy–Schwarz yields
\[
  \E\Bigl[
    \bigl\|\mathbf Q_{\hat\mu_N\mid\nu}^{(K)}
            - \mathbf Q_{\mu\mid\nu}^{(K)}\bigr\|_1
  \Bigr]
  \le \sum_{n=0}^K
         \E\bigl[
           \bigl|Q^{(K)}_{\hat\mu_N\mid\nu}(n)
                 - Q^{(K)}_{\mu\mid\nu}(n)\bigr|
         \bigr]
    \le \sum_{n = 0}^{K} \sqrt{\Var\bigl(Q^{(K)}_{\hat\mu_N\mid\nu}(n)
                     - Q^{(K)}_{\mu\mid\nu}(n)\bigr)}
  \le \frac{K+1}{2\sqrt N}.\qedhere
\]
\end{proof}

\begin{lemma}[Sampling $\nu$ only]
\label{lem:nu-sampling}
Fix $K\in\N$.
Let $\nu\in\P(\R)$ and let $\hat\nu_M$ be the empirical measure based on
$M$ i.i.d.\ samples from $\nu$.
Then for any $\mu \in \P(\R)$,
\[
  \E\Bigl[
    \bigl\|\mathbf Q_{\mu\mid\hat\nu_M}
            - \mathbf Q_{\mu\mid\nu}\bigr\|_1
  \Bigr]
  \le K\sqrt{\frac{2\pi}{M}},
\]
where $\mathbf Q$ is defined in \cref{lem:mu-sampling}.
In particular, the bound holds uniformly in $\mu$.
\end{lemma}

\begin{proof}
Fix $\mu\in\P(\R)$ and $K\in\N$.
Let $R_\nu$ and $R_{\hat\nu_M}$ denote the CDFs of $\nu$ and $\hat\nu_M$.
By \eqref{eq:Q^K_expression},
\[
  Q^{(K)}_{\mu\mid\nu}(n)
  = \E_\mu\bigl[b_{n, K}(R_\nu)\bigr],
  \qquad
  Q^{(K)}_{\mu\mid\hat\nu_M}(n)
  = \E_\mu\bigl[b_{n, K}(R_{\hat\nu_M})\bigr].
\]

By a simple coupling argument (alternatively combine \citep[Eq.~(3)]{K2026} with \citep[Eq.~(4)]{K2025}),
\[
  \sum_{n=0}^K |b_{n, K}(t)-b_{n, K}(s)|
  = 2\,d_{\TV}\bigl(\mathrm{Bin}(K,s),\mathrm{Bin}(K,t)\bigr)
  \le 2 K\,|s-t|.
\]
Now fix a realization of $\hat\nu_M$.
For any $x\in\R$ we have
\[
  \sum_{n=0}^K
    \bigl|b_{n, K}(R_{\hat\nu_M}(x)) - b_{n, K}(R_\nu(x))\bigr|
  \le 2K\,\bigl|R_{\hat\nu_M}(x) - R_\nu(x)\bigr|.
\]
Hence,
\begin{align*}
  \bigl\|\mathbf Q_{\mu\mid\hat\nu_M}
           - \mathbf Q_{\mu\mid\nu}\bigr\|_1
  &= \sum_{n=0}^K
       \Bigl|
         \E_\mu\bigl[b_{n, K}(R_{\hat\nu_M})\bigr]
         - \E_\mu\bigl[b_{n, K}(R_\nu)\bigr]
       \Bigr|
 \le
  \E_\mu\left[
    \sum_{n=0}^K
      \bigl|b_{n, K}(R_{\hat\nu_M})
            - b_{n, K}(R_\nu)\bigr|
  \right] \\
  & \le 2K\,\E_\mu\bigl[
         \bigl|R_{\hat\nu_M} - R_\nu\bigr|
       \bigr]
  \le 2K\,
       \sup_{x\in\R}
         \bigl|R_{\hat\nu_M}(x) - R_\nu(x)\bigr|.
\end{align*}
Taking expectations over $\hat\nu_M$ and applying the Dvoretzky–Kiefer–Wolfowitz (DKW) inequality \citep{massart1990tight},
\[
  \PP\Bigl(
    \sup_{x\in\R} |R_{\hat\nu_M}(x) - R_\nu(x)| > t
  \Bigr)
  \le 2 e^{-2Mt^2},\qquad t>0,
\]
we obtain
\begin{align*}
  \E\Bigl[
    \sup_{x\in\R} |R_{\hat\nu_M}(x) - R_\nu(x)|
  \Bigr]
  &= \int_0^\infty
       \PP\Bigl(
         \sup_{x\in\R} |R_{\hat\nu_M}(x) - R_\nu(x)| > t
       \Bigr) \d t
  \le \int_0^\infty 2e^{-2Mt^2} \d t
   = \sqrt{\frac{\pi}{2M}}.
\end{align*}
Combining the two displays yields
\[
  \E\Bigl[
    \bigl\|\mathbf Q_{\mu\mid\hat\nu_M}
            - \mathbf Q_{\mu\mid\nu}\bigr\|_1
  \Bigr]
  \le 2K\,\sqrt{\frac{\pi}{2M}}.
\]
The bound is uniform in $\mu$ because $\mu$ does not appear on the
right-hand side.
\end{proof}

\subsection{Proof of the univariate finite sample complexity bound \cref{thm:finite_sample_univariate}} \label{subsec:proof_univ_finite_sample_complexity_bound}
\FiniteSampleUnivariate*

\begin{proof}[Proof of \Cref{thm:finite_sample_univariate}]
For $K\in\N$ we have
\begin{equation} \label{eq:lipschitz-in-Q}
\begin{aligned}
  \bigl|D^{(K)}_{f, \hat{\nu}_M}(\hat\mu_N)
        -D^{(K)}_{f, \nu}(\mu)\bigr| 
  & \le
    \frac{1}{K+1}\sum_{n=0}^K
      \left|
        f\Bigl((K+1)Q^{(K)}_{\hat\mu_N\mid\hat\nu_M}(n)\Bigr)
        - f\Bigl((K+1)Q^{(K)}_{\mu\mid\nu}(n)\Bigr)
      \right| \\
      & \le
  \frac{1}{K+1}\sum_{n=0}^K
    L_f (K+1)\,
    \bigl|
      Q^{(K)}_{\hat\mu_N\mid\hat\nu_M}(n)
      - Q^{(K)}_{\mu\mid\nu}(n)
    \bigr| 
    =
  L_f \sum_{n=0}^K
    \bigl|
      Q^{(K)}_{\hat\mu_N\mid\hat\nu_M}(n)
      - Q^{(K)}_{\mu\mid\nu}(n)
    \bigr|,
\end{aligned}
\end{equation}
since for each $n \in [K]$, the quantities
$(K+1)Q^{(K)}_{\hat\mu_N\mid\hat\nu_M}(n)$ and
$(K+1)Q^{(K)}_{\mu\mid\nu}(n)$ lie in the interval $[0,K+1]$, and $f$ is $L_f$–Lipschitz on $[0,K+1]$.
Taking expectations, we obtain
\begin{equation}
  \label{eq:FS_main_LipQ_long2}
  \E\Bigl[
    \bigl|D^{(K)}_{f}(\hat\mu_N \mid \hat\nu_M)
          -D^{(K)}_{f}(\mu \mid \nu)\bigr|
  \Bigr]
  \le
  L_f\,\E\Bigl[
    \sum_{n=0}^K
      \bigl|
        Q^{(K)}_{\hat\mu_N\mid\hat\nu_M}(n)
        - Q^{(K)}_{\mu\mid\nu}(n)
      \bigr|
  \Bigr].
\end{equation}
We now decompose the rank-PMF error into two contributions:
the error due to sampling $\mu$ and the error due to sampling $\nu$.
By the triangle inequality,
\begin{equation} \label{eq:FS_triangle_long2}
  \sum_{n=0}^K
    \bigl|
      Q^{(K)}_{\hat\mu_N\mid\hat\nu_M}(n)
      - Q^{(K)}_{\mu\mid\nu}(n)
    \bigr|
  \le
  \sum_{n=0}^K
    \bigl|
      Q^{(K)}_{\hat\mu_N\mid\hat\nu_M}(n)
      - Q^{(K)}_{\hat\mu_N\mid\nu}(n)
    \bigr|
  +
  \sum_{n=0}^K
    \bigl|
      Q^{(K)}_{\hat\mu_N\mid\nu}(n)
      - Q^{(K)}_{\mu\mid\nu}(n)
    \bigr|. \nonumber
\end{equation}
Taking expectations and applying
\Cref{lem:nu-sampling} (which holds uniformly with respect to the argument of $D_{f, \nu}^{(K)}$, so it
can be used with the random $\hat\mu_N$) and \Cref{lem:mu-sampling}, we obtain
\begin{equation} \label{eq:FS_bound_Q_simple_long2}
\begin{aligned}
  \E\Bigl[
    \sum_{n=0}^K
      \bigl|
        Q^{(K)}_{\hat\mu_N\mid\hat\nu_M}(n)
        - Q^{(K)}_{\mu\mid\nu}(n)
      \bigr|
  \Bigr] 
  & \le
  \E\Bigl[
    \sum_{n=0}^K
      \bigl|
        Q^{(K)}_{\hat\mu_N\mid\hat\nu_M}(n)
        - Q^{(K)}_{\hat\mu_N\mid\nu}(n)
      \bigr|
  \Bigr]
  +
  \E\Bigl[
    \sum_{n=0}^K
      \bigl|
        Q^{(K)}_{\hat\mu_N\mid\nu}(n)
        - Q^{(K)}_{\mu\mid\nu}(n)
      \bigr|
  \Bigr] \\
  &\qquad\le
  2K\,\sqrt{\frac{\pi}{2M}}
  + \frac{K+1}{2\sqrt N}
  \le
  (K+1)\sqrt{2\pi}
  \left(
    \frac{1}{\sqrt N} + \frac{1}{\sqrt M}
  \right),
\end{aligned}
\end{equation}
where we used $\frac{1}{2} \le \sqrt{2\pi}$ to simplify the constants.
\end{proof}

\subsection{Proof of the concentration bound} \label{subsec:proof_concentration_bound}

\begin{proof}
Let $K \in \N$ and define the functional
\[
  F(X_1,\dots,X_N,Y_1,\dots,Y_M)
  \coloneqq D^{(K)}_{f, \hat{\nu}_M}(\hat\mu_N),
\]
where the empirical measures $\hat{\mu}_N$ and $\hat{\nu}_M$ are given by \eqref{eq:empirical_measure}.
We will apply \textit{McDiarmid's bounded differences inequality} (see \citep[Sec~6]{boucheron2013concentration}) to $F$.
We first quantify how much $F$ can change when we replace one observation $X_i$ in the sample from $\mu$, while keeping all other data points fixed. Consider two datasets that differ only in the $i$-th sample from $\mu$:
\[
  (X_1,\dots,X_i,\dots,X_N,Y_1,\dots,Y_M),
  \quad
  (X_1,\dots,X'_i,\dots,X_N,Y_1,\dots,Y_M),
\]
and fix $Y_1,\dots,Y_M$.
Thus $\hat\nu_M$ is the same in both cases, while the empirical measure of
$\mu$ changes from $\hat\mu_N$ to
\[
  \hat\mu_N'
  = \hat\mu_N - \frac{1}{N}\delta_{X_i} + \frac{1}{N}\delta_{X'_i}.
\]
Fix $\hat\nu_M$ and $K\in\N$.
For each $n\in[K]$ define $\phi_n$ by \eqref{eq:phi_n}.
Again,
\[
  Q^{(K)}_{\hat\mu_N\mid\hat\nu_M}(n)
  = \frac{1}{N}\sum_{k=1}^N \phi_n(X_k),
  \qquad
  Q^{(K)}_{\hat\mu_N'\mid\hat\nu_M}(n)
  = \frac{1}{N}\sum_{k=1}^N \phi_n(X'_k),
\]
where $X'_k = X_k$ for $k\neq i$ and $X'_i = X'_i$.
Hence
\[
  Q^{(K)}_{\hat\mu_N'\mid\hat\nu_M}(n)
  - Q^{(K)}_{\hat\mu_N\mid\hat\nu_M}(n)
  = \frac{1}{N} \bigl(\phi_n(X_i') - \phi_n(X_i)\bigr).
\]
Since $0\le \phi_n \le 1$, we obtain
\[
  \bigl\|
    Q^{(K)}_{\hat\mu_N'\mid\hat\nu_M}
    - Q^{(K)}_{\hat\mu_N\mid\hat\nu_M}
  \bigr\|_1
  \le \frac{K+1}{N}.
\]

Applying \Cref{eq:lipschitz-in-Q} with $(\hat\mu_N' \mid \hat\nu_M)$ instead of $(\mu \mid \nu)$,
we obtain
\begin{align*}
  \bigl|
    F(X_1,&\dots,X_i',\dots,X_N,Y_1,\dots,Y_M)
    - F(X_1,\dots,X_i,\dots,X_N,Y_1,\dots,Y_M)
  \bigr|
   \\&=\bigl|
       D^{(K)}_{f, \hat{\nu}_M}(\hat\mu_N')
       - D^{(K)}_{f, \hat{\nu}_M}(\hat\mu_N)
     \bigr| \le
     L_f\,
     \bigl\|
       Q^{(K)}_{\hat\mu_N'\mid\hat\nu_M}
       - Q^{(K)}_{\hat\mu_N\mid\hat\nu_M}
     \bigr\|_1 
     \le L_f\,\frac{K+1}{N}.
\end{align*}

Now, consider two datasets that differ only in $Y_j$:
\[
  (X_1,\dots,X_N,Y_1,\dots,Y_j,\dots,Y_M),
  \quad
  (X_1,\dots,X_N,Y_1,\dots,Y'_j,\dots,Y_M),
\]
so that $\hat\mu_N$ is fixed, while the empirical measure of $\nu$ changes from
$\hat\nu_M$ to $\hat\nu_M'
  = \hat\nu_M - \frac{1}{M}\delta_{Y_j} + \frac{1}{M}\delta_{Y'_j}$.
Changing a single atom in an empirical measure of size $M$ changes the CDF by at
most $1/M$, that is,
\[
  \sup_{x\in\R} |R_{\hat\nu_M'}(x) - R_{\hat\nu_M}(x)|
  \le \frac{1}{M}.
\]
As in the proof of \Cref{lem:nu-sampling},
we obtain 
\begin{align*}
  \bigl\|
    Q^{(K)}_{\hat\mu_N\mid\hat\nu_M'}
    - Q^{(K)}_{\hat\mu_N\mid\hat\nu_M}
  \bigr\|_1
  &= \sum_{n=0}^K
       \Bigl|
         \E_{\hat\mu_N}\bigl[b_{n, K}(R_{\hat\nu_M'})\bigr]
         - \E_{\hat\mu_N}\bigl[b_{n, K}(R_{\hat\nu_M})\bigr]
       \Bigr|
    \le \E_{\hat\mu_N}\left[
         \sum_{n=0}^K
           \bigl|
             b_{n, K}(R_{\hat\nu_M'})
             - b_{n, K}(R_{\hat\nu_M})
           \bigr|
       \right] \\
  &\le 2K\,\E_{\hat\mu_N}\bigl[
         |R_{\hat\nu_M'} - R_{\hat\nu_M}|
       \bigr]
    \le 2K\,
       \sup_{x\in\R}
         |R_{\hat\nu_M'}(x) - R_{\hat\nu_M}(x)|
   \le  \frac{2K}{M}.
\end{align*}
Applying \Cref{eq:lipschitz-in-Q} with $(\hat\mu_N,\hat\nu_M')$ instead of $(\mu \mid \nu)$, we obtain
\begin{align*}
  &\bigl|
    F(X_1,\dots,X_N,Y_1,\dots,Y_j',\dots,Y_M)
    - F(X_1,\dots,X_N,Y_1,\dots,Y_j,\dots,Y_M)
  \bigr| \\
  &\qquad=
  \bigl|
    D^{(K)}_{f, \hat{\nu}_M'}(\hat\mu_N)
    - D^{(K)}_{f, \hat{\nu}_M}(\hat\mu_N)
  \bigr|
  \le
  L_f\,
  \bigl\|
    Q^{(K)}_{\hat\mu_N\mid\hat\nu_M'}
    - Q^{(K)}_{\hat\mu_N\mid\hat\nu_M}
  \bigr\|_1 
  \le L_f\,\frac{2K}{M}.
\end{align*}
We have shown that $F$ satisfies the bounded differences condition with
\[
  c_i = L_f\frac{K+1}{N} \quad (i=1,\dots,N),
  \qquad
  c_{N+j} = L_f\frac{2K}{M} \quad (j=1,\dots,M).
\]
Hence, using $K^2\le (K+1)^2$,
\begin{align*}
  \sum_{i=1}^{N+M} c_i^2
  &= \sum_{i=1}^N \Bigl(L_f\frac{K+1}{N}\Bigr)^2
     + \sum_{j=1}^M \Bigl(L_f\frac{2K}{M}\Bigr)^2
   \le 
   4L_f^2 (K+1)^2\Bigl(\frac{1}{N} + \frac{1}{M}\Bigr),
\end{align*}

By McDiarmid's inequality (Theorem~6.2 in \citet{boucheron2013concentration}),
for any $t>0$,
\[
  \PP\Bigl(
    |F - \E[F]| \ge t
  \Bigr)
  \le
  2\exp\!\left(
    -\frac{2t^2}{\sum_{i=1}^{N+M} c_i^2}
  \right)
  \le
  2\exp\!\left(
    -\frac{t^2}{2L_f^2 (K+1)^2\left(\frac{1}{N}+\frac{1}{M}\right)}
  \right).
\]
Let $\delta>0$ and choose
\[
  t
  = L_f (K+1)
    \sqrt{
      2\log(2/\delta)
      \left(\frac{1}{N}+\frac{1}{M}\right)
    }.
\]
Then
\[
  \PP\Bigl(
    |F - \E[F]| \ge t
  \Bigr)
  \le \delta.
\]
Since $F =
D^{(K)}_{f, \hat{\nu}_M}(\hat\mu_N)$, this is exactly the desired bound:
with probability at least $1-\delta$,
\[
  \bigl|
    D^{(K)}_{f, \hat{\nu}_M}(\hat{\mu}_N)
    - \E\bigl[D^{(K)}_{f, \hat{\nu}_M}(\hat{\mu}_N)\bigr]
  \bigr|
  \le
  L_f (K+1)
  \sqrt{
    2\log(2/\delta)
    \left(\frac{1}{N}+\frac{1}{M}\right)
  }.\qedhere
\]
\end{proof}

\subsection{Proof of the asymptotic normality \cref{thm:asymptotic_normality_sliced}} \label{subsec:proof_asymptotic_normality}
\AsymptoticNormalitySliced*

\begin{proof}
\begin{enumerate}
\item
For each direction $s \in \S^{d-1}$ and $n \in [K]$, define
$\phi_{n,s} : \R^d \to [0,1]$ as the probability that a sample $x$ has rank $n$
among $K$ draws $\tilde Y_1,\dots,\tilde Y_K \stackrel{\text{i.i.d.}}{\sim} \nu$ when both are projected along $s$, i.e.,
\[
  \phi_{n,s}(x)
  \coloneqq
  \PP\Bigl(
    \#\{j \in \{1, \ldots, K \}: s^\top \tilde Y_j \le s^\top x\} = n
  \Bigr).
\]
As in the univariate case (compare \eqref{eq:Q^K_expression}) one checks that for any $\mu \in \P(\R^d)$
and $s\in\S^{d-1}$,
\begin{equation} \label{eq:identities}
  Q^{(K)}_{\mu_s\mid\nu_s}(n)
  = \E_{\mu}[\phi_{n,s}]
  \qquad
  \text{and} \qquad Q^{(K)}_{(\hat\mu_{N})_{s}\mid\nu_s}(n)
  = \frac{1}{N}\sum_{i=1}^N \phi_{n,s}(X_i),
\end{equation}
where the empirical version $\hat\mu_N$ is defined in \eqref{eq:empirical_measure}.
Define the Hilbert space $\H
  \coloneqq L^2\bigl(\S^{d-1}\times[K],\sigma\otimes U_K\bigr)$ with inner product
\[
  \langle h,g\rangle_{\H}
  \coloneqq
  \int_{\S^{d-1}}
  \frac{1}{K+1}\sum_{n=0}^K h(s,n)\,g(s,n) \d\sigma(s), \qquad g, h \in \H.
\]
For $x\in\R^d$, define the random element $\Phi(x)\in\H$ by $\Phi(x)(s,n) \coloneqq \phi_{n,s}(x)$.
Then $\|\Phi(x)\|_{\H} \le 1$ because $0\le\phi_{n,s}(x)\le 1$.
Let
\[
  T \coloneqq \E_\mu[\Phi] \in \H,
  \qquad
  T_N \coloneqq \frac{1}{N}\sum_{i=1}^N \Phi(X_i) \in \H.
\]
By the identities \eqref{eq:identities}, we have
\[
  T(s,n)
  = \E_\mu[\phi_{n,s}]
  = Q^{(K)}_{\mu_s\mid\nu_s}(n),
  \qquad
  T_N(s,n)
  = Q^{(K)}_{(\hat\mu_{N})_{s}\mid\nu_s}(n).
\]
The random elements $\left(\Phi(X_k)\right)_{k \in \N}$ are i.i.d.\ in $\H$
with $\E_{\mu}[\|\Phi\|_{\H}^2] \le 1$.
Hence, the standard central limit theorem in separable Hilbert spaces
applies, and we obtain
\[
  \sqrt{N}\bigl(T_N - T\bigr)
  \;\xrightarrow{d}\;
  \mathcal G
  \quad\text{in }\H,
\]
where $\mathcal G$ is a mean-zero Gaussian element in $\H$ with
covariance operator determined by $\Phi$.

\item 
We now write the rank-statistic $f$-divergence and its approximation by samples as $\Psi(T)$ and $\Psi(T_N)$ for some $\Psi$ and use the delta-method.

Indeed, for
\[
  \Psi \colon \H\to\R, \qquad
  h \mapsto
  \int_{\S^{d-1}}
    \frac{1}{K+1}\sum_{n=0}^K
      f\bigl((K+1)h(s,n)\bigr) \d\sigma(s)
\]
we have
\[
  D^{(K)}_{f,\nu}(\mu)
  = \Psi(T)
  = \int_{\S^{d-1}}
      \frac{1}{K+1}\sum_{n=0}^K
        f\Bigl((K+1)\,Q^{(K)}_{\mu_s\mid\nu_s}(n)\Bigr)
     \d\sigma(s)
     \qquad \text{and} \qquad
     D_{f, \nu}^{(K)}(\hat\mu_N) = \Psi(T_N).
\]
We now show that $\Psi$ is differentiable.
Since 
$f \in \C^1([0,K+1])$ and $f'$ is Lipschitz,
$\Psi$ is Fréchet differentiable at $T$ with Lipschitz continuous derivative
\[
  D\Psi(T) \colon \H\to\R, \qquad
  h \mapsto \int_{\S^{d-1}}
      \sum_{n=0}^K
        f'\bigl((K+1)T(s,n)\bigr)\,h(s,n) \d\sigma(s).
\]
In particular, at $T(s,n) = Q^{(K)}_{\mu_s\mid\nu_s}(n)$ this becomes
\[
  D\Psi(T)[h]
  = \int_{\S^{d-1}}
      \sum_{n=0}^K
        f'\bigl((K+1)Q^{(K)}_{\mu_s\mid\nu_s}(n)\bigr)\,
        h(s,n) \d\sigma(s),
\]
which defines a bounded linear functional on $\H$.

By the Hilbert-space delta method \citep[Thm.~3.9.4]{vandervaart_wellner_1996}, the combination of the CLT for $T_N$ and the Fréchet
differentiability of $\Psi$ at $T$ implies
\[
  \sqrt{N}\bigl(\Psi(T_N) - \Psi(T)\bigr)
  \xrightarrow{d}
  \mathcal N\bigl(0,\tau_{K}^2\bigr),
\]
with asymptotic variance
\[
  \tau_{K}^2
  = \Var\bigl(D\Psi(T)[\Phi(X_1)-T]\bigr)
  \ge 0.
\]
We have $\tau_K^2 > 0$ if $\int_{\S^{d-1}} \sum_{n=0}^K f'\bigl((K+1)Q^{(K)}_{\mu_s\mid\nu_s}(n)\bigr) b_{n, K}(R_{\nu_s}(s^{\top} x)) \d\sigma(s)$ is not constant on $\supp(\mu)$.

Recalling that $\Psi(T) = D_{f, \nu}^{(K)}(\mu)$ and $\Psi(T_N) =  D_{f, \nu}^{(K)}(\hat\mu_N)$, we obtain
the desired conclusion.\qedhere
\end{enumerate}
\end{proof}

\subsection{Proof of \cref{thm:asymptotic-choice-K}} \label{subsec:proof_asymptotic_choice_K}

The following theorem shows that when choosing $K$ in dependence of $N$, it should grow faster than $\sqrt N$ to remove the Bernstein approximation bias, but slower than $N$ so that the nonlinear plug-in remainder remains negligible
\AsymptoticChoiceK*

\begin{proof}
We decompose
\[
  \widetilde D_N^{(K_N)} - D_{f,\nu}(\mu)
  =
  \Bigl(\widetilde D_N^{(K_N)} - D_{f,\nu}^{(K_N)}(\mu)\Bigr)
  +
  \Bigl(D_{f,\nu}^{(K_N)}(\mu)-D_{f,\nu}(\mu)\Bigr).
\]
By the $C^2$ case of \cref{thm:convergence_K_to_Infty},
\[
  \left|D_{f,\nu}^{(K)}(\mu)-D_{f,\nu}(\mu)\right|
  = O(K^{-1}),
\]
so the second term is $o(N^{-1/2})$ whenever $\sqrt N/K_N\to0$.

It remains to identify the sampling fluctuation. For $n\in[K]$, set, as before,
\[
  q_n^{(K)} \coloneqq Q_{\mu\mid\nu}^{(K)}(n),
  \qquad
  c_n^{(K)} \coloneqq (K+1)q_n^{(K)},
  \qquad
  \hat q_{N,n}^{(K)}
  \coloneqq
  \frac1N\sum_{i=1}^N b_{n,K}(R_\nu(X_i)).
\]
The same Taylor expansion as in the proof of \cref{thm:asymptotic_normality_sliced}, now in the finite-dimensional vector $(q_0,\ldots,q_K)$, gives
\[
  D^{(K)}_{f,\nu}(\hat\mu_N)-D_{f,\nu}^{(K)}(\mu)
  =
  \sum_{n=0}^K
  f'(c_n^{(K)})\bigl(\hat q_{N,n}^{(K)}-q_n^{(K)}\bigr)
  + R_{N,K}.
\]
Equivalently, the linear term is
\[
  \frac1N\sum_{i=1}^N
  \Bigl(
    h_K(R_\nu(X_i))
    - \E_\mu[h_K(R_\nu(X))]
  \Bigr),
  \qquad
  h_K(u) \coloneqq \sum_{n=0}^K f'(c_n^{(K)}) b_{n,K}(u).
\]
Since $f''$ is bounded, the Taylor remainder satisfies
\[
  |R_{N,K}|
  \le
  C(K+1)\sum_{n=0}^K
  \bigl(\hat q_{N,n}^{(K)}-q_n^{(K)}\bigr)^2.
\]
Moreover, using that $U=R_\nu(X)$ has density $r$ on $[0,1]$ and that $r$ is bounded,
\[
  \E\sum_{n=0}^K
  \bigl(\hat q_{N,n}^{(K)}-q_n^{(K)}\bigr)^2
  =
  \frac1N\sum_{n=0}^K
  \Var_\mu\!\left(b_{n,K}(R_\nu(X))\right)
  \le
  \frac{\|r\|_\infty}{N}
  \int_0^1\sum_{n=0}^K b_{n,K}(u)^2\,\d u
  = O\left(\frac1{N\sqrt K}\right).
\]
Thus
\[
  \sqrt N\,R_{N,K_N}
  = O_p\!\left(\sqrt{\frac{K_N}{N}}\right)
  = o_p(1)
\]
whenever $K_N/N\to0$.

Finally, the $C^2$ regularity of $r$ and $f$ implies
\[
  h_K \longrightarrow f'\circ r
  \qquad\text{uniformly on }[0,1].
\]
Hence the triangular-array central limit theorem yields
\[
  \frac1{\sqrt N}\sum_{i=1}^N
  \Bigl(
    h_{K_N}(R_\nu(X_i))
    - \E_\mu[h_{K_N}(R_\nu(X))]
  \Bigr)
  \xrightarrow[N\to\infty]{d}
  \mathcal N\left(
    0,
    \Var_\mu\!\left(f'(r(R_\nu(X)))\right)
  \right).\qedhere
\]
\end{proof}

\section{Experiments} \label{sec:experiments_Appendix}
\subsection{Neural vs.\ rank-statistic divergence estimation across dimensions} \label{app:Neural vs. rank-statistic divergence estimation across dimensions}

This appendix provides the full experimental details for the benchmark in
\cref{sec:experiments} comparing the proposed rank-statistic estimator to the neural
KL-divergence estimator of \citet{sreekumar2021non} on the truncated-Gaussian vs.\ uniform setup.

\vspace{0.25em}
\paragraph{Distributions and supports.}
For each dimension $d\in\{2,5,10\}$, the target distribution $\mu$ is the standard Gaussian
$\mathcal N(0,I_d)$ truncated (and renormalized) to an axis-aligned box $X$, and the reference
distribution is the uniform measure $\nu=\mathrm{Unif}(X)$. Concretely,
\[
X_{2}=[0.1,2]\times[-1,0],
\qquad
X_{5}=[0.1,2]\times[-1,0]\times[2,3]\times[-2,-1.5]\times[-1,1],
\]
and for $d=10$ we use the product support $X_{10}=X_{5}\times X_{5}$.
Sampling from $\nu$ is done by drawing each coordinate independently and uniformly over the
corresponding interval. Sampling from $\mu$ is done by accept-reject: draw from
$\mathcal N(0,I_d)$ until the sample falls in $X$.

\vspace{0.25em}
\paragraph{Sample sizes and repetitions.}
For each $n \in 10^{4}\cdot\{1,2,4,8,16,32,64,128,256,512\}$ we draw $n$ i.i.d.\ samples from $\mu$ and $n$ i.i.d.\ samples from $\nu$, and repeat the whole
procedure over $R=10$ independent random seeds. Figure~\ref{fig:kl_vs_n_d2_d5_d10} reports the
mean and $\pm1$ standard deviation bands over these runs (for both estimators).

\vspace{0.25em}
\paragraph{Analytic KL reference.}
In this benchmark, the ``ground-truth'' $\mathrm{KL}(\mu\|\nu)$ shown as a dashed line in
Figure~\ref{fig:kl_vs_n_d2_d5_d10} is computed analytically (no additional Monte Carlo layer).
Let $\tilde\mu=\mathcal N(0,I_d)$ denote the untruncated Gaussian density and let $X=\prod_{j=1}^d[a_j,b_j]$
be the truncation box. The truncated density is
$p_\mu(x)=\tilde p(x)\mathbf 1\{x\in X\}/Z$, where $Z=\int_X \tilde p(x)\,dx$ is the truncation mass, and
$p_\nu(x)=1/\mathrm{Vol}(X)$ for $x\in X$. Using
$\mathrm{KL}(\mu\|\nu)=\mathbb E_\mu[\log p_\mu(X)-\log p_\nu(X)]$ and separability of the standard Gaussian,
$Z$ factorizes as $Z=\prod_{j=1}^d(\Phi(b_j)-\Phi(a_j))$, where $\Phi$ denotes the CDF of the standard normal
distribution (and $\phi$ its density).
The remaining expectation reduces to the sum of one-dimensional truncated moments $\mathbb E_\mu[X_j^2]$
(available in closed form via $\phi$ and $\Phi$). This yields an exact value for each $d$ and box $X$; the
implementation follows these standard identities.

\vspace{0.25em}
\paragraph{Neural baseline (full protocol).}
We compare against the neural KL-divergence estimator of \citet{sreekumar2021non} and follow their protocol
exactly. For each sample size $n$, the network width is set to $k=\lceil n^{1/5}\rceil$ and the model is
trained for $200$ epochs with Adam, using learning rate $10^{-2}$ and a single decay to $10^{-3}$ after
$100$ epochs. Minibatches have size $10^{-3}n$. Results are averaged over the same $R=10$ seeds used for the
rank-based estimator. (All remaining hyperparameters and architectural details are as in
\citet{sreekumar2021non}.)

\vspace{0.25em}
\paragraph{Rank-statistic estimator settings.}

 In contrast, the rank-statistic estimator requires no iterative optimization: once the samples are fixed, the estimate is fully determined by the rank resolution $K$. In this benchmark, we exploit the product structure of the distributions and use the axis-corrected estimator described in the main text: we compute the one-dimensional degree-$K$ rank-statistic KL terms along each coordinate axis and sum them over coordinates. Thus, no random projection parameter $L$ is used for this experiment. Unless stated otherwise, all randomness in the rank estimator comes solely from the sampled data.

\subsection{Univariate empirical convergence and the influence of resolution $K$}\label{app:One-Dimensional Approximation Accuracy}

This appendix provides additional implementation details for the one-dimensional benchmarks reported in
\cref{subsec:onedim-exp}. For each configuration we draw $n_\mu$ i.i.d.\
samples from $\mu$ and $n_\nu$ i.i.d.\ samples from $\nu$, construct the Bernstein rank histogram
$Q^{(K)}_{\mu\mid\nu}$, and compute the corresponding discrete $f$-divergence
$D^{(K)}_{f,\nu}(\mu)$. We repeat each setting for $R=10$ seeds and report mean$\pm$std.

\vspace{0.25em}
\paragraph{Distributions and supports.}
We study four representative mismatch families that collectively capture shifts in location and scale, departures from unimodality, and tail mismatch:
\begin{itemize}
  \item Location shift (Gaussian mean).
  We set $\mu=\mathcal N(0,1)$ and $\nu=\mathcal N(\Delta,1)$ with $\Delta\in\{0,0.5,1,2\}$, and report \emph{JS}, \emph{KL}, and \emph{TV}.

  \item Scale change (Gaussian variance).
  We take $\mu=\mathcal N(0,1)$ and $\nu=\mathcal N(0,\sigma)$ for $\sigma\in\{1,1.2,1.5,2\}$, and report \emph{KL}, squared Hellinger (\emph{Hell$^2$}), and \emph{TV}.

  \item Multimodality (mixture vs.\ unimodal).
  To probe sensitivity to multiple modes, we compare the symmetric mixture
  $\mu=\tfrac12\mathcal N(-\Delta,1)+\tfrac12\mathcal N(+\Delta,1)$ against $\nu=\mathcal N(0,1)$ over the same $\Delta$ values, and report \emph{JS}, \emph{KL}, and \emph{TV}.

  \item Tail mismatch (heavy-tailed vs.\ Gaussian).
  Finally, we compare $\mu=\mathrm{Laplace}(0,1)$ to $\nu=\mathcal N(0,1)$ and report \emph{JS}, \emph{KL}, and \emph{TV}.
\end{itemize}
Unless stated otherwise we use $n_\mu=n_\nu=10{,}000$, and we evaluate $K\in\{32,64,128,256,512\}$.

\vspace{0.25em}
\paragraph{Reference.}
Whenever a closed form is available, we use it as ground truth (in particular, Gaussian-Gaussian KL, squared Hellinger, and TV for Gaussian mean/scale changes).
For Jensen-Shannon (JS), we compute a high-accuracy reference from
\[
\mathrm{JS}(\mu,\nu)
=\tfrac12\,\E_{X\sim \mu}\!\Big[\log\frac{2\,p_\mu(X)}{p_\mu(X)+p_\nu(X)}\Big]
+\tfrac12\,\E_{Y\sim \nu}\!\Big[\log\frac{2\,p_\nu(Y)}{p_\mu(Y)+p_\nu(Y)}\Big],
\]
and evaluate the resulting one-dimensional expectations by numerical quadrature with a tight tolerance.
All evaluations are carried out in the log domain to avoid numerical issues in the tails.
In the mixture-vs-Gaussian setting, the mixture density $p_\mu$ is computed exactly as the average of its two Gaussian components inside the same procedure.

For settings where our reference is not available in closed form in our implementation, namely, mixture-vs-Gaussian KL and TV, as well as Laplace$(0,1)$ vs.\ $\mathcal N(0,1)$ for JS/KL/TV, we compute a single high-sample Monte Carlo reference once and reuse it across all $R$ runs.
Specifically, we draw $n_{\mathrm{ref}}=10^{7}$ i.i.d.\ samples from each distribution and plug them into the corresponding expectation identity (JS as above, and analogously for KL/TV).

For each configuration and each $K$, we report the estimate $\widehat D^{(K)}_{f,\nu}(\mu)$ (mean$\pm$std over $R$ seeds)
and the ratio $\widehat D^{(K)}_{f,\nu}(\mu)/D_{f,\nu}(\mu)$ to summarize the finite-$K$ approximation gap.

\vspace{0.25em}
\paragraph{Results}

Taken together, Table~\ref{tab:1d_results_refined} and Figures~\ref{fig:all_convergence}-\ref{fig:convergence_shared_legend} show that increasing the rank resolution $K$ systematically closes the finite-$K$ gap: the estimate/reference ratio moves toward $1$ across all mismatch families and all reported divergences (JS, KL, Hell$^2$, and the added TV), with near-unbiased behavior already for moderate $K$ in the smooth Gaussian cases (mean/scale shifts), while multimodality and tail mismatch require larger $K$ to reduce bias and typically exhibit larger finite-sample variability. In particular, TV is already accurate at moderate $K$ in the Gaussian settings (after the $\tfrac12$ conversion from the $\ell_1$ form returned by our implementation), whereas KL under heavy tails remains substantially underestimated even at the largest $K$ considered, reflecting the increased difficulty of capturing tail contributions with finite rank resolution. At fixed $K=256$, the estimator remains stable across a range of shift magnitudes and scales, as summarized in Figure~\ref{fig:robustness_combined}. The $K$-sweeps in Figure~\ref{fig:all_convergence} and, especially, the log-log plot in Figure~\ref{fig:convergence_shared_legend} further clarify the convergence hierarchy by visualizing the absolute ratio error $|1-\text{Ratio}|$ versus $K$: the Gaussian scale-change case decays the fastest (tracking a near-optimal $\mathcal{O}(K^{-1})$ slope), the Gaussian mean-shift case converges more slowly (consistent with reduced tail regularity), and the Laplace-vs-Gaussian heavy-tail mismatch improves the slowest, remaining closest to the shallowest guide slope. This separation aligns with the qualitative ordering suggested by the regularity-based rate discussion in the main text and motivates the larger-$K$ settings used in the tail-mismatch experiments.

\sisetup{
  separate-uncertainty = true,
  uncertainty-separator = {\,\pm\,},
  table-format = 1.3(3), 
  detect-all
}

\begin{table}[!tbhp]
\centering
\footnotesize
\renewcommand{\arraystretch}{1.1}
\setlength{\tabcolsep}{4pt}

\begin{tabular}{@{} l l l S S S S S @{}}
\toprule
& & & \multicolumn{5}{c}{Ratio $D_{f,\nu}^{(K)}(\mu)$ for $K=$} \\
\cmidrule(lr){4-8}
Family & Scen. & Param. & {32} & {64} & {128} & {256} & {512} \\
\midrule

\multirow{9}{*}{Mean shift}
  & JS & $\Delta=0.5$ & 0.933(40) & 0.968(41) & 0.989(42) & 1.003(42) & 1.013(42) \\
  & JS & $\Delta=1.0$ & 0.928(33) & 0.961(34) & 0.981(35) & 0.992(35) & 0.999(35) \\
  & JS & $\Delta=2.0$ & 0.930(08) & 0.962(08) & 0.981(09) & 0.991(09) & 0.997(09) \\
\addlinespace[4pt]
  & KL & $\Delta=0.5$ & 0.946(60) & 0.987(63) & 1.013(65) & 1.030(66) & 1.044(68) \\
  & KL & $\Delta=1.0$ & 0.880(24) & 0.924(25) & 0.952(26) & 0.969(27) & 0.980(27) \\
  & KL & $\Delta=2.0$ & 0.775(10) & 0.844(12) & 0.895(13) & 0.933(15) & 0.959(16) \\
\addlinespace[4pt]
  & TV & $\Delta=0.5$ & 0.979(26) & 0.991(27) & 0.998(27) & 1.001(27) & 1.003(27) \\
  & TV & $\Delta=1.0$ & 0.974(11) & 0.985(11) & 0.991(11) & 0.994(11) & 0.996(11) \\
  & TV & $\Delta=2.0$ & 0.977(02) & 0.989(02) & 0.996(02) & 0.999(03) & 1.001(03) \\

\midrule

\multirow{9}{*}{Scale change}
  & KL & $\sigma=1.2$ & 0.743(63) & 0.841(70) & 0.908(72) & 0.954(72) & 0.991(72) \\
  & KL & $\sigma=1.5$ & 0.779(27) & 0.872(29) & 0.927(30) & 0.958(31) & 0.977(31) \\
  & KL & $\sigma=2.0$ & 0.803(18) & 0.898(20) & 0.953(21) & 0.982(22) & 0.998(22) \\
\addlinespace[4pt]
  & Hell$^2$ & $\sigma=1.2$ & 0.741(77) & 0.853(89) & 0.931(98) & 0.986(106) & 1.029(111) \\
  & Hell$^2$ & $\sigma=1.5$ & 0.735(35) & 0.842(39) & 0.908(41) & 0.948(42) & 0.973(42) \\
  & Hell$^2$ & $\sigma=2.0$ & 0.744(14) & 0.858(14) & 0.926(14) & 0.965(13) & 0.987(12) \\
\addlinespace[4pt]
  & TV & $\sigma=1.2$ & 0.934(33) & 0.970(36) & 0.990(39) & 1.001(40) & 1.008(41) \\
  & TV & $\sigma=1.5$ & 0.907(14) & 0.948(15) & 0.970(17) & 0.982(18) & 0.989(18) \\
  & TV & $\sigma=2.0$ & 0.898(09) & 0.947(10) & 0.974(10) & 0.988(10) & 0.995(10) \\

\midrule

\multirow{9}{*}{Multimodal}
  & JS & $\Delta=0.5$ & 0.746(157) & 0.849(176) & 0.926(189) & 0.994(196) & 1.068(199) \\
  & JS & $\Delta=1.0$ & 0.769(38) & 0.849(40) & 0.898(41) & 0.929(41) & 0.948(42) \\
  & JS & $\Delta=2.0$ & 0.846(19) & 0.912(20) & 0.950(21) & 0.972(21) & 0.985(21) \\
\addlinespace[4pt]
  & KL & $\Delta=0.5$ & 0.742(125) & 0.853(147) & 0.936(165) & 1.000(178) & 1.054(187) \\
  & KL & $\Delta=1.0$ & 0.766(32) & 0.864(36) & 0.930(39) & 0.971(41) & 0.998(41) \\
  & KL & $\Delta=2.0$ & 0.669(09) & 0.765(10) & 0.837(12) & 0.889(13) & 0.926(13) \\
\addlinespace[4pt]
  & TV & $\Delta=0.5$ & 0.935(66) & 0.969(70) & 0.990(74) & 1.005(76) & 1.017(78) \\
  & TV & $\Delta=1.0$ & 0.947(23) & 0.975(23) & 0.990(23) & 0.998(23) & 1.003(24) \\
  & TV & $\Delta=2.0$ & 0.942(05) & 0.969(05) & 0.982(05) & 0.989(06) & 0.993(06) \\

\midrule

\multirow{3}{*}{Heavy tails}
  & JS & -- & 0.488(28) & 0.651(36) & 0.778(41) & 0.869(44) & 0.933(46) \\
  & KL & -- & 0.210(12) & 0.299(17) & 0.383(22) & 0.458(25) & 0.524(28) \\
  & TV & -- & 0.824(14) & 0.907(17) & 0.955(20) & 0.981(23) & 0.996(25) \\
\bottomrule
\end{tabular}
\vspace{0.5mm}
\caption{\small 1D divergence estimation benchmarks (10 runs). We report the ratio estimate/reference (mean $\pm$ std) for various $K$ values.}
\label{tab:1d_results_refined}
\end{table}

\begin{figure*}[!tbhp]
\centering

\definecolor{niceblue}{RGB}{31, 119, 180}
\definecolor{nicegray}{RGB}{80, 80, 80}

\pgfplotstableread[col sep=space]{
K    mean      std
32   0.010691  0.000619
64   0.014269  0.000787
128  0.017049  0.000896
256  0.019057  0.000967
512  0.020464  0.001019
}\DataLaplace

\pgfplotstableread[col sep=space]{
K    mean      std
32   1.549813  0.019456
64   1.687851  0.023195
128  1.790635  0.026537
256  1.865320  0.029412
512  1.918886  0.031819
}\DataMean

\pgfplotstableread[col sep=space]{
K    mean      std
32   0.255448  0.005768
64   0.285834  0.006438
128  0.303039  0.006787
256  0.312404  0.006937
512  0.317541  0.006995
}\DataScale

\pgfplotsset{
    myaxisstyle/.style={
        width=\linewidth, 
        height=5.0cm,
        axis line style={-}, 
        axis x line=bottom, 
        axis y line=left,
        grid=major, 
        grid style={dashed, gray!30},
        label style={font=\footnotesize}, 
        tick label style={font=\scriptsize},
        xmin=0, xmax=550, 
        xtick={32, 128, 256, 512},
        xlabel={$K$ (Basis Functions)},
    }
}

\begin{subfigure}[t]{0.32\linewidth}
\centering
\begin{tikzpicture}
\begin{axis}[
  myaxisstyle,
  ymin=0.0, ymax=0.03,
  ylabel={Estimated Divergence},
  legend columns=-1,
  legend entries={True Reference, Rank-$f$ Estimate},
  legend to name=CommonLegend,
  legend style={draw=none, fill=none, font=\footnotesize, /tikz/every even column/.append style={column sep=0.5cm}}
]
\addplot[nicegray, thick, dashed, domain=0:550] {0.021869};
\addplot[name path=U1, draw=none, forget plot] table[x=K, y expr=\thisrow{mean}+\thisrow{std}] {\DataLaplace};
\addplot[name path=L1, draw=none, forget plot] table[x=K, y expr=\thisrow{mean}-\thisrow{std}] {\DataLaplace};
\addplot[niceblue, fill opacity=0.2, forget plot] fill between[of=U1 and L1];
\addplot[niceblue, thick, mark=*, mark size=1.5pt, mark options={fill=white, draw=niceblue}] table[x=K, y=mean] {\DataLaplace};
\end{axis}
\end{tikzpicture}
\caption{JS: Laplace$(0,1)$ vs.\ $\mathcal{N}(0,1)$}
\label{subfig:js_laplace}
\end{subfigure}\hfill%
\begin{subfigure}[t]{0.32\linewidth}
\centering
\begin{tikzpicture}
\begin{axis}[
  myaxisstyle,
  ymin=1.4, ymax=2.1,
]
\addplot[nicegray, thick, dashed, domain=0:550] {2.0};
\addplot[name path=U2, draw=none, forget plot] table[x=K, y expr=\thisrow{mean}+\thisrow{std}] {\DataMean};
\addplot[name path=L2, draw=none, forget plot] table[x=K, y expr=\thisrow{mean}-\thisrow{std}] {\DataMean};
\addplot[niceblue, fill opacity=0.2] fill between[of=U2 and L2];
\addplot[niceblue, thick, mark=*, mark size=1.5pt, mark options={fill=white, draw=niceblue}] table[x=K, y=mean] {\DataMean};
\end{axis}
\end{tikzpicture}
\caption{KL: Mean Shift ($\Delta=2$)}
\label{subfig:kl_mean}
\end{subfigure}\hfill%
\begin{subfigure}[t]{0.32\linewidth}
\centering
\begin{tikzpicture}
\begin{axis}[
  myaxisstyle,
  ymin=0.24, ymax=0.34,
]
\addplot[nicegray, thick, dashed, domain=0:550] {0.318147};
\addplot[name path=U3, draw=none, forget plot] table[x=K, y expr=\thisrow{mean}+\thisrow{std}] {\DataScale};
\addplot[name path=L3, draw=none, forget plot] table[x=K, y expr=\thisrow{mean}-\thisrow{std}] {\DataScale};
\addplot[niceblue, fill opacity=0.2] fill between[of=U3 and L3];
\addplot[niceblue, thick, mark=*, mark size=1.5pt, mark options={fill=white, draw=niceblue}] table[x=K, y=mean] {\DataScale};
\end{axis}
\end{tikzpicture}
\caption{KL: Scale Change ($\sigma=2$)}
\label{subfig:kl_scale}
\end{subfigure}

\vspace{0.1cm}
\ref{CommonLegend}

\caption{Convergence of the Rank-$f$ estimator as the number of basis functions $K$ increases. The estimator (blue) consistently converges to the true analytic or Monte Carlo reference (dashed gray) across different divergence types and scenarios.}
\label{fig:all_convergence}
\end{figure*}

\begin{figure*}[!htbp]
\centering

\definecolor{niceblue}{RGB}{31, 119, 180}
\definecolor{niceorange}{RGB}{255, 127, 14}
\definecolor{nicegreen}{RGB}{44, 160, 44}
\definecolor{nicegray}{RGB}{80, 80, 80}

\pgfplotstableread[col sep=space]{
x    ratio   err
0.5  1.003   0.042
1.0  0.992   0.035
2.0  0.991   0.009
}\DataMeanJS

\pgfplotstableread[col sep=space]{
x    ratio   err
0.5  1.030   0.066
1.0  0.969   0.027
2.0  0.933   0.015
}\DataMeanKL

\pgfplotstableread[col sep=space]{
x    ratio   err
1.2  0.954   0.072
1.5  0.958   0.031
2.0  0.982   0.022
}\DataScaleKL

\pgfplotstableread[col sep=space]{
x    ratio   err
1.2  0.986   0.106
1.5  0.948   0.042
2.0  0.965   0.013
}\DataScaleHell

\pgfplotsset{
    ratioaxis/.style={
        width=\linewidth,
        height=5.5cm,
        axis line style={-},
        axis x line=bottom,
        axis y line=left,
        grid=major,
        grid style={dashed, gray!30},
        label style={font=\small},
        tick label style={font=\small},
        ymin=0.85, ymax=1.15,
        ylabel={Ratio (Estimate / Truth)},
        ytick={0.9, 0.95, 1.0, 1.05, 1.1},
        yticklabels={0.9, 0.95, 1.0, 1.05, 1.1},
        legend style={draw=none, fill=none, font=\small},
    }
}

\begin{subfigure}[t]{0.48\linewidth}
\centering
\begin{tikzpicture}
\begin{axis}[
    ratioaxis,
    xmin=0.4, xmax=2.1,
    xtick={0.5, 1.0, 2.0},
    xlabel={Mean Shift $\Delta$},
]

\addplot[nicegray, thick, dashed, domain=0:3] {1.0};
\addlegendentry{Ideal (1.0)}

\addplot[niceblue, thick, mark=*, mark options={fill=white, draw=niceblue, line width=1pt}, error bars/.cd, y dir=both, y explicit] 
table[x=x, y=ratio, y error=err] {\DataMeanJS};
\addlegendentry{Jensen–Shannon}

\addplot[niceorange, thick, mark=square*, mark options={fill=white, draw=niceorange, line width=1pt}, error bars/.cd, y dir=both, y explicit] 
table[x=x, y=ratio, y error=err] {\DataMeanKL};
\addlegendentry{KL Divergence}

\end{axis}
\end{tikzpicture}
\caption{Accuracy vs.\ Mean Shift $\Delta$ ($\mathcal{N}(0,1)$ vs $\mathcal{N}(\Delta,1)$)}
\label{fig:robustness_mean}
\end{subfigure}\hfill%
\begin{subfigure}[t]{0.48\linewidth}
\centering
\begin{tikzpicture}
\begin{axis}[
    ratioaxis,
    xmin=1.1, xmax=2.1,
    xtick={1.2, 1.5, 2.0},
    xlabel={Standard Deviation $\sigma$},
]

\addplot[nicegray, thick, dashed, domain=0:3] {1.0};
\addlegendentry{Ideal (1.0)}

\addplot[niceorange, thick, mark=square*, mark options={fill=white, draw=niceorange, line width=1pt}, error bars/.cd, y dir=both, y explicit] 
table[x=x, y=ratio, y error=err] {\DataScaleKL};
\addlegendentry{KL Divergence}

\addplot[nicegreen, thick, mark=triangle*, mark size=3pt, mark options={fill=white, draw=nicegreen, line width=1pt}, error bars/.cd, y dir=both, y explicit] 
table[x=x, y=ratio, y error=err] {\DataScaleHell};
\addlegendentry{Hellinger$^2$}

\end{axis}
\end{tikzpicture}
\caption{Accuracy vs.\ Scale $\sigma$ ($\mathcal{N}(0,1)$ vs $\mathcal{N}(0,\sigma)$)}
\label{fig:robustness_scale}
\end{subfigure}

\caption{Robustness of the Rank-$f$ estimator at fixed $K=256$. The plots show the ratio of the estimated divergence to the ground truth (closer to 1.0 is better) as the shift ($\Delta$) or scale ($\sigma$) increases.}
\label{fig:robustness_combined}
\end{figure*}

\begin{figure*}[!htbp]
\centering

\definecolor{niceblue}{RGB}{31, 119, 180}
\definecolor{niceorange}{RGB}{255, 127, 14}
\definecolor{nicegreen}{RGB}{44, 160, 44}
\definecolor{nicegray}{RGB}{80, 80, 80}

\pgfplotstableread[col sep=space]{
K    ratio
8    0.441
16   0.648
32   0.803
64   0.898
128  0.953
256  0.982
512  0.990
1024 1.008
}\DataConvScale

\pgfplotstableread[col sep=space]{
K    ratio
8    0.714
16   0.816
32   0.880
64   0.924
128  0.952
256  0.969
512  0.980
1024 0.988
}\DataConvMean

\pgfplotstableread[col sep=space]{
K    ratio
8    0.137
16   0.303
32   0.488
64   0.651
128  0.778
256  0.869
512  0.933
1024 0.979
}\DataConvHeavy

\begin{tikzpicture}
\begin{axis}[
    width=0.8\linewidth, 
    height=7.0cm,
    xmode=log,
    ymode=log,
    log basis x={2},
    log basis y={10},
    axis line style={-},
    axis x line=bottom,
    axis y line=left,
    grid=major,
    grid style={dashed, gray!30},
    label style={font=\small},
    tick label style={font=\footnotesize},
    xtick={8, 16, 32, 64, 128, 256, 512, 1024},
    xticklabels={8, 16, 32, 64, 128, 256, 512, 1024},
    xlabel={Number of Rank Bins ($K$)},
    ylabel={Estimation Error $|1 - \text{Ratio}|$},
    ymin=0.005, ymax=2.0,
    xmin=6, xmax=1600,
    legend style={
        at={(0.5,-0.25)},      
        anchor=north,          
        legend columns=3,      
        draw=none,             
        fill=none,             
        font=\footnotesize,
        column sep=1em,        
        /tikz/every even column/.append style={column sep=0.5cm} 
    },
]


\addplot[nicegray, thick, dashed, domain=8:1024] {10 * x^(-1)};
\addlegendentry{Theory $\mathcal{O}(K^{-1})$}

\addplot[nicegray, thick, dashdotted, domain=8:1024] {1.2 * x^(-0.65)};
\addlegendentry{Theory $\approx \mathcal{O}(K^{-0.65})$}

\addplot[nicegray, thick, dotted, domain=8:1024] {2.5 * x^(-0.5)};
\addlegendentry{Theory $\mathcal{O}(K^{-0.5})$}


\addplot[nicegreen, thick, mark=triangle*, mark size=2.5pt, mark options={fill=white, draw=nicegreen, line width=1pt}] 
    table[x=K, y expr={abs(1-\thisrow{ratio})}] {\DataConvScale};
\addlegendentry{Scale change ($\sigma=2$)}

\addplot[niceorange, thick, mark=square*, mark options={fill=white, draw=niceorange, line width=1pt}] 
    table[x=K, y expr={abs(1-\thisrow{ratio})}] {\DataConvMean};
\addlegendentry{Mean shift ($\Delta=1$)}

\addplot[niceblue, thick, mark=*, mark size=2.5pt, mark options={fill=white, draw=niceblue, line width=1pt}] 
    table[x=K, y expr={abs(1-\thisrow{ratio})}] {\DataConvHeavy};
\addlegendentry{Heavy tails (Laplace)}

\draw[->, gray, thin] (axis cs: 64, 0.06) -- (axis cs: 64, 0.09);

\end{axis}
\end{tikzpicture}

\caption{Convergence rate analysis across different tail behaviors. We compare the empirical error against theoretical slopes. The Scale Change (Green) is bounded and achieves the optimal $\mathcal{O}(K^{-1})$ rate. The Mean Shift (Orange) suffers from an unbounded density ratio at the tail, degrading convergence to $\approx \mathcal{O}(K^{-0.65})$. The Heavy Tail (Blue) case is the most difficult, bounded by the Hölder continuity limit of $\mathcal{O}(K^{-0.5})$.}
\label{fig:convergence_shared_legend}
\end{figure*}

\subsection{Sliced rank-statistic $f$-divergences: empirical convergence} \label{app:Multivariate benchmarks via random projections}

This appendix provides the full experimental protocol underlying \cref{sec:exp_sliced_nd}.
For each configuration (dimension $d$ and distribution pair $(\mu,\nu)$), we draw
$n_\mu$ i.i.d.\ samples from $\mu$ and $n_\nu$ i.i.d.\ samples from $\nu$, compute the sliced
rank-$f$ estimate using $L$ random directions and rank order $K$, and repeat the procedure for
$R$ independent runs, reporting mean$\pm$std. Unless stated otherwise we use
$K=64$, $L=128$, $n_\mu=n_\nu=10{,}000$, and $R=10$. Random directions are sampled uniformly on
$\mathbb{S}^{d-1}$.

\paragraph{Distribution families.}
We consider the following multivariate mismatch families (in dimension $d$):
(i) Gaussian mean shifts, $\mu=\mathcal N(0,I_d)$ and $\nu=\mathcal N(\Delta e_1,I_d)$ with
$\Delta\in\{0,0.5,1.0\}$; (ii) Gaussian scale changes, $\mu=\mathcal N(0,I_d)$ and
$\nu=\mathcal N(0,\sigma^2 I_d)$ with $\sigma\in\{1.0,1.2,1.5,2.0\}$; (iii) anisotropic Gaussian
covariance mismatch, $\mu=\mathcal N(0,I_d)$ and $\nu=\mathcal N(0,\mathrm{diag}(1,\dots,2))$;
and (iv) non-Gaussian comparisons, including factorized Laplace vs.\ Gaussian for JS,
Student-$t$ vs.\ Gaussian for KL, and a symmetric two-component Gaussian mixture vs.\ Gaussian for JS.

\paragraph{Scaling and reported ratios.}
Alongside the sliced estimate $D^{(K)}_{f,\nu}(\mu)$, we also report the simple normalization
$d\times D^{(K)}_{f,\nu}(\mu)$ and summarize accuracy via the ratio
$(d\times\text{sliced})/\text{true}$. This scaling is not intended to be exact in general; it is
a lightweight calibration that keeps ratios on a comparable scale across dimensions.

\paragraph{Reference (``ground-truth'') divergences.}
Reference divergences are computed as follows, depending on whether closed forms are available:
\begin{itemize}
  \item \textbf{Gaussian-Gaussian (analytic references).}
  For $\mu=\mathcal N(\mu_0,\Sigma_0)$ and $\nu=\mathcal N(\mu_1,\Sigma_1)$, the reference
  $\mathrm{KL}(\mu\|\nu)$ and squared Hellinger $\mathrm{H}^2(\mu,\nu)$ are evaluated in closed form.
  This avoids an additional numerical approximation layer, so discrepancies can be attributed to the
  rank-statistic estimator rather than to the reference computation.

  \item \textbf{Gaussian Jensen-Shannon (deterministic proxy).}
  The multivariate Jensen-Shannon divergence between two Gaussians does not admit a simple closed form
  because the mixture $\tfrac12\mu+\tfrac12\nu$ is not Gaussian. To keep the reference deterministic
  (and avoid injecting extra Monte Carlo variance), we approximate the mixture by a single Gaussian
  with matched mean and covariance (moment matching), and define the reference as
  $\tfrac12\mathrm{KL}(\mu\|M)+\tfrac12\mathrm{KL}(\nu\|M)$ for that matched Gaussian $M$.

  \item \textbf{Non-Gaussian pairs (Monte Carlo from known log-densities).}
  When at least one distribution is non-Gaussian (e.g., Laplace vs.\ Gaussian, Student-$t$ vs.\ Gaussian,
  or a Gaussian mixture vs.\ Gaussian), a closed-form multivariate reference is typically unavailable.
  In these cases, the reference divergence is computed by Monte Carlo from its expectation form using
  the \emph{known} log-densities (e.g., $\mathrm{KL}(\mu\|\nu)=\mathbb E_{\mu}[\log p_\mu(X)-\log p_\nu(X)]$,
  and similarly for $\mathrm{JS}$ via expectations under $\mu$ and $\nu$).
  All evaluations are performed in the log domain using log-sum-exp to ensure numerical stability.
\end{itemize}

Figure~\ref{fig:three_plots_wide} and Table~\ref{tab:final_cleaned_table} provide complementary views of the same phenomenon.
Figure~\ref{fig:three_plots_wide} isolates the Gaussian mean-shift setting and shows that the normalized quantity
$(d\times\text{sliced})/\text{true}$ remains close to the ideal value $1$ across dimensions, with moderate,
dimension-dependent deviations that are consistent with a mismatch between the sliced functional and the full
multivariate divergence (and with the crudeness of the $d\times(\cdot)$ calibration). Table~\ref{tab:final_cleaned_table}
summarizes this behavior across a broader set of distribution pairs: for Gaussian-Gaussian benchmarks (KL and
Hellinger$^2$) the ratios typically stay near $1$, while JS experiments that rely on the Gaussian-proxy reference and
non-Gaussian misspecification cases exhibit larger and more variable departures, especially as $d$ grows, highlighting
that the sliced estimator is best interpreted as a stable, sample-based surrogate whose absolute scale can drift from
the multivariate reference in challenging regimes. Finally, Figure~\ref{fig:side_by_side_convergence} illustrates how
increasing the rank resolution $K$ systematically reduces the finite-$K$ approximation gap in representative cases,
with ratios approaching the ideal baseline as $K$ increases.


\begin{table*}[!htbp]
\centering
\small 
\sisetup{
    separate-uncertainty = true,
    uncertainty-separator = {\,\pm\,},
    bracket-numbers = false, 
    detect-all,
    table-format = 1.3(3)
}
\renewcommand{\arraystretch}{1.2}
\setlength{\tabcolsep}{3pt} 

\begin{tabular}{ll c SSSSS}
\toprule
& & & \multicolumn{5}{c}{\textbf{Dimension} $d$} \\
\cmidrule(lr){4-8}
\textbf{Setting} & \textbf{$f$-divergence} & \textbf{Parameter} & {$2$} & {$5$} & {$10$} & {$20$} & {$50$} \\
\midrule

\multirow{6}{*}{Mean Shift} 
    & KL & $\Delta=0.5$ & 1.015(67) & 1.098(52) & 1.188(72) & 0.977(60) & 1.304(44) \\
    & KL & $\Delta=1.0$ & 0.991(30) & 1.087(32) & 1.170(31) & 0.899(28) & 1.113(42) \\
    \cmidrule(lr){2-8}
    & $\text{Hellinger}^{2}$ & $\Delta=0.5$ & 1.005(62) & 0.935(48) & 0.973(49) & 1.002(37) & 1.035(60) \\
    & $\text{Hellinger}^{2}$ & $\Delta=1.0$ & 0.972(33) & 0.931(22) & 0.965(48) & 0.943(42) & 0.851(21) \\
    \cmidrule(lr){2-8}
    & JS (Gaussian) & $\Delta=0.5$ & 1.007(76) & 0.899(40) & 0.887(31) & 1.006(60) & 1.234(38) \\
    & JS (Gaussian) & $\Delta=1.0$ & 0.979(43) & 0.892(13) & 0.895(35) & 0.951(41) & 1.144(31) \\

\midrule

\multirow{4}{*}{\shortstack[l]{Scale and \\ Covariance}} 
    & JS (Scale) & $\sigma=1.2$ & 0.891(52) & 0.936(48) & 0.913(30) & 0.851(16) & 0.846(16) \\
    & JS (Scale) & $\sigma=1.5$ & 0.856(27) & 0.856(15) & 0.858(13) & 0.797(07) & 0.803(06) \\
    & JS (Scale) & $\sigma=2.0$ & 0.776(14) & 0.775(10) & 0.785(05) & 0.734(05) & 0.734(03) \\
    \cmidrule(lr){2-8}
    & JS (Anisotropic) & {---} & 0.837(37) & 0.761(40) & 0.791(32) & 0.726(15) & 0.736(11) \\

\midrule

\multirow{3}{*}{\shortstack[l]{Model \\ Misspecification}} 
    & JS (Laplace vs. Gaussian) & {---} & 0.767(15) & 0.861(16) & 1.052(17) & 1.227(15) & 1.968(14) \\
    & KL ($t$-dist vs. Gaussian) & $df=3$ & 0.164(23) & 0.218(17) & 0.267(15) & 0.268(10) & 0.318(12) \\
    & JS (GMM vs. Gaussian) & $\Delta=1.0$ & 0.923(06) & 0.759(09) & 0.753(07) & 1.012(09) & 2.247(16) \\

\bottomrule
\end{tabular}
\caption{Ratio summary across dimensions ($d$). Values report mean $\pm$ standard deviation over 10 runs.}
\label{tab:final_cleaned_table}
\end{table*}

\begin{figure*}[!htbp]
\centering

\definecolor{niceblue}{RGB}{31, 119, 180}
\definecolor{niceorange}{RGB}{255, 127, 14}
\definecolor{nicegreen}{RGB}{44, 160, 44}
\definecolor{nicegray}{RGB}{80, 80, 80}


\pgfplotstableread[col sep=comma]{
K,      ratio
8,      0.743
16,     0.844
32,     0.913
128,    0.982
256,    0.998
512,    1.008
}\DataMeanShiftTwo

\pgfplotstableread[col sep=comma]{
K,      ratio
8,      0.338
16,     0.514
32,     0.646
128,    0.777
256,    0.802
512,    0.815
}\DataScaleTwo

\pgfplotstableread[col sep=comma]{
K,      ratio
8,      0.193
16,     0.343
32,     0.503
128,    0.749
256,    0.826
512,    0.878
}\DataLaplaceTwo

\pgfplotstableread[col sep=comma]{
K,      ratio
8,      0.832
16,     0.941
32,     1.013
64,     1.059
128,    1.087
256,    1.106
512,    1.120
}\DataMeanShiftFive

\pgfplotstableread[col sep=comma]{
K,      ratio
8,      0.337
16,     0.512
32,     0.645
64,     0.728
128,    0.775
256,    0.801
512,    0.814
}\DataScaleFive

\pgfplotstableread[col sep=comma]{
K,      ratio
8,      0.316
16,     0.479
32,     0.634
64,     0.763
128,    0.861
256,    0.932
512,    0.982
}\DataLaplaceFive


\begin{subfigure}[b]{0.48\textwidth}
    \centering
    \begin{tikzpicture}
    \begin{axis}[
        width=\linewidth,
        height=6.0cm,
        xmode=log,
        log basis x={2},
        xtick={8, 16, 32, 64, 128, 256, 512},
        xticklabels={8, 16, 32, 64, 128, 256, 512},
        xlabel={Number of Rank Bins ($K$)},
        ylabel={Ratio Estimate ($\hat{D}/D$)},
        ymin=0.1, ymax=1.25,
        grid=major,
        grid style={dashed, gray!30},
        title style={font=\bfseries},
        label style={font=\small},
        tick label style={font=\footnotesize},
        legend to name=CommonLegend,
        legend style={
            legend columns=3,
            draw=none,
            fill=none,
            font=\footnotesize,
            column sep=1.5em,
        },
    ]
    \addplot[nicegray, dashed, thick, domain=6:600] {1.0};
    
    \addplot[niceblue, thick, mark=*, mark options={fill=white, line width=0.8pt}] table[x=K, y=ratio] {\DataMeanShiftTwo};
    \addlegendentry{Mean Shift (KL, $\Delta=1$)}
    
    \addplot[niceorange, thick, mark=square*, mark options={fill=white, line width=0.8pt}] table[x=K, y=ratio] {\DataScaleTwo};
    \addlegendentry{Scale Change (JS, $\sigma=2$)}
    
    \addplot[nicegreen, thick, mark=triangle*, mark size=3pt, mark options={fill=white, line width=0.8pt}] table[x=K, y=ratio] {\DataLaplaceTwo};
    \addlegendentry{Laplace (JS, Heavy Tail)}

    \end{axis}
    \end{tikzpicture}
    \caption{Dimension $d=2$}
    \label{fig:d2}
\end{subfigure}
\hfill
\begin{subfigure}[b]{0.48\textwidth}
    \centering
    \begin{tikzpicture}
    \begin{axis}[
        width=\linewidth,
        height=6.0cm,
        xmode=log,
        log basis x={2},
        xtick={8, 16, 32, 64, 128, 256, 512},
        xticklabels={8, 16, 32, 64, 128, 256, 512},
        xlabel={Number of Rank Bins ($K$)},
        ymin=0.1, ymax=1.25,
        grid=major,
        grid style={dashed, gray!30},
        label style={font=\small},
        tick label style={font=\footnotesize},
    ]
    \addplot[nicegray, dashed, thick, domain=6:600] {1.0};
    
    \addplot[niceblue, thick, mark=*, mark options={fill=white, line width=0.8pt}] table[x=K, y=ratio] {\DataMeanShiftFive};
    \addplot[niceorange, thick, mark=square*, mark options={fill=white, line width=0.8pt}] table[x=K, y=ratio] {\DataScaleFive};
    \addplot[nicegreen, thick, mark=triangle*, mark size=3pt, mark options={fill=white, line width=0.8pt}] table[x=K, y=ratio] {\DataLaplaceFive};
    

    \end{axis}
    \end{tikzpicture}
    \caption{Dimension $d=5$}
    \label{fig:d5}
\end{subfigure}

\vspace{0.5em}
\ref*{CommonLegend}

\caption{\small Convergence of ratio estimates ($\hat{D}/D$) versus number of rank bins ($K$) for dimensions $d=2$ (left) and $d=5$ (right).}
\label{fig:side_by_side_convergence}
\end{figure*}

\clearpage

\subsection{Generative transport dynamics for rank-statistic $f$-divergences} \label{app:Generative Transport Dynamics for Rank $f$-divergence}

We next provide pseudocode for the sliced rank-proximal transport update used in
\cref{subsec:flows}.

\begin{algorithm}[!htbp]
\begin{algorithmic}
  \STATE {\bfseries Input:} particles $(x_i)_{i=1}^{N}\subset\R^d$, reference samples $(y_j)_{j=1}^{M}\subset\R^d$,
  slices $L$, rank order $K$, temperature $\tau$, trust $\eta$, step size $\varepsilon$, $f$-generator $f(\cdot)$.
  \STATE {\bfseries Output:} updated particles $(x_i)_{i=1}^{N}$.

  \STATE Draw directions $s_1,\dots,s_L\in\S^{d-1}$ (optionally include antithetic pairs $\pm s_\ell$)
  \STATE Initialize $\Delta x_i \gets 0\in\R^d$ for all $i=1,\dots,N$

  \FOR{$\ell=1$ {\bfseries to} $L$}
    \STATE Project: $x_i^{(\ell)}\gets\langle x_i,s_\ell\rangle$, $y_j^{(\ell)}\gets\langle y_j,s_\ell\rangle$
    \STATE Soft ranks: $v^{(\ell)}_{0,i}\gets \widehat R_{\nu^{(\ell)},\tau}(x_i^{(\ell)})\in(0,1)$ for $i=1,\dots,N$; write $\mathbf v^{(\ell)}_0=(v^{(\ell)}_{0,1},\dots,v^{(\ell)}_{0,N})$
    \STATE Prox in rank space (approx.\ by \textsc{SGD}/\textsc{ULA}/\textsc{MALA}):
    \[
      \mathbf v^{(\ell)}_{1}
      \approx
      \argmin_{\mathbf v\in(0,1)^N}
      \Big\{
      D_f(\widehat Q^{(K)}(\mathbf v)\Vert U_K)
      +
      \tfrac{1}{2\eta}\|\mathbf v-\mathbf v^{(\ell)}_{0}\|_2^2
      \Big\}
    \]
    \STATE Quantile match: $z_i^{(\ell)}\gets Q_{\widehat\nu^{(\ell)}}(v^{(\ell)}_{1,i})$, \;
           $\delta_i^{(\ell)}\gets z_i^{(\ell)}-x_i^{(\ell)}$
    \STATE Accumulate: $\Delta x_i \gets \Delta x_i + \delta_i^{(\ell)} s_\ell$
  \ENDFOR

  \STATE Update: $x_i \gets x_i + \varepsilon\,\frac{d}{L}\,\Delta x_i$ \quad for all $i$
  \STATE {\bfseries return} $(x_i)_{i=1}^N$
\end{algorithmic}
  \caption{Sliced rank-proximal transport (one outer step)}
  \label{alg:rank-prox-transport-compact}
\end{algorithm}

\subsection{Generative experiments on image datasets} \label{app:CIFAR10}

In this appendix, we provide implementation details for the image-generation
experiments. Section~\ref{app:cifar10-corpt} gives the full experimental
setup for the CO-RPT experiment discussed in
Section~\ref{subsubsec:celeba}. Section~\ref{app:cifar10-feature-rank-tv}
reports an additional feature-space rank-TV training experiment on CIFAR-10,
which is included as supplementary evidence that rank-based objectives can
provide a useful training signal for neural generators.

\subsubsection{CO-RPT on CIFAR-10} \label{app:cifar10-corpt}

\paragraph{Center-outward rank-proximal transport (CO-RPT)}

Algorithm~\ref{alg:co-rpt} implements a slice-free variant of rank-proximal transport based on a
center-outward decomposition. Starting from particles $X=\{x_i\}_{i=1}^N$ and reference samples
$Y=\{y_j\}_{j=1}^M$, we first recenter the configuration using the target mean
$\bar y=\frac1M\sum_{j=1}^M y_j$, i.e., $\tilde x_i=x_i-\bar y$ and $\tilde y_j=y_j-\bar y$.
Optionally, we apply a whitening transform (e.g.\ ZCA fitted on $\tilde Y$) so that the target is
approximately isotropic; this reduces anisotropy and makes the radial/angular decomposition more
stable, and the inverse transform is applied at the end.

We then decompose each point into its radius and direction:
$r_i^x=\|\tilde x_i\|$, $u_i^x=\tilde x_i/r_i^x$ and $r_j^y=\|\tilde y_j\|$, $u_j^y=\tilde y_j/r_j^y$. The transport step is built by updating radii through
a one-dimensional rank-proximal refinement, and updating directions through a simple matching on
the unit sphere. Concretely, we compute soft radial ranks
$v_{0,i}\approx \widehat R_{r^y,\tau}(r_i^x)\in[0,1]$ using a smoothed empirical CDF of the target radii,
where the temperature $\tau$ controls how sharp the rank assignment is. We refine these ranks by
approximately solving the proximal objective
\[
\mathbf v_{1} \approx \argmin_{\mathbf v\in(0,1)^N}\Big\{
D_f(\widehat Q^{(K)}(\mathbf v)\Vert U_K)
+
\tfrac{1}{2\eta}\|\mathbf v-\mathbf v_{0}\|_2^2
\Big\},
\]
where $\widehat Q^{(K)}(\mathbf v)$ is the Bernstein-smoothed histogram over $[K]$ induced by $\mathbf v$,
$D_f(\cdot\Vert U_K)$ measures deviation from uniformity, and $\eta>0$ acts as a trust region that
prevents overly aggressive rank updates. In practice, we use only a few inner iterations (e.g., \ SGD
or a simple extragradient update). The refined ranks are mapped back to the radial axis via the
empirical quantile of the target radii, $r_i^\star=Q_{\widehat r^y}(v_{1,i})$, ensuring that
uniform ranks would reproduce the target radial distribution.

To align the angular structure, we match each particle direction $u_i^x$ to a nearby target direction
using cosine similarity, i.e.\ $j^\star(i)\in\argmax_j\langle u_i^x,u_j^y\rangle$, and blend toward
it with weight $\beta\in[0,1]$:
$u_i^\star=\mathrm{normalize}\big((1-\beta)u_i^x+\beta u_{j^\star(i)}^y\big)$.
Setting $\beta=0$ yields a purely radial update, while larger $\beta$ accelerates angular
adaptation. Combining the transported radius and direction gives the center-outward proposal
$\tilde x_i^\star=r_i^\star u_i^\star$, and we take an outer step
$\tilde x_i^+=\tilde x_i+\varepsilon(\tilde x_i^\star-\tilde x_i)$ with step size $\varepsilon>0$.
Optionally, we clip the increment to a maximum norm $c$ to avoid rare large jumps. Finally, we
undo the optional whitening and add back the center $\bar y$ to obtain the updated particles
$X^+=\{x_i^+\}_{i=1}^N$.

Overall, CO-RPT replaces the multi-slice quantile matching of
\eqref{eq:particle_transport_update} by a single univariate rank-prox update on radii together with
a lightweight spherical matching for directions. This yields a geometrically interpretable,
fully sample-based update that explicitly controls the radial marginal through the rank objective,
while encouraging directional alignment through nearest-neighbor coupling on $\mathbb S^{d-1}$.

\begin{algorithm}[!htbp]
  \caption{Center-outward rank-proximal transport (CO-RPT) update}
  \label{alg:co-rpt}
\begin{algorithmic}
  \STATE {\bfseries Input:} particles $X=\{x_i\}_{i=1}^N$, targets $Y=\{y_j\}_{j=1}^M$,
  $K,f,\tau,\eta$, step $\varepsilon$, angular blend $\beta$ (optional cap $c$).
  \STATE {\bfseries Output:} updated particles $X^+$

  \STATE (Optional) center/whiten: $\tilde x_i,\tilde y_j$.
  \STATE Compute radii/directions: $r^x_i=\|\tilde x_i\|,\ u^x_i=\tilde x_i/r^x_i$ and $r^y_j=\|\tilde y_j\|,\ u^y_j=\tilde y_j/r^y_j$.
  \STATE Soft radial ranks: $v_{0,i} \approx \widehat R_{r^y,\tau}(r^x_i)\in(0,1)$ for $i=1,\dots,N$; write $\mathbf v_0=(v_{0,1},\dots,v_{0,N})$.
  \STATE Prox in rank space: $\mathbf v_1 \approx \argmin_{\mathbf v\in(0,1)^N} \Big\{ D_f(\widehat Q^{(K)}(\mathbf v)\Vert U_K)+\tfrac{1}{2\eta}\|\mathbf v-\mathbf v_0\|_2^2 \Big\}$.
  \STATE Target radii: $r_i^\star \gets Q_{\widehat r^y}(v_{1,i})$ for $i=1,\dots,N$.
  \STATE Angular match: $j^\star(i)\in\argmax_j \langle u^x_i,u^y_j\rangle$, \;\;
        $u^\star_i\gets \mathrm{normalize}\!\big((1-\beta)u^x_i+\beta u^y_{j^\star(i)}\big)$.
  \STATE Proposal: $\tilde x_i^\star \gets r_i^\star u_i^\star$, \;\; $\Delta_i \gets \tilde x_i^\star-\tilde x_i$ \;\; (optional: $\Delta_i \gets \Delta_i\cdot\min\{1,c/\|\Delta_i\|\}$).
  \STATE Update: $\tilde x_i^+ \gets \tilde x_i + \varepsilon\,\Delta_i$ and uncenter/unwhiten to get $x_i^+$.
  \STATE {\bfseries return} $X^+=\{x_i^+\}_{i=1}^N$
\end{algorithmic}
\end{algorithm}

\paragraph{Experimental setup and evaluation.}

We ran the proposed center-outward rank-proximal transport (CO-RPT) directly in pixel space on CIFAR-10.
We randomly selected $M=1000$ images from the CIFAR-10 training set, upsampled them to $64\times 64$ using bicubic interpolation with antialiasing, and mapped pixel intensities to $[0,1]$; each image was then flattened to $\mathbb{R}^{3HW}$ with $H=W=64$.
We used the Jensen--Shannon generator ($f=\mathrm{JS}$) with trust-region parameter $\eta=0.5$ and $3$ inner extragradient steps per outer iteration.
We initialized $N=M$ particles from a Gaussian matched to the mean and marginal standard deviation of the whitened target features, and iterated CO-RPT for $T=20000$ outer steps.
We linearly annealed the rank resolution from $K=64$ to $K=160$, the rank-smoothing temperature from $\tau=0.30$ to $\tau=0.07$, and the outer step size from $\varepsilon=0.16$ to $\varepsilon=0.10$, while clipping per-particle updates to a maximum norm of $c=0.30$.
We also used a small angular blending parameter $\beta_{\mathrm{angle}}=0.01$.

\subsubsection{Feature-space rank-TV training on CIFAR-10} \label{app:cifar10-feature-rank-tv}

We also evaluate the rank-based total variation objective as a feature-space
training signal on CIFAR-10. Since CIFAR-10 images are natively
$32\times32$, we resize them to $64\times64$ using bicubic interpolation
with antialiasing and normalize pixel values to $[-1,1]$. We use a
DCGAN-style generator and a spectrally normalized convolutional
discriminator. The discriminator is trained with the standard hinge loss,
while the generator is trained to match real and generated samples in the
feature space of the discriminator using the sliced rank-TV objective. We
also consider a hybrid variant with a small adversarial generator loss. An
exponential moving average of the generator weights is used for sampling and
evaluation.

The goal of this experiment is not to obtain state-of-the-art CIFAR-10
generation results, but to test whether the proposed rank-based discrepancy
provides a usable training signal in a nontrivial image setting. We observe a
steady improvement in FID during training: in our longest run, the FID
decreases from roughly $380$ at early iterations to about $41$ after
$245{,}000$ generator updates. This suggests that the feature-space rank-TV
objective can progressively improve image quality and distributional
matching.

The method can be applied after any fixed encoder, such as Inception, VGG, DINOv2, or a task-specific representation, yielding a rank divergence between feature pushforwards. Hence, the discrepancy is representation-dependent, as with FID or precision–recall metrics in learned feature spaces. Unlike FID, the rank statistic does not impose a Gaussian approximation in feature space; however, slicing may still miss informative directions when only finitely many projections are used.

\subsection{Addressing mode collapse on MNIST} \label{subsec:MNIST_mode_collapse}

Building on the pretraining approach of \citet{dFVOM2024b}, who employ a rank-based total-variation divergence to reduce mode collapse in GANs, we study whether the same strategy carries over to a wider class of rank-statistic $f$-divergences. Concretely, we incorporate the sliced rank-statistic $f$-divergences objective as a pretraining signal for the DCGAN generator \citep{radford2015unsupervised} and evaluate the resulting pipeline on MNIST. To measure both realism and coverage, we follow \citet{sajjadi2018assessing} and report precision (as a proxy for fidelity) and recall (as a proxy for diversity). All models are trained for $40$ epochs with batch size $128$. For the pretrained variants, we first run $20$ epochs under the sliced rank-statistic $f$-divergences objective and then continue with $40$ additional epochs of standard DCGAN training.

In Table \ref{table:model colapse}, we compare \emph{rank-statistic $f$-divergence} pretraining variants (TV, KL, JS, and $\mathrm{Hell}^2$) with standard DCGAN training, \emph{rank-statistic $f$-divergence}$+$DCGAN fine-tuning, and stronger multi-discriminator baselines \citep{durugkar2016generative, choi2022mcl}. Focusing on precision and recall, the standalone \emph{rank-statistic $f$-divergence} models already achieve strong recall on MNIST: JS and $\mathrm{Hell}^2$ are the most recall-oriented, reaching $97.0\%$ and $98.7\%$ recall for $m{=}50$, respectively, while TV and KL yield a more balanced behavior (e.g., $95.0\%$ and $91.1\%$ recall for $m{=}50$). When we pretrain with a \emph{rank-statistic $f$-divergence} and then fine-tune with the adversarial loss, precision increases substantially relative to vanilla DCGAN: TV$+$DCGAN and KL$+$DCGAN reach $95.0\%$ and $96.2\%$ precision, respectively (vs.\ $93.85\%$ for DCGAN), while maintaining competitive recall (around $92.8\%$ and $90.5\%$). Overall, these results highlight a clear trade-off between the choice of $f$ (more recall-oriented for JS/$\mathrm{Hell}^2$) and the benefit of adversarial fine-tuning for improving precision without collapsing recall.

Our estimator targets scalar $f$-divergences rather than precision–recall curves. However, precision–recall divergences form a family of $f$-divergence-type trade-off quantities \cite{SWR2019,SWSR2023,VNPC2023}; choosing a corresponding generator yields a rank-histogram approximation of the associated scalar PR divergence, while varying the PR parameter would produce a discretized PR trade-off curve. This differs from \cite{SWR2019}, who estimate the PR curve directly.

\begin{table}[ht]
\centering
\small
\setlength{\tabcolsep}{6pt} 
\renewcommand{\arraystretch}{1.1} 

\begin{tabular}{ll *{4}{S[table-format=2.2, table-figures-uncertainty=1]}}
\toprule
\textbf{Dataset} & \textbf{Method} & \multicolumn{2}{c}{\textbf{F-score}} & \multicolumn{2}{c}{\textbf{P\&R}} \\
\cmidrule(lr){3-4} \cmidrule(lr){5-6}
& & {$F_{1/8} \uparrow$} & {$F_8 \uparrow$} & {Precision $\uparrow$} & {Recall $\uparrow$} \\
\midrule
\textbf{MNIST} 
& TV (m=20) & 88.09 \pm 0.32 & 93.91 \pm 0.72 & 88.01 \pm 0.52 & 94.25 \pm 0.91 \\
& TV (m=50) & 88.89 \pm 0.31 & 94.90 \pm 0.71 & 88.80 \pm 0.50 & 95.08 \pm 0.94 \\
\addlinespace[8pt] 

& KL (m=20) & 90.50 \pm 0.43 & 90.21 \pm 0.62 & 90.50 \pm 0.47 & 90.18 \pm 0.91 \\
& KL (m=50) & 91.59 \pm 0.46 & 91.11 \pm 0.68 & 91.62 \pm 0.47 & 91.11 \pm 0.88 \\
\addlinespace[8pt]

& JS (m=20) & 86.64 \pm 0.56 & 96.13 \pm 0.76 & 86.48 \pm 0.40 & 96.32 \pm 0.92 \\
& JS (m=50) & 87.73 \pm 0.32 & 96.84 \pm 0.75 & 87.60 \pm 0.49 & 97.09 \pm 0.91 \\
\addlinespace[8pt]

& $\text{Hell}^{2}$ (m=20) & 83.00 \pm 0.35 & 97.92 \pm 0.74 & 82.81 \pm 0.50 & 98.23 \pm 0.87 \\
& $\text{Hell}^{2}$ (m=50) & 83.83 \pm 0.35 & 98.43 \pm 0.69 & 83.62 \pm 0.50 & 98.71 \pm 0.90 \\
\midrule
\addlinespace[4pt]
& DCGAN & 93.54 \pm 0.64 & 75.66 \pm 1.46 & 93.85 \pm 1.45 & 75.43 \pm 2.56 \\
& TV + DCGAN & 94.97 \pm 0.35 & 92.83 \pm 0.72 & 95.00 \pm 0.53 & 92.80 \pm 0.85 \\
\addlinespace[4pt]
& KL + DCGAN & 96.11 \pm 0.36 & 90.58 \pm 0.70 & 96.20 \pm 0.46 & 90.50 \pm 0.83 \\
& JS + DCGAN & 94.71 \pm 0.42 & 95.19 \pm 0.73 & 94.67 \pm 0.42 & 95.21 \pm 0.87 \\
& $\text{Hell}^{2}$ + DCGAN & 93.84 \pm 0.35 & 96.75 \pm 0.75 & 93.82 \pm 0.45 & 96.81 \pm 0.85 \\
\midrule
\addlinespace[4pt]
& GMAN & 97.78 \pm 0.40 & 96.52 \pm 0.57 & 97.80 \pm 0.71 & 96.53 \pm 0.89 \\
& \textbf{MCL-GAN} & \bfseries 98.20 \pm 0.19 & \bfseries 98.00 \pm 0.25 & \bfseries 98.20 \pm 0.30 & \bfseries 98.00 \pm 0.40 \\
\bottomrule
\end{tabular}
\caption{\small Quantitative results on MNIST ($28\times 28$), reporting $F_{1/8}$ and $F_8$ (the $\beta$-weighted harmonic means of precision and recall), precision, and recall (mean$\pm$std, \%). Results are grouped by divergence type for enhanced scannability.} \label{table:model colapse}
\end{table}

\subsection{Two-sample testing with rank chi-square statistics}
\label{app:two_sample_rank_chi}

We include a small two-sample testing experiment to illustrate that the proposed rank-statistic construction also yields a natural nonparametric test. Given two independent samples
\[
X_1,\ldots,X_N \sim \mu, 
\qquad 
Y_1,\ldots,Y_M \sim \nu,
\]
we test
\[
H_0:\mu=\nu
\qquad\text{against}\qquad
H_1:\mu\neq\nu .
\]
For a fixed rank resolution \(K\), we compute the rank histogram of the \(X\)-sample with respect to the \(Y\)-sample and compare it to the discrete uniform distribution on \(\{0,\ldots,K\}\). More precisely, denoting by \(\widehat Q^{(K)}_{N,M}\) the empirical rank histogram, we use the Pearson/Cressie-Read-type statistic \cite{cressie1984multinomial}
\[
T_K
=
\sum_{k=0}^{K}
\frac{
\left(
\widehat Q^{(K)}_{N,M}(k)-\frac{1}{K+1}
\right)^2
}{
\frac{1}{K+1}
}.
\]
Equivalently, up to a deterministic multiplicative factor, this is the empirical rank-statistic \(f\)-divergence associated with the chi-square generator
\[
f_{\chi^2}(t)=\frac12(t-1)^2.
\]
Thus, in the univariate case, the rank chi-square test is exactly the proposed rank-statistic divergence estimator specialized to the \(\chi^2\) entropy.

For higher-dimensional data, we use the corresponding sliced version based on random projections. We draw \(L\) random projection directions \(s_1,\ldots,s_L\in\mathbb S^{d-1}\) and compute the univariate rank chi-square statistic on each projected pair
\[
\{s_\ell^\top X_i\}_{i=1}^N,
\qquad
\{s_\ell^\top Y_j\}_{j=1}^M .
\]
The final statistic is obtained by averaging the projected discrepancies,
\[
T_{K,L}
=
\frac{1}{L}\sum_{\ell=1}^L T_K^{(\ell)}.
\]
When \(d=1\), there is only one possible projection, and the random-subspace statistic reduces to the univariate rank-statistic chi-square divergence described above.

Since the null distribution of the statistic may depend on the sample sizes and on the finite-sample rank construction, we calibrate the test by permutation. Specifically, we repeatedly permute the pooled sample
\[
\{X_1,\ldots,X_N,Y_1,\ldots,Y_M\},
\]
split it into two groups of sizes \(N\) and \(M\), recompute the statistic, and use the resulting empirical null distribution to obtain a \(p\)-value. We reject \(H_0\) at level \(\alpha=0.05\) whenever this permutation \(p\)-value is smaller than \(\alpha\).

We evaluate the resulting test on a collection of controlled synthetic two-sample benchmarks designed to probe different types of distributional mismatch. The benchmark families include Gaussian location shifts, Gaussian scale changes, Laplace and Student-\(t\) shape changes, Gaussian mixture alternatives, and dependence alternatives generated from \(t\)-, Clayton-, and Gumbel-copula models. These settings cover mean, scale, tail, multimodal, and dependence changes. We compare the sliced rank \(\chi^2\) test against standard two-sample baselines, namely Hotelling's \(T^2\), sliced optimal transport, MMD, and a tuned MMD variant.

We first assess calibration under the null. Figure~\ref{fig:rankchi_subspaces_calibration} reports the empirical rejection probability at nominal level \(\alpha=0.05\) across dimensions \(d\in\{2,4,10\}\). The results show that the permutation-calibrated sliced rank \(\chi^2\) statistic remains close to the target type-I error across the considered families, with deviations compatible with the Monte Carlo variability of the experiment.

We then evaluate power under controlled alternatives. Figures~\ref{fig:power_runtime_grouped_d4_scale_shape} and~\ref{fig:power_runtime_grouped_d4_dependence} report empirical power as a function of average runtime per repetition in dimension \(d=4\). The sliced rank \(\chi^2\) statistic provides a competitive efficiency--power tradeoff on scale and shape alternatives, reaching high power at substantially lower runtime than the more computationally expensive MMD and sliced optimal-transport baselines in these settings. For dependence alternatives, its behavior is more problem-dependent: it performs strongly on the Gumbel-copula alternative, while the Clayton-copula alternative is less favorable for this particular random-projection configuration.

Overall, these experiments indicate that the permutation-calibrated rank-statistic construction behaves as a calibrated two-sample test in the considered settings, and that, in the one-dimensional case, the rank chi-square subspace statistic coincides with our proposed estimator for the chi-square \(f\)-divergence.

\begin{figure}[t]
\centering


\pgfplotstableread[col sep=space]{
x pi se
1 0.020000 0.014000
2 0.040000 0.019596
3 0.060000 0.023749
4 0.070000 0.025515
5 0.040000 0.019596
6 0.060000 0.023749
7 0.060000 0.023749
8 0.050000 0.021794
}\RankCalibDtwo

\pgfplotstableread[col sep=space]{
x pi se
1 0.050000 0.021794
2 0.100000 0.030000
3 0.090000 0.028618
4 0.050000 0.021794
5 0.060000 0.023749
6 0.050000 0.021794
7 0.090000 0.028618
8 0.070000 0.025515
}\RankCalibDfour

\pgfplotstableread[col sep=space]{
x pi se
1 0.070000 0.025515
2 0.060000 0.023749
3 0.060000 0.023749
4 0.070000 0.025515
5 0.050000 0.021794
6 0.040000 0.019596
7 0.030000 0.017059
8 0.040000 0.019596
}\RankCalibDten


\newcommand{\rankcalibpanel}[3]{%
\begin{subfigure}[t]{0.32\linewidth}
\centering
\begin{tikzpicture}
\begin{axis}[
  width=\linewidth,
  height=0.82\linewidth,
  grid=major,
  major grid style={dashed, gray!30},
  xmin=0.5, xmax=8.5,
  ymin=0, ymax=0.16,
  tick align=outside,
  xtick={1,2,3,4,5,6,7,8},
  xticklabels={
    loc.,
    scale,
    Lap.,
    Stud.,
    mix.,
    t-cop.,
    Clay.,
    Gum.
  },
  x tick label style={rotate=40, anchor=east, font=\scriptsize},
  yticklabel style={font=\small},
  ylabel={#2},
  title={#1},
  title style={font=\small},
]

\addplot[
  gray,
  dashed,
  thick,
  domain=0.5:8.5,
  samples=2,
  forget plot
] {0.05}
node[pos=0.06, above, font=\tiny] {$\alpha=0.05$};

\addplot+[
  only marks,
  mark=*,
  mark size=2pt,
  thick,
  color=cRankSub,
  error bars/.cd,
    y dir=both,
    y explicit,
] table[
  x=x,
  y=pi,
  y error=se
] {#3};

\end{axis}
\end{tikzpicture}
\end{subfigure}
}


\rankcalibpanel{$d=2$}{Empirical size $\hat{\pi}$}{\RankCalibDtwo}
\hfill
\rankcalibpanel{$d=4$}{}{\RankCalibDfour}
\hfill
\rankcalibpanel{$d=10$}{}{\RankCalibDten}

\caption{\small Calibration of the rank-statistic \(\chi^2\)-divergence two-sample test under the null. The statistic corresponds to the sliced rank \(f\)-divergence with the chi-square generator \(f_{\chi^2}(t)=\frac12(t-1)^2\), calibrated by permutation. We report empirical rejection probabilities at nominal level \(\alpha=0.05\) across distributional families and dimensions. Error bars denote \(\hat{\pi}\pm\mathrm{SE}\) over \(R=100\) repetitions. All experiments use \(N=M=2000\), \(L=64\), \(K=4\), and \(B=500\) permutations. The dashed horizontal line indicates the nominal level.}
\label{fig:rankchi_subspaces_calibration}
\end{figure}



\definecolor{rfdRankSub}{RGB}{31,119,180}
\definecolor{rfdHotelling}{RGB}{255,127,14}
\definecolor{rfdSOT}{RGB}{44,160,44}
\definecolor{rfdMMD}{RGB}{214,39,40}
\definecolor{rfdMMDTuned}{RGB}{148,103,189}

\newcommand{\rfdmethodplotlegend}[4]{%
  \addplot[forget plot, draw=none, name path=#1_lo] table[x=ms,y=lo]{#4};
  \addplot[forget plot, draw=none, name path=#1_hi] table[x=ms,y=hi]{#4};
  \addplot[forget plot, fill=#2, fill opacity=0.10] fill between[of=#1_lo and #1_hi];
  \addplot[
    color=#2,
    thick,
    mark=*,
    mark size=1.6pt,
    mark options={fill=#2, opacity=0.85}
  ] table[x=ms,y=pi]{#4};
  \addlegendentry{#3}%
}

\newcommand{\rfdmethodplotnolegend}[4]{%
  \addplot[forget plot, draw=none, name path=#1_lo] table[x=ms,y=lo]{#4};
  \addplot[forget plot, draw=none, name path=#1_hi] table[x=ms,y=hi]{#4};
  \addplot[forget plot, fill=#2, fill opacity=0.10] fill between[of=#1_lo and #1_hi];
  \addplot[
    color=#2,
    thick,
    mark=*,
    mark size=1.6pt,
    mark options={fill=#2, opacity=0.85}
  ] table[x=ms,y=pi]{#4};
}

\newcommand{\rfdmethodplotlegendlight}[4]{%
  \addplot[
    color=#2,
    thick,
    mark=*,
    mark size=1.5pt,
    mark options={fill=#2, opacity=0.85}
  ] table[x=ms,y=pi]{#4};
  \addlegendentry{#3}%
}

\newcommand{\rfdmethodplotnolegendlight}[4]{%
  \addplot[
    color=#2,
    thick,
    mark=*,
    mark size=1.5pt,
    mark options={fill=#2, opacity=0.85}
  ] table[x=ms,y=pi]{#4};
}


\begin{figure}[t]
\centering


\pgfplotstableread[col sep=space]{
ms       pi       lo        hi
91.86    0.220000 0.178575  0.261425
132.04   0.500000 0.450000  0.550000
115.73   0.920000 0.892871  0.947129
109.24   0.980000 0.966000  0.994000
109.22   1.000000 1.000000  1.000000
140.65   1.000000 1.000000  1.000000
205.91   1.000000 1.000000  1.000000
}\RFDScaleRankSubDfour

\pgfplotstableread[col sep=space]{
ms       pi       lo        hi
17.66    0.060    0.036251  0.083749
18.94    0.020    0.006000  0.034000
26.57    0.060    0.036251  0.083749
47.36    0.030    0.012941  0.047059
51.83    0.070    0.044485  0.095515
123.88   0.050    0.028206  0.071794
}\RFDScaleHotellingDfour

\pgfplotstableread[col sep=space]{
ms       pi       lo        hi
802.58   0.340    0.292629  0.387371
837.70   1.000    1.000000  1.000000
1209.34  0.680    0.633352  0.726648
1349.08  1.000    1.000000  1.000000
1452.44  1.000    1.000000  1.000000
3143.23  1.000    1.000000  1.000000
}\RFDScaleSOTDfour

\pgfplotstableread[col sep=space]{
ms        pi       lo        hi
2055.99   0.530    0.480090  0.579910
4511.22   0.830    0.792437  0.867563
7831.55   1.000    1.000000  1.000000
28704.34  1.000    1.000000  1.000000
82422.26  1.000    1.000000  1.000000
}\RFDScaleMMDDfour

\pgfplotstableread[col sep=space]{
ms        pi        lo        hi
8652.95   0.180     0.141581  0.218419
17898.13  0.290     0.244624  0.335376
35100.97  0.910     0.881382  0.938618
80856.48  0.986111  0.972319  0.999903
}\RFDScaleMMDTunedDfour


\pgfplotstableread[col sep=space]{
ms       pi       lo        hi
117.49   0.280000 0.235100  0.324900
138.90   0.600000 0.551010  0.648990
121.23   0.960000 0.940404  0.979596
125.29   1.000000 1.000000  1.000000
125.56   1.000000 1.000000  1.000000
169.39   1.000000 1.000000  1.000000
260.83   1.000000 1.000000  1.000000
}\RFDLapRankSubDfour

\pgfplotstableread[col sep=space]{
ms       pi       lo        hi
14.88    0.020    0.006000  0.034000
15.72    0.060    0.036251  0.083749
20.40    0.030    0.012941  0.047059
29.67    0.070    0.044485  0.095515
67.38    0.060    0.036251  0.083749
113.70   0.020    0.006000  0.034000
}\RFDLapHotellingDfour

\pgfplotstableread[col sep=space]{
ms       pi       lo        hi
333.43   0.030    0.012941  0.047059
401.97   0.090    0.061382  0.118618
1251.83  0.490    0.440010  0.539990
1280.46  0.960    0.940404  0.979596
1630.10  1.000    1.000000  1.000000
3098.44  1.000    1.000000  1.000000
}\RFDLapSOTDfour

\pgfplotstableread[col sep=space]{
ms        pi       lo        hi
1363.36   0.200    0.160000  0.240000
4148.61   0.420    0.370644  0.469356
10747.47  0.970    0.952941  0.987059
24212.32  1.000    1.000000  1.000000
82206.26  1.000    1.000000  1.000000
}\RFDLapMMDDfour

\pgfplotstableread[col sep=space]{
ms        pi        lo        hi
8450.12   0.160     0.123339  0.196661
18745.97  0.170     0.132437  0.207563
40048.10  0.450     0.400251  0.499749
80283.02  0.775281  0.731037  0.819525
}\RFDLapMMDTunedDfour


\begin{subfigure}[t]{0.49\linewidth}
\centering
\begin{tikzpicture}
\begin{axis}[
  width=\linewidth,
  height=0.72\linewidth,
  xmode=log,
  log basis x=10,
  grid=major,
  major grid style={dashed, gray!30},
  xmin=12, xmax=1e5,
  ymin=0, ymax=1.05,
  tick align=outside,
  xlabel={Avg runtime per repetition (ms, log scale)},
  ylabel={Power $\hat{\pi}$},
  ytick={0,0.25,0.5,0.75,1.0},
  yticklabels={0,0.25,0.50,0.75,1.00},
  yticklabel style={font=\scriptsize},
  x tick label style={font=\scriptsize},
  legend to name=RFDGlobalLegendPowerDfour,
  legend columns=3,
  legend cell align=left,
  legend style={
    draw=gray!40,
    fill=white,
    rounded corners=2pt,
    font=\scriptsize,
    /tikz/every even column/.append style={column sep=7pt}
  },
]

\addplot[gray, dashed, thick, domain=12:1e5, samples=2, forget plot] {0.05}
  node[pos=0.03, above, font=\tiny] {$\alpha=0.05$};

\rfdmethodplotlegend{rfd_sg_sub}{rfdRankSub}{Sliced rank $\chi^2$}{\RFDScaleRankSubDfour}
\rfdmethodplotlegend{rfd_sg_hot}{rfdHotelling}{Hotelling $T^2$}{\RFDScaleHotellingDfour}
\rfdmethodplotlegend{rfd_sg_sot}{rfdSOT}{Sliced OT}{\RFDScaleSOTDfour}
\rfdmethodplotlegend{rfd_sg_mmd}{rfdMMD}{MMD}{\RFDScaleMMDDfour}
\rfdmethodplotlegend{rfd_sg_mmdt}{rfdMMDTuned}{MMD tuned}{\RFDScaleMMDTunedDfour}

\end{axis}
\end{tikzpicture}
\caption{scale-gauss ($d=4$)}
\label{fig:power_runtime_d4_scale}
\end{subfigure}
\hfill
\begin{subfigure}[t]{0.49\linewidth}
\centering
\begin{tikzpicture}
\begin{axis}[
  width=\linewidth,
  height=0.72\linewidth,
  xmode=log,
  log basis x=10,
  grid=major,
  major grid style={dashed, gray!30},
  xmin=12, xmax=1e5,
  ymin=0, ymax=1.05,
  tick align=outside,
  xlabel={Avg runtime per repetition (ms, log scale)},
  ylabel={},
  ytick={0,0.25,0.5,0.75,1.0},
  yticklabels={0,0.25,0.50,0.75,1.00},
  yticklabel style={font=\scriptsize},
  x tick label style={font=\scriptsize},
]

\addplot[gray, dashed, thick, domain=12:1e5, samples=2, forget plot] {0.05}
  node[pos=0.03, above, font=\tiny] {$\alpha=0.05$};

\rfdmethodplotnolegend{rfd_sl_sub}{rfdRankSub}{Sliced rank $\chi^2$}{\RFDLapRankSubDfour}
\rfdmethodplotnolegend{rfd_sl_hot}{rfdHotelling}{Hotelling $T^2$}{\RFDLapHotellingDfour}
\rfdmethodplotnolegend{rfd_sl_sot}{rfdSOT}{Sliced OT}{\RFDLapSOTDfour}
\rfdmethodplotnolegend{rfd_sl_mmd}{rfdMMD}{MMD}{\RFDLapMMDDfour}
\rfdmethodplotnolegend{rfd_sl_mmdt}{rfdMMDTuned}{MMD tuned}{\RFDLapMMDTunedDfour}

\end{axis}
\end{tikzpicture}
\caption{shape-laplace ($d=4$)}
\label{fig:power_runtime_d4_laplace}
\end{subfigure}

\vspace{0.4em}

\begin{tikzpicture}
\node[anchor=north]{\pgfplotslegendfromname{RFDGlobalLegendPowerDfour}};
\end{tikzpicture}

\caption{\small Efficiency--power tradeoff for scale and shape alternatives in dimension \(d=4\). We compare the sliced rank-statistic \(\chi^2\)-divergence test against Hotelling's \(T^2\), sliced optimal transport, MMD, and tuned MMD. Curves show empirical power as a function of average runtime per repetition, and shaded regions denote \(\hat{\pi}\pm\mathrm{SE}\). All tests are calibrated at nominal level \(\alpha=0.05\).}
\label{fig:power_runtime_grouped_d4_scale_shape}
\end{figure}


\begin{figure}[t]
\centering


\pgfplotstableread[col sep=space]{
ms       pi
72.01    0.050000
84.65    0.070000
92.77    0.050000
94.30    0.030000
109.03   0.050000
137.71   0.030000
185.17   0.040000
}\RFDClayRankSubDfour

\pgfplotstableread[col sep=space]{
ms       pi
12.59    0.030
16.57    0.020
21.27    0.020
32.00    0.040
62.58    0.110
168.88   0.030
}\RFDClayHotellingDfour

\pgfplotstableread[col sep=space]{
ms       pi
323.84   0.040
417.67   0.070
707.33   0.150
1389.69  0.430
2974.62  0.910
4573.74  1.000
}\RFDClaySOTDfour

\pgfplotstableread[col sep=space]{
ms       pi
941.06   0.060
2202.13  0.120
11754.93 0.350
37483.62 0.769231
86458.52 1.000
}\RFDClayMMDDfour

\pgfplotstableread[col sep=space]{
ms       pi
8119.00  0.020
17803.20 0.100
41826.79 0.070
87285.59 0.206522
}\RFDClayMMDTunedDfour


\pgfplotstableread[col sep=space]{
ms       pi
124.77   0.650000
101.75   0.910000
96.42    1.000000
111.34   1.000000
107.43   1.000000
162.72   1.000000
238.69   1.000000
}\RFDGumRankSubDfour

\pgfplotstableread[col sep=space]{
ms       pi
27.85    0.930
28.75    0.600
29.43    1.000
38.94    1.000
57.06    1.000
160.41   1.000
}\RFDGumHotellingDfour

\pgfplotstableread[col sep=space]{
ms       pi
805.92   0.970
816.57   0.760
865.19   1.000
1093.70  1.000
1510.81  1.000
3460.79  1.000
}\RFDGumSOTDfour

\pgfplotstableread[col sep=space]{
ms       pi
1620.55  0.570
4040.57  0.910
7891.34  1.000
25048.39 1.000
83922.40 1.000
}\RFDGumMMDDfour

\pgfplotstableread[col sep=space]{
ms       pi
7562.81  0.250
18557.22 0.390
31422.17 0.750
76820.53 0.990
}\RFDGumMMDTunedDfour


\begin{subfigure}[t]{0.49\linewidth}
\centering
\begin{tikzpicture}
\begin{axis}[
  width=\linewidth,
  height=0.72\linewidth,
  xmode=log,
  log basis x=10,
  grid=major,
  major grid style={dashed, gray!30},
  xmin=12, xmax=1e5,
  ymin=0, ymax=1.05,
  tick align=outside,
  xlabel={Avg runtime per repetition (ms, log scale)},
  ylabel={Power $\hat{\pi}$},
  ytick={0,0.25,0.5,0.75,1.0},
  yticklabels={0,0.25,0.50,0.75,1.00},
  yticklabel style={font=\scriptsize},
  x tick label style={font=\scriptsize},
  legend to name=RFDGlobalLegendPowerDepDfour,
  legend columns=3,
  legend cell align=left,
  legend style={
    draw=gray!40,
    fill=white,
    rounded corners=2pt,
    font=\scriptsize,
    /tikz/every even column/.append style={column sep=7pt}
  },
]

\addplot[gray, dashed, thick, domain=12:1e5, samples=2, forget plot] {0.05}
  node[pos=0.03, above, font=\tiny] {$\alpha=0.05$};

\rfdmethodplotlegendlight{rfd_cl_sub}{rfdRankSub}{Sliced rank $\chi^2$}{\RFDClayRankSubDfour}
\rfdmethodplotlegendlight{rfd_cl_hot}{rfdHotelling}{Hotelling $T^2$}{\RFDClayHotellingDfour}
\rfdmethodplotlegendlight{rfd_cl_sot}{rfdSOT}{Sliced OT}{\RFDClaySOTDfour}
\rfdmethodplotlegendlight{rfd_cl_mmd}{rfdMMD}{MMD}{\RFDClayMMDDfour}
\rfdmethodplotlegendlight{rfd_cl_mmdt}{rfdMMDTuned}{MMD tuned}{\RFDClayMMDTunedDfour}

\end{axis}
\end{tikzpicture}
\caption{dependence Clayton ($d=4$)}
\label{fig:power_runtime_d4_clayton}
\end{subfigure}
\hfill
\begin{subfigure}[t]{0.49\linewidth}
\centering
\begin{tikzpicture}
\begin{axis}[
  width=\linewidth,
  height=0.72\linewidth,
  xmode=log,
  log basis x=10,
  grid=major,
  major grid style={dashed, gray!30},
  xmin=12, xmax=1e5,
  ymin=0, ymax=1.05,
  tick align=outside,
  xlabel={Avg runtime per repetition (ms, log scale)},
  ylabel={},
  ytick={0,0.25,0.5,0.75,1.0},
  yticklabels={0,0.25,0.50,0.75,1.00},
  yticklabel style={font=\scriptsize},
  x tick label style={font=\scriptsize},
]

\addplot[gray, dashed, thick, domain=12:1e5, samples=2, forget plot] {0.05}
  node[pos=0.03, above, font=\tiny] {$\alpha=0.05$};

\rfdmethodplotnolegendlight{rfd_gu_sub}{rfdRankSub}{Sliced rank $\chi^2$}{\RFDGumRankSubDfour}
\rfdmethodplotnolegendlight{rfd_gu_hot}{rfdHotelling}{Hotelling $T^2$}{\RFDGumHotellingDfour}
\rfdmethodplotnolegendlight{rfd_gu_sot}{rfdSOT}{Sliced OT}{\RFDGumSOTDfour}
\rfdmethodplotnolegendlight{rfd_gu_mmd}{rfdMMD}{MMD}{\RFDGumMMDDfour}
\rfdmethodplotnolegendlight{rfd_gu_mmdt}{rfdMMDTuned}{MMD tuned}{\RFDGumMMDTunedDfour}

\end{axis}
\end{tikzpicture}
\caption{dependence Gumbel ($d=4$)}
\label{fig:power_runtime_d4_gumbel}
\end{subfigure}

\vspace{0.4em}

\begin{tikzpicture}
\node[anchor=north]{\pgfplotslegendfromname{RFDGlobalLegendPowerDepDfour}};
\end{tikzpicture}

\caption{\small Efficiency--power tradeoff for dependence alternatives in dimension \(d=4\). We compare the sliced rank-statistic \(\chi^2\)-divergence test against Hotelling's \(T^2\), sliced optimal transport, MMD, and tuned MMD. Curves show empirical power as a function of average runtime per repetition. All tests are calibrated at nominal level \(\alpha=0.05\).}
\label{fig:power_runtime_grouped_d4_dependence}
\end{figure}

\subsection{CIFAR-10-C benchmark}
\label{app:cifar10c_rank_fdiv}

We also evaluate the same sliced rank \(f\)-divergence statistic on a high-dimensional image benchmark. We use clean CIFAR-10 test images as the reference distribution \(\nu\), and CIFAR-10-C \citep{hendrycks2019benchmarking} images as corrupted alternatives \(\mu_{c,s}\), indexed by corruption type \(c\) and severity level \(s\). Images are flattened into pixel-space vectors in \([0,1]^{3072}\).

For each trial, we draw samples from \(\nu\) and \(\mu_{c,s}\), sample \(L\) random one-dimensional projection directions, and compute the univariate rank \(f\)-divergence on each projected pair. As in the synthetic experiments, we use the chi-square generator \(f_{\chi^2}(t)=\frac12(t-1)^2\), and average the resulting one-dimensional statistics over projections. The statistic is calibrated by the same permutation procedure described above, and empirical power is estimated from the rejection frequency over repeated trials.

We compare against sliced optimal transport (SOT), ridge-regularized Hotelling's \(T^2\) \cite{hotelling1931generalization}, Gaussian-kernel maximum mean discrepancy (MMD) \cite{gretton2012kernel}, and a classifier two-sample test (C2ST) \cite{lopezpaz2017revisiting}. This experiment tests whether the proposed sliced rank \(f\)-divergence detects realistic high-dimensional distribution shifts induced by image corruptions.

The results are summarized in Figure~\ref{fig:three-across-power-runtime-severity}. The proposed rank-based statistic gains power as the matched sample size \(N=M\) increases and achieves high empirical power with substantially lower runtime than the classifier-based baseline. The runtime comparison highlights the favorable power--runtime trade-off of the rank statistic in this pixel-space benchmark. The right panel reports the updated C2ST run at $N=M=100$, showing that snow corruptions are detected more reliably at high severities, whereas JPEG compression and defocus blur are less clearly separated at this sample size.

\definecolor{colJoint}{HTML}{3776B6}
\definecolor{colSOT}{HTML}{D64B40}
\definecolor{colMMD}{HTML}{5AA04E}
\definecolor{colC2ST}{HTML}{9A73C4}

\definecolor{colGauss}{HTML}{3776B6}
\definecolor{colJPEG}{HTML}{D64B40}
\definecolor{colBlur}{HTML}{5AA04E}
\definecolor{colSnow}{HTML}{9A73C4}

\pgfplotsset{
  jointstyle/.style={
    thick, color=colJoint,
    mark=*,
    mark size=2.2pt,
    mark options={fill=colJoint}
  },
  sotstyle/.style={
    thick, color=colSOT, dashed,
    mark=square*,
    mark size=2.0pt,
    mark options={fill=colSOT}
  },
  mmdstyle/.style={
    thick, color=colMMD, dash dot,
    mark=diamond*,
    mark size=2.3pt,
    mark options={fill=colMMD}
  },
  c2ststyle/.style={
    thick, color=colC2ST,
    mark=o,
    mark size=2.2pt,
    mark options={draw=colC2ST, fill=white}
  },
  gaussstyle/.style={
    thick, color=colGauss,
    mark=*,
    mark size=2.2pt,
    mark options={fill=colGauss}
  },
  jpegstyle/.style={
    thick, color=colJPEG, dashed,
    mark=square*,
    mark size=2.0pt,
    mark options={fill=colJPEG}
  },
  blurstyle/.style={
    thick, color=colBlur, dotted,
    mark=triangle*,
    mark size=2.4pt,
    mark options={fill=colBlur}
  },
  snowstyle/.style={
    thick, color=colSnow, dash dot,
    mark=diamond*,
    mark size=2.2pt,
    mark options={fill=colSnow}
  }
}

\begin{figure*}[t]
\centering

\begin{subfigure}[t]{0.315\textwidth}
\centering
\begin{tikzpicture}
\begin{axis}[
  width=\linewidth,
  height=5.2cm,
  xmode=log,
  xmin=8, xmax=1200,
  ymin=0, ymax=1.08,
  xtick={10,50,100,500,1000},
  xticklabel style={/pgf/number format/fixed},
  ytick={0,0.25,0.5,0.75,1.0},
  xlabel={Samples $N$},
  ylabel={Average test power},
  ymajorgrids=true,
  xmajorgrids=true,
  grid style={densely dotted},
  tick style={black},
  tick label style={font=\scriptsize},
  label style={font=\scriptsize},
  title style={font=\scriptsize},
  legend=false,
  clip marker paths=true
]

\addplot+[jointstyle]
  coordinates {(10,0.000) (20,0.100) (50,0.050) (100,0.400) (200,0.700) (500,0.950)};
\addplot[name path=jointhighA, draw=none, forget plot]
  coordinates {(10,0.000) (20,0.231) (50,0.146) (100,0.615) (200,0.901) (500,1.000)};
\addplot[name path=jointlowA, draw=none, forget plot]
  coordinates {(10,0.000) (20,0.000) (50,0.000) (100,0.185) (200,0.499) (500,0.854)};
\addplot[fill=colJoint, fill opacity=0.15, forget plot]
  fill between[of=jointhighA and jointlowA];

\addplot+[sotstyle]
  coordinates {(10,0.000) (20,0.050) (50,0.400) (100,0.900) (200,1.000) (500,1.000)};
\addplot[name path=sothighA, draw=none, forget plot]
  coordinates {(10,0.000) (20,0.146) (50,0.615) (100,1.000) (200,1.000) (500,1.000)};
\addplot[name path=sotlowA, draw=none, forget plot]
  coordinates {(10,0.000) (20,0.000) (50,0.185) (100,0.769) (200,1.000) (500,1.000)};
\addplot[fill=colSOT, fill opacity=0.15, forget plot]
  fill between[of=sothighA and sotlowA];

\addplot+[mmdstyle]
  coordinates {(10,0.000) (20,0.150) (50,0.800) (100,1.000) (200,1.000) (500,1.000)};
\addplot[name path=mmdhighA, draw=none, forget plot]
  coordinates {(10,0.000) (20,0.306) (50,0.975) (100,1.000) (200,1.000) (500,1.000)};
\addplot[name path=mmdlowA, draw=none, forget plot]
  coordinates {(10,0.000) (20,0.000) (50,0.625) (100,1.000) (200,1.000) (500,1.000)};
\addplot[fill=colMMD, fill opacity=0.15, forget plot]
  fill between[of=mmdhighA and mmdlowA];

\addplot+[c2ststyle]
  coordinates {(100,0.670) (200,1.000) (500,1.000) (1000,1.000)};
\addplot[name path=c2sthighA, draw=none, forget plot]
  coordinates {(100,0.735) (200,1.000) (500,1.000) (1000,1.000)};
\addplot[name path=c2stlowA, draw=none, forget plot]
  coordinates {(100,0.605) (200,1.000) (500,1.000) (1000,1.000)};
\addplot[fill=colC2ST, fill opacity=0.15, forget plot]
  fill between[of=c2sthighA and c2stlowA];

\end{axis}
\end{tikzpicture}
\caption*{\scriptsize Mean power vs.\ samples (averaged over corruptions and severities)}
\end{subfigure}\hfill
\begin{subfigure}[t]{0.315\textwidth}
\centering
\begin{tikzpicture}
\begin{axis}[
  width=\linewidth,
  height=5.2cm,
  xmode=log,
  xmin=0.07, xmax=700,
  xtick={0.1,1,10,100,500},
  xticklabel style={/pgf/number format/fixed},
  ymin=0, ymax=1.08,
  ytick={0,0.25,0.5,0.75,1.0},
  xlabel={Avg.\ runtime (s)},
  ylabel={Average test power},
  ymajorgrids=true,
  xmajorgrids=true,
  grid style={densely dotted},
  tick style={black},
  tick label style={font=\scriptsize},
  label style={font=\scriptsize},
  title style={font=\scriptsize},
  legend=false,
  clip marker paths=true
]

\addplot+[jointstyle]
  coordinates {(0.349,0.000) (0.411,0.100) (0.440,0.050) (0.498,0.400) (0.750,0.700) (1.215,0.950)};
\addplot[name path=jointhighB, draw=none, forget plot]
  coordinates {(0.349,0.000) (0.411,0.231) (0.440,0.146) (0.498,0.615) (0.750,0.901) (1.215,1.000)};
\addplot[name path=jointlowB, draw=none, forget plot]
  coordinates {(0.349,0.000) (0.411,0.000) (0.440,0.000) (0.498,0.185) (0.750,0.499) (1.215,0.854)};
\addplot[fill=colJoint, fill opacity=0.15, forget plot]
  fill between[of=jointhighB and jointlowB];

\addplot+[sotstyle]
  coordinates {(0.953,0.000) (1.561,0.050) (3.909,0.400) (4.467,0.900) (4.880,1.000) (5.783,1.000)};
\addplot[name path=sothighB, draw=none, forget plot]
  coordinates {(0.953,0.000) (1.561,0.146) (3.909,0.615) (4.467,1.000) (4.880,1.000) (5.783,1.000)};
\addplot[name path=sotlowB, draw=none, forget plot]
  coordinates {(0.953,0.000) (1.561,0.000) (3.909,0.185) (4.467,0.769) (4.880,1.000) (5.783,1.000)};
\addplot[fill=colSOT, fill opacity=0.15, forget plot]
  fill between[of=sothighB and sotlowB];

\addplot+[mmdstyle]
  coordinates {(0.081,0.000) (0.348,0.150) (1.728,0.800) (3.759,1.000) (8.790,1.000) (34.078,1.000)};
\addplot[name path=mmdhighB, draw=none, forget plot]
  coordinates {(0.081,0.000) (0.348,0.306) (1.728,0.975) (3.759,1.000) (8.790,1.000) (34.078,1.000)};
\addplot[name path=mmdlowB, draw=none, forget plot]
  coordinates {(0.081,0.000) (0.348,0.000) (1.728,0.625) (3.759,1.000) (8.790,1.000) (34.078,1.000)};
\addplot[fill=colMMD, fill opacity=0.15, forget plot]
  fill between[of=mmdhighB and mmdlowB];

\addplot+[c2ststyle]
  coordinates {(27.757,0.670) (64.066,1.000) (196.152,1.000) (510.164,1.000)};
\addplot[name path=c2sthighB, draw=none, forget plot]
  coordinates {(27.757,0.735) (64.066,1.000) (196.152,1.000) (510.164,1.000)};
\addplot[name path=c2stlowB, draw=none, forget plot]
  coordinates {(27.757,0.605) (64.066,1.000) (196.152,1.000) (510.164,1.000)};
\addplot[fill=colC2ST, fill opacity=0.15, forget plot]
  fill between[of=c2sthighB and c2stlowB];

\end{axis}
\end{tikzpicture}
\caption*{\scriptsize Mean power vs.\ runtime (averaged over corruptions and severities)}
\end{subfigure}\hfill
\begin{subfigure}[t]{0.315\textwidth}
\centering
\begin{tikzpicture}
\begin{axis}[
  width=\linewidth,
  height=5.2cm,
  xmin=0.8, xmax=5.2,
  ymin=0, ymax=1.08,
  xtick={1,2,3,4,5},
  ytick={0,0.25,0.5,0.75,1.0},
  xlabel={Severity},
  ylabel={Power},
  ymajorgrids=true,
  xmajorgrids=true,
  grid style={densely dotted},
  tick style={black},
  tick label style={font=\scriptsize},
  label style={font=\scriptsize},
  title style={font=\scriptsize},
  legend=false,
  clip marker paths=true
]

\addplot+[gaussstyle]
  coordinates {(1,0.700) (2,0.600) (3,0.600) (4,0.700) (5,0.700)};
\addplot[name path=gausshigh, draw=none, forget plot]
  coordinates {(1,0.984) (2,0.904) (3,0.904) (4,0.984) (5,0.984)};
\addplot[name path=gausslow, draw=none, forget plot]
  coordinates {(1,0.416) (2,0.296) (3,0.296) (4,0.416) (5,0.416)};
\addplot[fill=colGauss, fill opacity=0.12, forget plot]
  fill between[of=gausshigh and gausslow];

\addplot+[jpegstyle]
  coordinates {(1,0.600) (2,0.600) (3,0.600) (4,0.600) (5,0.600)};
\addplot[name path=jpeghigh, draw=none, forget plot]
  coordinates {(1,0.904) (2,0.904) (3,0.904) (4,0.904) (5,0.904)};
\addplot[name path=jpeglow, draw=none, forget plot]
  coordinates {(1,0.296) (2,0.296) (3,0.296) (4,0.296) (5,0.296)};
\addplot[fill=colJPEG, fill opacity=0.12, forget plot]
  fill between[of=jpeghigh and jpeglow];

\addplot+[blurstyle]
  coordinates {(1,0.600) (2,0.500) (3,0.500) (4,0.500) (5,0.700)};
\addplot[name path=blurhigh, draw=none, forget plot]
  coordinates {(1,0.904) (2,0.810) (3,0.810) (4,0.810) (5,0.984)};
\addplot[name path=blurlow, draw=none, forget plot]
  coordinates {(1,0.296) (2,0.190) (3,0.190) (4,0.190) (5,0.416)};
\addplot[fill=colBlur, fill opacity=0.12, forget plot]
  fill between[of=blurhigh and blurlow];

\addplot+[snowstyle]
  coordinates {(1,0.600) (2,0.900) (3,0.900) (4,0.900) (5,1.000)};
\addplot[name path=snowhigh, draw=none, forget plot]
  coordinates {(1,0.904) (2,1.000) (3,1.000) (4,1.000) (5,1.000)};
\addplot[name path=snowlow, draw=none, forget plot]
  coordinates {(1,0.296) (2,0.714) (3,0.714) (4,0.714) (5,1.000)};
\addplot[fill=colSnow, fill opacity=0.12, forget plot]
  fill between[of=snowhigh and snowlow];

\end{axis}
\end{tikzpicture}
\caption*{\scriptsize C2ST power vs.\ severity ($N=M=100$)}
\end{subfigure}

\vspace{0.15em}

\begin{tikzpicture}
\begin{axis}[
  hide axis,
  xmin=0, xmax=1,
  ymin=0, ymax=1,
  legend columns=4,
  legend style={
    /tikz/every even column/.style={column sep=6pt},
    draw=none,
    fill=none,
    font=\scriptsize
  }
]
  \addlegendimage{jointstyle} \addlegendentry{Joint ranks}
  \addlegendimage{sotstyle}   \addlegendentry{SOT}
  \addlegendimage{mmdstyle}   \addlegendentry{MMD}
  \addlegendimage{c2ststyle}  \addlegendentry{C2ST}
\end{axis}
\end{tikzpicture}

\vspace{-0.35em}

\begin{tikzpicture}
\begin{axis}[
  hide axis,
  xmin=0, xmax=1,
  ymin=0, ymax=1,
  legend columns=4,
  legend style={
    /tikz/every even column/.style={column sep=6pt},
    draw=none,
    fill=none,
    font=\scriptsize
  }
]
  \addlegendimage{gaussstyle} \addlegendentry{Gaussian noise}
  \addlegendimage{jpegstyle}  \addlegendentry{JPEG compression}
  \addlegendimage{blurstyle}  \addlegendentry{Defocus blur}
  \addlegendimage{snowstyle}  \addlegendentry{Snow}
\end{axis}
\end{tikzpicture}

\caption{\small
Empirical power for two-sample testing between CIFAR-10 and CIFAR-10-C at level $\alpha=0.01$.
Left: average power versus matched sample size $N=M$.
Middle: average power versus average runtime.
Right: C2ST power versus corruption severity for $N=M=100$.
Averages are computed over the considered corruption types and severities when applicable, and shaded bands show $95\%$ binomial confidence intervals. The C2ST results are based on 10 independent trials for each corruption--severity pair.
}
\label{fig:three-across-power-runtime-severity}
\end{figure*}

\section{Limitations and future work} \label{app:limitations_future_work}

The proposed rank-statistic construction reduces divergence estimation to operations on ranks and histograms, yielding a fully sample-based surrogate that avoids explicit density-ratio modeling. At the same time, several limitations remain, many of which are shared by other projection-averaging objectives. In particular, the multivariate variant is defined by averaging a univariate discrepancy over random one-dimensional projections. While projection families can characterize distributions in the limit, any finite number of directions $L$ can miss informative orientations, especially when the discrepancy is concentrated in a low-dimensional subspace, encoded in rare but important directions, or dominated by higher-order dependence patterns. Similar phenomena are documented for sliced objectives in optimal transport and generative modeling \citep{kolouri2019gsw}. A practical implication is that performance can depend non-trivially on $L$ and on how directions are sampled, and diagnosing ``missed directions'' may be difficult without problem-specific insight.

A second limitation is that the surrogate is inherently discretized through the resolution parameter $K$. Although the theoretical analysis establishes monotonicity and consistency as $K\to\infty$, the choice of a finite parameter $K$ induces approximation error and may distort the geometry of the objective. This is particularly relevant when the divergence is used as a learning signal: the discretization can alter local sensitivity and may emphasize coarse distributional differences over fine structure. Developing principled, data-dependent rules for selecting $K$ (and, in the sliced case, jointly selecting $(K,L)$) remains an open problem. Promising directions include selection via held-out calibration, adaptive schedules that increase $K$ over training, and criteria based on stability of estimates across nearby resolutions.

Third, the sliced rank-statistic construction inherits an accuracy-compute trade-off from Monte Carlo integration on the sphere. In high dimensions, each projection reduces the problem to a 1D rank histogram, but capturing direction-dependent mismatch may require many directions $L$. Increasing $L$ improves coverage of informative orientations and typically reduces Monte Carlo variability through averaging, yet the overall cost grows roughly linearly in $L$ (and also increases with the rank resolution $K$ through the histogram/Bernstein evaluation). At realistic scales, this can make the sliced surrogate expensive, whereas using too few directions risks missing informative projections and yielding overly optimistic (or misleading) discrepancy estimates--a limitation shared by other sliced objectives \citep{kolouri2019gsw}.

Several extensions could increase the information per projection beyond i.i.d.\ random directions. One option is to replace purely random directions with structured ensembles that reduce redundancy (e.g., near-orthogonal directions) or with low-discrepancy point sets on the sphere, which can lower projection variance at fixed $L$ compared to standard Monte Carlo \citep{sobol1967distribution,dick2010digital}. Another direction is data-dependent slicing: rather than sampling $s$ uniformly, directions can be biased toward projections with the largest (rank-based) discrepancy, connecting to max-sliced and projection-pursuit ideas \citep{deshpande2019max,paty2019subspace}. More generally, one could adapt learned slicing to the rank-statistic setting by choosing projection directions in a learned feature space, or by learning a small set of directions jointly with the generator so that each slice is maximally informative \citep{kolouri2019gsw}. Recent analyses of sliced distances and direction sampling also suggest that carefully designed projection sets can achieve greater statistical efficiency and stronger practical guarantees \citep{nietert2022statistical}.

Fourth, it would be valuable to place the sliced rank-statistic $f$-divergence in a more formal ``flow'' framework, in the same spirit as sliced-Wasserstein flows \citep{liutkus2019sliced}. Concretely, one can view the sliced rank objective as defining a functional on probability measures whose descent induces transport dynamics: at each time step, projected one-dimensional rank corrections define a drift field that moves particles toward the data distribution, while optional entropy/noise terms control dispersion and prevent collapse. A theory along these lines could clarify the accuracy-compute trade-off introduced by slicing (finite $L$ directions) and discretization (finite $K$), and could enable finite-time guarantees for particle discretizations that explicitly track how the error depends on $(m,K)$ and step sizes, mirroring the role of Monte Carlo slicing and time discretization in flow-based analyses \citep{liutkus2019sliced}.

In parallel, these transport dynamics suggest an amortized alternative: instead of running particle updates at test time, a generator could be trained to emulate one (or a few) steps of rank-proximal transport, or to directly map base noise to samples that minimize the sliced rank divergence. Such amortization could dramatically reduce the per-iteration cost at image scale.

Finally, empirical evaluation can be broadened along several axes. The current experiments emphasize synthetic settings and representative implicit learning tasks, but more diverse benchmarks (including larger-scale image generation, text/sequence data, and domain adaptation scenarios, and time-series forecasting) would better delineate when rank-statistic divergences outperform classical and neural alternatives. In addition, ablations that isolate the effects of $K$, $L$, projection sampling, and batching would help translate theoretical guarantees into robust practitioner guidance. Extensions to conditional divergences, two-sample testing, and settings with nuisance variables (e.g., covariate shift) are also natural, since rank constructions are compatible with sample-based pipelines and may be combined with representation learning.

\section{Computational Cost and Memory Footprint}
\label{app:computational-cost}

We briefly discuss the computational cost of the sliced rank-statistic
$f$-divergence estimators. For each projection direction, both the
rank-statistic estimator and sliced Wasserstein first project the samples
$X=\{x_i\}_{i=1}^N \subset \mathbb{R}^d$ and
$Y=\{y_j\}_{j=1}^M \subset \mathbb{R}^d$ onto one dimension, which costs
\[
    \mathcal{O}((N+M)d).
\]
Thus, for a fixed number of samples and slices, both methods have the same
explicit linear dependence on the ambient dimension $d$.

For evaluation and two-sample testing, we use a non-differentiable hard-rank
implementation. In each slice, we sort the projected reference samples and
evaluate the empirical CDF of the reference distribution at the projected
samples from $\mu$:
\[
    \widehat F_\nu(x_i)
    =
    \frac{1}{M}\#\{j : y_j \le x_i\}.
\]
The resulting empirical CDF values are used to construct the degree-$K$ rank
histogram and the corresponding discrete $f$-divergence.
This gives the per-slice cost
\[
    \mathcal{O}\big((N+M)d + M\log M + N\log M + NK\big),
\]
and therefore
\[
    \mathcal{O}\big(L((N+M)d + M\log M + N\log M + NK)\big)
\]
over $L$ projections. The corresponding memory footprint is mild: apart from
the input samples, it stores the projected samples, the rank/CDF values, and
the rank histogram, giving approximately
\[
    \mathcal{O}\big((N+M)d + N + M + K\big).
\]

For differentiable training, we use a soft-rank variant in which the empirical
CDF is smoothed by logistic comparisons and the rank histogram is represented
with a Bernstein basis of degree $K$. This yields a per-slice cost
\[
    \mathcal{O}\big((N+M)d + NM + NK\big),
\]
where the $NM$ term comes from the pairwise soft-CDF computation. Its peak
memory is approximately
\[
    \mathcal{O}\big((N+M)d + NM + NK\big),
\]
since the implementation stores a pairwise comparison matrix of size
$N\times M$ and a Bernstein basis matrix of size $N\times(K+1)$. When slices
are processed sequentially, increasing $L$ mainly increases runtime rather
than peak memory.

For comparison, a standard sliced Wasserstein estimator has per-slice cost
\[
    \mathcal{O}\big((N+M)d + N\log N + M\log M\big),
\]
due to sorting the one-dimensional projected samples. Hence, both sliced
Wasserstein and the hard-rank rank-statistic estimator have the same
asymptotic dependence on the ambient dimension $d$, namely linear scaling
through the projection step. The hard-rank estimator adds an extra
one-dimensional term $NK$ for evaluating the rank-statistic $f$-divergence at
resolution $K$. The soft-rank variant is more expensive because it pays the
pairwise $NM$ cost, but this overhead is independent of the ambient dimension.

\paragraph{Empirical time-to-quality comparison.}
We also report an empirical time-to-quality comparison on a synthetic
generative modeling benchmark in dimensions \(d \in \{2,4,10\}\). For
\(d=2\), the target distribution is a noisy two-moons distribution. For
larger \(d\), we keep the same two-dimensional nonlinear backbone and augment
it with additional smooth nonlinear coordinates, followed by coordinate-wise
standardization.

All methods train the same fully connected generator
\(G_\theta:\mathbb{R}^{16}\to\mathbb{R}^d\), with three hidden layers of
width \(128\) and SiLU activations. We use Adam with learning rate
\(2\cdot 10^{-3}\) and gradient clipping at \(5\). Sliced rank $f$-divergence (Rank), sliced Wasserstein (SWD),
and maximum mean discrepancy (MMD) use mini-batches of size \(256\), while the full assignment optimal transport (OT)
baseline uses mini-batches of size \(96\), due to the cubic cost of the
Hungarian solver. The rank objective uses \(L=64\) random projections,
Bernstein degree \(K=32\), the JS generator \(f\), and soft-rank temperature
scale \(0.05\). Sliced Wasserstein  uses \(L=64\) projections and \(p=2\).
MMD uses a multi-scale RBF kernel with bandwidths given by fixed multiples of
the median heuristic. Evaluation uses \(n_{\rm eval}=2048\) real and
generated samples; rank-JS and SWD evaluations use \(128\) random projections.

As an external stopping criterion, we use a classifier-based estimate of the
Jensen--Shannon divergence. At each evaluation checkpoint, a post-hoc binary
classifier with two hidden layers of width \(128\) is trained for \(300\)
steps to distinguish generated from real samples, using a \(70/30\)
train/validation split. This classifier is used only for evaluation and early
stopping, not for training the generators.

Because a fixed classifier-JS threshold becomes more stringent as dimension
increases, we use dimension-adapted targets:
\[
    \widehat{\mathrm{JS}}_{\mathrm{clf}} \le 0.05
    \quad (d=2), \qquad
    \widehat{\mathrm{JS}}_{\mathrm{clf}} \le 0.15
    \quad (d=4), \qquad
    \widehat{\mathrm{JS}}_{\mathrm{clf}} \le 0.30
    \quad (d=10).
\]
The resulting time-to-target comparison is summarized in Table~\ref{tab:clf-js-time-to-target} and visualized in Figures~\ref{fig:time-to-target-all} and~\ref{fig:final-clf-js}. Table~\ref{tab:clf-js-time-to-target} reports, for each method and dimension, whether the prescribed classifier-JS target is reached, together with the corresponding optimization step, wall-clock time, and final value of \(\widehat{\mathrm{JS}}_{\mathrm{clf}}\). Figure~\ref{fig:time-to-target-all} provides a direct visual comparison of the runtime required to reach the dimension-adapted target, while Figure~\ref{fig:final-clf-js} shows the final classifier-based JS estimate attained by each method.

\begin{table}[h]
\begin{tabular}{llrrrr}
\toprule
Dimension & Method & Reached target & Step & Time (s) &
Final \(\widehat{\mathrm{JS}}_{\mathrm{clf}}\) \\
\midrule
\(d=2\)  & Rank & yes & \(600\)  & \(91.3\)  & \(0.0487\) \\
         & SWD  & yes & \(800\)  & \(147.0\) & \(0.0444\) \\
         & MMD  & no  & --       & --        & \(0.1415\) \\
         & OT   & yes & \(2500\) & \(266.9\) & \(0.0465\) \\
\midrule
\(d=4\)  & Rank & yes & \(1100\) & \(131.6\) & \(0.1410\) \\
         & SWD  & yes & \(1500\) & \(117.6\) & \(0.1101\) \\
         & MMD  & no  & --       & --        & \(0.3942\) \\
         & OT   & no  & --       & --        & \(0.2845\) \\
\midrule
\(d=10\) & Rank & yes & \(600\)  & \(59.5\)  & \(0.2941\) \\
         & SWD  & yes & \(1700\) & \(120.5\) & \(0.3000\) \\
         & MMD  & no  & --       & --        & \(0.4280\) \\
         & OT   & no  & --       & --        & \(0.4856\) \\
\bottomrule
\end{tabular}
\centering
\caption{Time-to-target comparison using the external classifier-based JS
criterion. A dash indicates that the method did not reach the target within
the \(3000\)-step budget.}
\label{tab:clf-js-time-to-target}
\end{table}

Overall, these results show that the rank objective is competitive in
wall-clock time under an external classifier-based stopping criterion. In
\(d=2\), rank reaches the target faster than all baselines. In \(d=4\), both
rank and sliced Wasserstein reach the adapted target; rank requires fewer
optimization steps, while sliced Wasserstein is slightly faster in wall-clock
time and attains a lower final classifier-JS value. In \(d=10\), rank again
reaches the target substantially faster than sliced Wasserstein, whereas MMD
and full assignment OT do not reach the target within the training budget.
These results suggest that the finite-projection rank objective provides a
competitive time-to-quality tradeoff, while also motivating the use of
complementary evaluation metrics such as classifier-JS, sliced Wasserstein,
MMD, and rank-JS.

\begin{figure}[H]
\centering
\begin{tikzpicture}
\begin{axis}[
    ybar,
    width=0.85\linewidth,
    height=5.6cm,
    bar width=7pt,
    ylabel={Time to target (s)},
    symbolic x coords={$d=2$,$d=4$,$d=10$},
    xtick=data,
    ymin=0,
    ymax=480,
    enlarge x limits=0.22,
    legend style={
        at={(0.5,-0.18)},
        anchor=north,
        legend columns=4,
        draw=none
    },
    nodes near coords,
    nodes near coords style={font=\scriptsize, rotate=90, anchor=west},
    every axis plot/.append style={fill opacity=0.75},
    tick label style={font=\small},
    label style={font=\small},
]
\addplot coordinates {
    ($d=2$,91.3)
    ($d=4$,131.6)
    ($d=10$,59.5)
};

\addplot coordinates {
    ($d=2$,147.0)
    ($d=4$,117.6)
    ($d=10$,120.5)
};

\addplot coordinates {
    ($d=2$,452.7)
    ($d=4$,298.0)
    ($d=10$,310.3)
};

\addplot coordinates {
    ($d=2$,266.9)
    ($d=4$,205.2)
    ($d=10$,215.2)
};

\legend{Rank, SWD, MMD, OT}
\end{axis}
\end{tikzpicture}
\caption{Time-to-target comparison under the external classifier-based JS
stopping criterion. The targets are dimension-adapted:
\(\widehat{\mathrm{JS}}_{\mathrm{clf}}\le 0.05\) for \(d=2\),
\(\le 0.15\) for \(d=4\), and \(\le 0.30\) for \(d=10\).
For methods that do not reach the target, the reported value corresponds to
the total runtime under the \(3000\)-step budget. In particular, MMD does not
reach the target in any dimension, and OT does not reach it in \(d=4,10\).}
\label{fig:time-to-target-all}
\end{figure}

\begin{figure}[H]
\centering
\begin{tikzpicture}
\begin{axis}[
    ybar,
    width=0.85\linewidth,
    height=5.6cm,
    bar width=7pt,
    ylabel={Final $\widehat{\mathrm{JS}}_{\mathrm{clf}}$},
    symbolic x coords={$d=2$,$d=4$,$d=10$},
    xtick=data,
    ymin=0,
    ymax=0.55,
    enlarge x limits=0.22,
    legend style={
        at={(0.5,-0.18)},
        anchor=north,
        legend columns=4,
        draw=none
    },
    nodes near coords,
    nodes near coords style={font=\scriptsize, rotate=90, anchor=west},
    every axis plot/.append style={fill opacity=0.75},
    tick label style={font=\small},
    label style={font=\small},
]
\addplot coordinates {
    ($d=2$,0.0487)
    ($d=4$,0.1410)
    ($d=10$,0.2941)
};

\addplot coordinates {
    ($d=2$,0.0444)
    ($d=4$,0.1101)
    ($d=10$,0.3000)
};

\addplot coordinates {
    ($d=2$,0.1415)
    ($d=4$,0.3942)
    ($d=10$,0.4280)
};

\addplot coordinates {
    ($d=2$,0.0465)
    ($d=4$,0.2845)
    ($d=10$,0.4856)
};

\legend{Rank, SWD, MMD, OT}
\end{axis}
\end{tikzpicture}
\caption{Final classifier-based JS estimate for each method and dimension.
The dimension-adapted stopping thresholds are
\(\widehat{\mathrm{JS}}_{\mathrm{clf}}\le 0.05\) for \(d=2\),
\(\le 0.15\) for \(d=4\), and \(\le 0.30\) for \(d=10\). Lower is better.}
\label{fig:final-clf-js}
\end{figure}

\section{Experimental Setup}\label{Experimental Setup}
All experiments were performed on a MacBook Pro running macOS 13.2.1, equipped with an Apple M1 Pro CPU and 16 GB of RAM. When GPU acceleration was required, we used a single NVIDIA TITAN Xp with 12 GB of VRAM. Detailed hyperparameter settings for each experiment are provided in the corresponding sections. The code is available at \url{https://github.com/josemanuel22/rsfdiv}.

\end{document}